\documentclass{article}

\usepackage[final,nonatbib]{neurips_2022}

\def\neurips{1}
\def\arxiv{0} %

\usepackage[utf8]{inputenc} %
\usepackage[T1]{fontenc}    %
\usepackage{url}            %
\usepackage{booktabs}       %
\usepackage{amsfonts}       %
\usepackage{nicefrac}       %
\usepackage{microtype}      %
\usepackage[usenames,dvipsnames]{xcolor}         %

\definecolor{redcolor}{rgb}{0.79607843, 0.25098039, 0.25882353}
\definecolor{bluecolor}{rgb}{0,0.36,0.69}
\definecolor{pinkcolor}{rgb}{0.70196078, 0.16078431, 0.50980392}
\definecolor{darkbluecolor}{rgb}{0,0.08,0.45}
\usepackage[colorlinks=true,linkcolor=redcolor,citecolor=bluecolor,urlcolor=bluecolor]{hyperref}       %

\usepackage{amsmath}
\usepackage{amssymb}
\usepackage{mathtools}
\usepackage{amsthm}
\usepackage[capitalize]{cleveref}
\usepackage{enumitem}
\usepackage[font=small]{caption}
\usepackage{nicefrac}
\usepackage{algorithm}
\usepackage[noend]{algpseudocode}
\usepackage{wrapfig}
\usepackage{booktabs} %
\usepackage{multicol}
\usepackage{multirow, varwidth}

\usepackage{stackengine}

\setstackEOL{\\}

\usepackage{etoc}
\theoremstyle{plain}
\newtheorem{theorem}{Theorem}[section]
\newtheorem{proposition}[theorem]{Proposition}
\newtheorem{lemma}[theorem]{Lemma}
\newtheorem{corollary}[theorem]{Corollary}
\theoremstyle{definition}
\newtheorem{definition}[theorem]{Definition}
\newtheorem{assumption}[theorem]{Assumption}
\theoremstyle{remark}

\usepackage{xspace} %

\usepackage{etoolbox}
\newcommand{\smalldisplayskips}{%
  \setlength{\abovedisplayskip}{2pt}%
  \setlength{\belowdisplayskip}{2pt}%
  \setlength{\abovedisplayshortskip}{2pt}%
  \setlength{\belowdisplayshortskip}{2pt}}
\appto{\normalsize}{\smalldisplayskips}
\appto{\small}{\smalldisplayskips}
\appto{\footnotesize}{\smalldisplayskips}

\DeclareMathOperator*{\argmin}{argmin}
\DeclareMathOperator*{\argmax}{argmax}
\newcommand{\mrmtl}{\textsf{MR-MTL}\xspace}
\allowdisplaybreaks

\crefformat{section}{\S#2#1#3} %
\crefformat{subsection}{\S#2#1#3}
\crefformat{subsubsection}{\S#2#1#3}

\newcommand{\lam}{\lambda}
\newcommand{\eps}{\varepsilon}

\title{On Privacy and Personalization in\\ Cross-Silo Federated Learning}

\author{%
  Ziyu Liu \quad Shengyuan Hu \quad Zhiwei Steven Wu \quad Virginia Smith \\
  Carnegie Mellon University\\
  \texttt{\{kzliu, shengyuanhu, zstevenwu, smithv\}@cmu.edu} \\
}

\begin{document}

\ifnum\arxiv=1
  \date{\vspace{-1em}}
\fi

\etocdepthtag.toc{mtchapter}
\etocsettagdepth{mtchapter}{subsection}
\etocsettagdepth{mtappendix}{none}

\maketitle

\vspace{-0.1in}
\begin{abstract}
  \noindent
While the application of differential privacy (DP) has been well-studied in cross-device federated learning (FL), there is a lack of work considering DP and its implications for cross-silo FL, a setting characterized by a limited number of clients each containing many data subjects. In cross-silo FL, usual notions of \textit{client-level} DP are less suitable as real-world privacy regulations typically concern the in-silo data subjects rather than the silos themselves. In this work, we instead consider an alternative notion of \textit{silo-specific sample-level} DP, where silos set their own privacy targets for their local examples. Under this setting, we reconsider the roles of personalization in federated learning. In particular, we show that mean-regularized multi-task learning (\textsf{MR-MTL}), a simple personalization framework, is a strong baseline for cross-silo FL: under stronger privacy requirements, silos are incentivized to federate more with each other to mitigate DP noise, resulting in consistent improvements relative to standard baseline methods. We provide an empirical study of competing methods as well as a theoretical characterization of \textsf{MR-MTL} for mean estimation, highlighting the interplay between privacy and cross-silo data heterogeneity. Our work serves to establish baselines for private cross-silo FL as well as identify key directions of future work in this area.

\end{abstract}

\vspace{-0.1in}
\section{Introduction} \label{sec:intro}

Recent advances in machine learning often rely on large, {centralized} datasets~\cite{dalle, chowdhery2022palm, liu2022convnet}, but curating such data may not always be viable, particularly when the data contains private information and  must remain siloed across clients (e.g.\ mobile devices or hospitals).
Recently, federated learning~(FL) \cite{mcmahan2017communication, kairouz2021advances} has emerged as a paradigm for learning from such distributed data, but it has been shown that its data minimization principle alone may not provide adequate privacy protection for  participants~\cite{wen2022fishing,yin2021see}.
To obtain formal privacy guarantees, there has thus been extensive work applying \textit{differential privacy} (DP)~\cite{dwork2006calibrating, dwork2014algorithmic} to various parts of the FL pipeline (e.g.\ \cite{geyer2017differentially, mcmahan2018learning, heikkila2020differentially, kairouz2021distributed, agarwal2021skellam, proj-fedavg, hu2021private, ramaswamy2020training, girgis2021shuffled}).

Existing approaches for differentially private FL are typically designed for \textit{client-level} DP in that they protect the federated clients, such as mobile devices (``user-level''),
tasks in multi-task learning (``task-level''),
or data silos like institutions (``silo-level''), and DP is achieved by clipping and noising the client model updates.
While client-level DP is considered a strong privacy notion as all data of a single client is protected, it may not be suitable for \textit{cross-silo} FL, where there are fewer clients but each hold many data subjects that require protection.
For example, when hospitals/banks/schools wish to federate patient/customer/student records, it is the people owning those records rather than the participating silos that should be protected. In fact, laws and regulations may mandate such participation in FL be disclosed publicly~\cite{Vaid2020.08.11.20172809}, compromising the privacy of the federating clients.

In this work, we instead consider a more natural model of \textit{silo-specific sample-level privacy} (\cref{fig:dp-setting}, with variants appearing in~\cite{heikkila2020differentially, lowy2021private, zheng2021federated, kanani2021private}): the $k$-th silo may set its own $(\eps_k, \delta_k)$ sample-level DP target for any learning algorithm with respect to its local dataset. With this formulation in mind, we then reconsider the impact of privacy, heterogeneity, and personalization in cross-silo FL. In particular, we explore existing baselines for FL (mostly developed in cross-device settings) across private cross-silo benchmarks, and we find that the simple baseline of mean-regularized MTL (\mrmtl)
\begin{wrapfigure}[10]{t}[5pt]{0.57\textwidth}
  \centering
  \includegraphics[width=\linewidth]{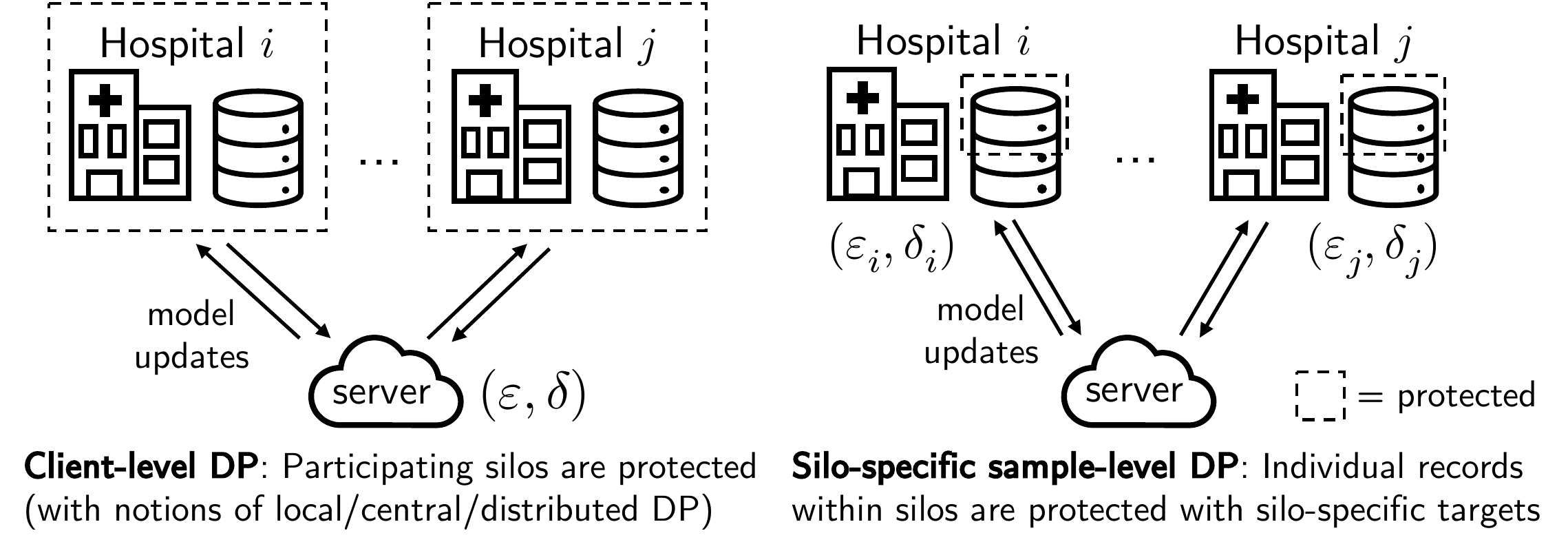}
  \vspace{-1.5em}
  \caption{\textbf{Client-level} DP vs \textbf{Silo-specific sample-level} DP.}
  \label{fig:dp-setting}
  \vspace{-1em}
\end{wrapfigure}
has many advantages for this setting relative to other more common (and possibly more complex) methods.
We then further analyze the performance of \mrmtl under varying levels of heterogeneity and privacy, both in theory and practice. In addition to establishing baselines for cross-silo FL, we also identify interesting future directions in this area (\cref{sec:discussion,supp-sec:future-directions}).
We summarize our contributions below:\footnote{Code is available at \url{https://github.com/kenziyuliu/private-cross-silo-fl}.}

\begin{itemize}[leftmargin=*,itemsep=2pt,parsep=0pt,topsep=0pt,partopsep=0pt]
  \item We consider the notion of \textit{silo-specific sample-level differential privacy} (DP) as a more realistic privacy model for \textit{cross-silo federated learning} (FL). We analyze its implications on existing FL algorithms and, in particular, how it interfaces with data heterogeneity across silos.
  \item We empirically show that mean-regularized multi-task learning (\textsf{MR-MTL}), a  simple form of model personalization, is a remarkably strong baseline under silo-specific sample-level DP.
  Core to its effectiveness is its ability to (roughly) interpolate on the model personalization spectrum between local training and FedAvg with minimal privacy overhead.
  \item We provide a theoretical analysis of \textsf{MR-MTL} under mean estimation and characterize how \textsf{MR-MTL} navigates the tension between privacy and cross-silo data heterogeneity.
  \item Finally, we examine the complications of deploying an optimal \textsf{MR-MTL} instance that stem from the privacy cost of \textit{hyperparameter tuning}.
  Our reasoning also applies to other personalization methods whose advantage over local training and/or FedAvg hinges on selecting the best hyperparameter(s).
  This raises important questions around the practicality of leveraging personalization to balance the emerging tradeoffs under silo-specific sample-level DP.

\end{itemize}

\section{Preliminaries}

\textbf{Federated Learning (FL)}~\cite{mcmahan2017communication,li2020federated,kairouz2021advances} is a distributed learning paradigm with an emphasis on data protection: in every training round, each client (participant) downloads the current global model from a central server, trains it with the local dataset, and uploads the model changes (instead of the data) back to the server, which then aggregates the changes into a new global model.
A basic instantiation of FL is FedAvg~\cite{mcmahan2017communication}, where clients are stateless and the server performs a simple (weighted) average.
\textit{Cross-device} FL refers to settings with many clients each with limited data, bandwidth, availability, etc. (e.g.\ mobile devices). In contrast, \textit{cross-silo} FL typically involves less clients (e.g.\ banks, schools, hospitals) but each with more resources. Two distinguishing characteristics of cross-silo FL relevant to our work are that (1) silos may have sufficient data to fit a reasonable local model \textit{without} FL, and (2) each data point in a silo tends to map to a data subject (a person) requiring privacy protection.

\textbf{Differential Privacy (DP).}\enspace
Despite its ability to mitigate systemic privacy risks, FL by itself does {not} provide formal privacy guarantees for participants' data~\cite{kairouz2021advances,yin2021see,wen2022fishing}, and differential privacy is often used in conjunction with FL to ensure that an algorithm does not leak the privacy of its inputs.
\begin{definition}[Differential Privacy~\cite{dwork2006calibrating, dwork2014algorithmic}]
  A randomized algorithm $M: \mathcal X^n \to \mathcal Y$, where $\mathcal X^n$ is the set of datasets with $n$ samples and $\mathcal Y$ is the set of outputs, is $(\eps, \delta)$-DP if for any subset $S \subseteq \mathcal Y$ and any neighboring $x, x'$ differing in only one sample (by replacement), we have
  \begin{align}
    \Pr[M(x) \in S] \le \exp(\eps) \cdot \Pr[M(x') \in S] + \delta.
  \end{align}
\end{definition}
\vspace{-0.5em}
To apply DP to a dataset query, one commonly used method is the Gaussian mechanism~\cite{dwork2014algorithmic}, which involves bounding the contribution ($\ell^2$-norm) of each sample in the dataset followed by adding Gaussian noise proportional to that bound onto the aggregate.
To apply DP in FL, one needs to define the ``dataset'' to protect; typically, as in \textit{client-level} DP, this is the set of FL participants and thus the model updates from each participant in every round should be bounded and noised.
In learning settings, we need to repeatedly query a dataset and the privacy guarantee composes. We use DP-SGD~\cite{song2013stochastic,bassily2014private,abadi2016deep} for ensuring sample-level DP for model training, and we use R\'enyi DP~\cite{mironov2017renyi} and zCDP~\cite{bun2016concentrated} for tight privacy composition.
In certain FL algorithms, clients also perform additional work such as cluster selection~\cite{mansour2020three,ghosh2020efficient} that incurs privacy overhead with respect to its local dataset that must be accounted for independently from DP-SGD. See \cref{supp-sec:background-extra} for additional background.

\textbf{Personalized FL.}\enspace
Model personalization is a key technique for improving utility under data heterogeneity across silos.\footnote{Note that ``personalization'' refers to customizing models for each \textit{client} in FL rather than a specific person.}
Past work has examined the roles of
local adaptation~\cite{wang2019federated,yu2020salvaging, cheng2021fine},
multi-task learning~\cite{smith2017federated, sattler2020clustered}, clustering~\cite{ghosh2020efficient, cho2021personalized, mansour2020three, sattler2020clustered}, public data~\cite{zhao2018federated, mansour2020three}, meta learning~\cite{jiang2019improving,li2019differentially, fallah2020personalized}, or other forms of
model mixtures~\cite{liang2020think,li2021ditto,mansour2020three,hanzely2020federated,deng2020adaptive,agarwal2020federated}.
Notably, many methods leverage extra computation to some extent (e.g.\ extra iterations~\cite{liang2020think,li2021ditto, cheng2021fine} or cluster selection~\cite{ghosh2020efficient,mansour2020three}), which will result in privacy overhead under silo-specific sample-level DP as discussed in the following section.
Of particular interest is the family of mean-regularized multi-task learning (\textsf{MR-MTL}) methods~\cite{evgeniou2004regularized, t2020personalized, hanzely2020federated, hanzely2020lower} (see \cref{alg:mr-mtl} for a typical instantiation). We find that \textsf{MR-MTL}, while extremely simple, is a strong baseline for private cross-silo FL.

\vspace{-0.4em}
\section{Revisiting the Privacy Model for Cross-Silo Federated Learning}
\label{sec:privacy-model}

\begin{figure}[t]
  \centering
  \includegraphics[width=\linewidth]{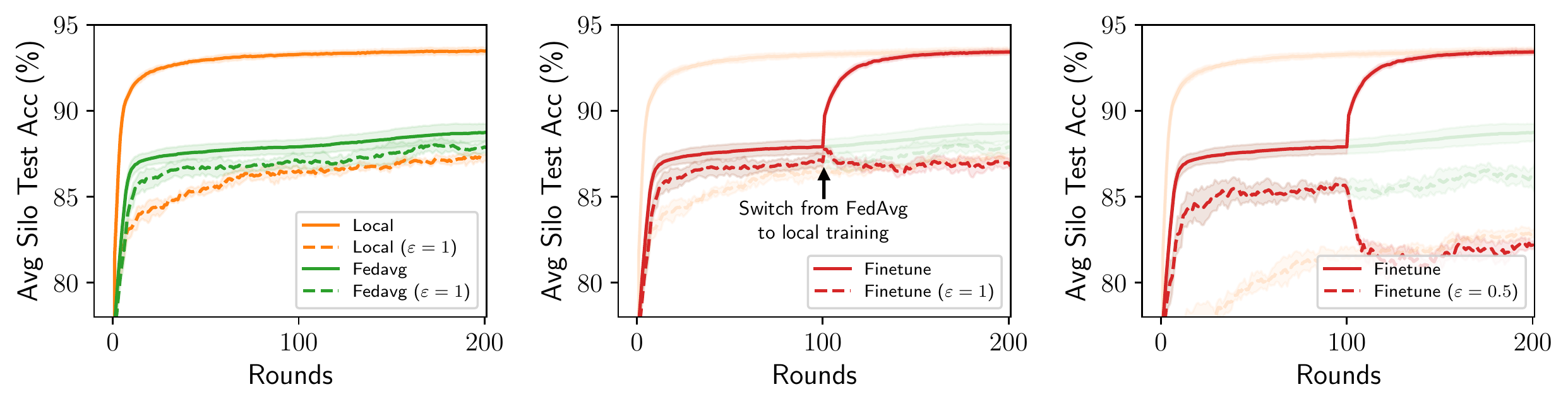}
  \vspace{-2em}
  \caption{
    \textbf{Two notable phenomena under \textit{silo-specific sample-level DP}}: (1) FedAvg can serve to cancel out per-silo DP noise and thus outperform local training even when the latter works better without privacy ({left});
    (2) Local finetuning~\cite{yu2020salvaging,cheng2021fine} (FedAvg followed by local training) may not improve utility as expected, as the effect of noise reduction is removed when finetuning begins ({mid} \& {right}).
    Results report mean test acc~$\pm$~std on the Vehicle dataset over 5 seeds. For simplicity, all silos budgets for the same labeled $\eps$ with $\delta = 10^{-7}$.
    Transparent curves refer to local/FedAvg runs with the same $\eps$ labeled for finetuning (compare {left} \& {mid}).
    }
  \label{fig:finetune-gap}
  \vspace{-1em}
\end{figure}

To date, the prevalent privacy model for federated learning has been to protect the participating clients, i.e.\ client-level DP.
For cross-silo FL, however, several factors render client-level DP less appropriate. First, cross-silo FL often involves a small number of clients and it can be utility-wise more costly to attain the same privacy targets. For example, privacy amplification via sampling~\cite{abadi2016deep,mironov2019r} may not apply on the client level since all silos typically participate in every round.
Second, many existing methods focus on enforcing client-level DP in a non-local model and thus defines a shared privacy target for all participants, but in real-world cross-silo settings, participants under different jurisdictions (e.g.\ states) may have varying privacy requirements and thus opt for different privacy-utility tradeoffs.
Third, while silo-level protection implies sample-level protection, it may be too stringent in practice as silos often have large local datasets.
These unique properties for private cross-silo learning motivate us to consider \textit{silo-specific sample-level} DP as an alternative privacy model  (\cref{fig:dp-setting}):
\begin{definition}[Silo-specific sample-level DP] \label{def:sssl-dp}
  A cross-silo FL algorithm with $K$ clients (silos) satisfy $\{(\eps_k, \delta_k)\}_{k \in [K]}$-``\textit{silo-specific sample-level DP}'' if the local (personalized) model $M_k$ of every silo $k \in [K]$ satisfies $(\eps_k, \delta_k)$-DP w.r.t.\ the silo's local dataset of training examples.
\end{definition}
\vspace{-0.5em}
\textbf{Characteristics of silo-specific sample-level privacy.}\enspace
Importantly, silo-specific sample-level DP is defined over the \textit{disjoint} datasets of the individual silos, rather than the combined dataset of all silos.\footnote{
  A record in such a combined dataset is at most ($\max_i \eps_i$, $\max_i \delta_i$)-DP~\cite{mcsherry2009privacy,yu2019differentially,proj-fedavg}.
  Moreover, if multiple records (either within a silo or across silos) map to the same person,
  then it is more intricate to protect the person rather than their records.
  Here we focus on the case where each entity has at most one record across the combined dataset (e.g.\ students attending exactly one school).
  See \cref{supp-sec:non-bijective} for discussions.
}
To instantiate this setup in FL, each silo can simply run DP-SGD~\cite{song2013stochastic,bassily2014private,abadi2016deep} with a noise scale calibrated to gradually spend its privacy budget over $T$ training rounds, and return the noisy model update at each round.
This privacy notion has several important implications on the dynamics of FL:
\begin{enumerate}[leftmargin=*,itemsep=2pt,parsep=0pt,topsep=0pt,partopsep=0pt]
  \item
  \textbf{Silos incur privacy costs with queries to their data, but \textit{not} with participation in FL.}
  This follows from DP's robustness to post-processing: the silos' model updates in each round already satisfy their own sample-level DP targets, and participation by itself does not involve extra dataset queries (e.g.\ DP-SGD steps).
  In contrast, local training without communication can be kept noise-free under client-level DP, but participation in FL requires privatization.
  Two immediate consequences of the above are that (1) local training and FedAvg now have \textit{identical} privacy costs, and (2) \textit{local finetuning} for model personalization may no longer work as expected (\cref{fig:finetune-gap}).
  \item
  \textbf{Less reliance on a trusted server.}
  As a corollary of the above, all model updates of silo $k$ satisfy (at least) $(\eps_k, \delta_k)$-DP against external adversaries, including all other silos and the orchestrating server~\cite{dwork2014algorithmic,zheng2021federated}.
  In contrast, client-level DP under a non-local model necessitates some trust on the server, even for distributed DP methods (e.g.~\cite{erlingsson2019amplification, cheu2019distributed, kairouz2021distributed,agarwal2021skellam, cheu2021differential, chen2022poisson}).

  \item
  \textbf{Tradeoff emerges between costs from privacy and heterogeneity.}
  As privacy-perserving noises are added independently on each silo, they are reflected in silos' model updates and can thus be mitigated when the model updates are aggregated (e.g.\ via FedAvg), leading to a smaller utility drop due to DP for the shared model.
  On the other hand, federation also means that
  the shared model may suffer from \textit{client heterogeneity} (non-iid data across silos).
  This intuition is observed in \cref{fig:finetune-gap}: while local training may outperform FedAvg without privacy (as a result of heterogeneity), the opposite can be true when privacy is added (as a result of noise variance reduction).
\end{enumerate}

The first and last in the above are of particular interest because they suggest that \textit{model personalization} can play a key and distinct role in our privacy setting.
Specifically, local training (no FL participation) and FedAvg (full FL participation) can be viewed as two ends of a \textit{personalization spectrum} with identical privacy costs; if local training minimizes the effect of data heterogeneity but enjoys no DP noise reduction, and contrarily for FedAvg, it is then natural to ask whether there exist personalization methods that lie in between and achieve better utility, and, if so, what methods would work best.

\textbf{Related privacy settings.}\enspace
Past work on differentially private FL has concentrated on client-level DP and cross-device FL (e.g.\ \cite{geyer2017differentially,brendan2018learning,hu2021private,girgis2021shuffled,kairouz2021practical,kairouz2021distributed,aldaghri2021feo2}), and the application of DP in cross-silo FL, particularly where each silo defines its {own DP targets} for records of its {own dataset}, is relatively underexplored.
Privacy notions closest to ours first appeared in~\cite{truex2019hybrid,li2019differentially, heikkila2020differentially, zheng2021federated, lowy2021private, kanani2021private, proj-fedavg}.
In \cite{truex2019hybrid}, each client adds its own one-shot noise onto its outgoing update, but in learning scenarios this provides client-level protection.
The works of \cite{li2019differentially,lowy2021private,zheng2021federated,heikkila2020differentially,proj-fedavg} study analogous privacy notions, though they respectively focus on boosting utility~\cite{li2019differentially}, analyzing statistical rates~\cite{lowy2021private}, adapting FL to $f$-DP~\cite{zheng2021federated,gaussian-dp}, applying security primitives~\cite{heikkila2020differentially}, and learning a better global model; the aspects of heterogeneity, DP noise reduction, the personalization spectrum, and their interplay (e.g.\ \cref{fig:finetune-gap,fig:mrmtl-bump}) were unexplored.
The work of \cite{kanani2021private} also considers a similar privacy notion, but the authors study a disparate trust assumption where DP noise is \textit{not} added to local training/finetuning such that the final personalized models lack privacy guarantees.
We note that the trust model most suitable for private cross-silo FL may be application-specific; in this work, we focus on the setting where the outputs of the FL procedure (the personalized models) must remain differentially private.

\vspace{-0.5em}
\section{Baselines for Private Cross-Silo Federated Learning}
\label{sec:baselines}

With the characteristics from \cref{sec:privacy-model} in mind, we now explore various methods on cross-silo benchmarks.
We defer additional details as well as results on more settings and datasets to the appendix.

\textbf{Datasets.}
We consider four cross-silo datasets that span regression/classification and convex/non-convex tasks: Vehicle~\cite{duarte2004vehicle}, School~\cite{goldstein1991multilevel}, Google Glass (GLEAM)~\cite{rahman2015unintrusive}, and CIFAR-10~\cite{krizhevsky2009learning}. The first three datasets have real-world cross-silo characteristics: Vehicle contains measurements of road segments for classifying the type of passing vehicles, School contains student attributes for predicting exam scores, and GLEAM contains motion tracking data to classify wearers' activities. CIFAR-10 has heterogeneous client splits following~\cite{t2020personalized,shamsian2021personalized}. See \cref{supp-sec:dataset-models} for more details and datasets.

\textbf{Benchmark methods.}
We consider several representative methods in the personalized FL literature beyond local training and FedAvg~\cite{mcmahan2017communication}:
local finetuning~\cite{wang2019federated,yu2020salvaging,cheng2021fine} (a simple but strong baseline for model personalization),
Ditto~\cite{li2021ditto} (state-of-the-art personalization method),
Mocha~\cite{smith2017federated} (personalization with task relationship learning),
IFCA/HypCluster~\cite{ghosh2020efficient,mansour2020three} (state-of-the-art hard clustering method for client models), and the mean-regularized multi-task learning (\mrmtl) methods~\cite{evgeniou2004regularized, t2020personalized, hanzely2020federated, hanzely2020lower} (which we analyze in \cref{sec:method}).
For fair comparison under silo-specific sample-level DP, we align all benchmark methods on the total privacy budget by first restricting the total number of iterations over the local datasets and then account for any privacy overheads (in the form of necessary extra steps~\cite{li2021ditto} or cluster selection for IFCA/HypCluster~\cite{ghosh2020efficient,mansour2020three}).
Importantly, many other personalization methods can either be reduced to one of the above under convex settings (e.g.\ \cite{jiang2019improving,liang2020think}) or are unsuitable due to large privacy overheads (e.g.\ large factor of extra steps for \cite{fallah2020personalized}).

\begin{figure}[t]
  \centering
  \includegraphics[width=\linewidth]{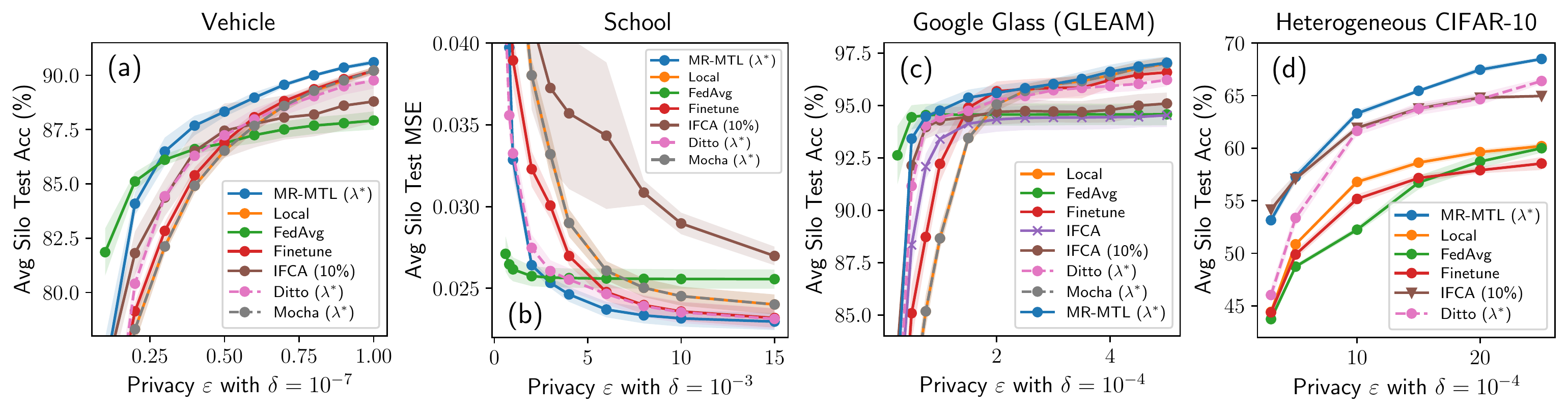}
  \vspace{-1.75em}
  \captionof{figure}{
    \textbf{Privacy-utility tradeoffs} (privacy budgets $\eps$ vs.\ test metrics, mean~$\pm$~std w/ 5 seeds) for various personalization methods on {\textbf{Vehicle}}, {\textbf{School}}, {\textbf{GLEAM}}, and {\textbf{Heterogeneous CIFAR-10}} datasets respectively.\
    For simplicity, every silo targets for the same $(\eps, \delta)$ under {silo-specific sample-level DP}.
    $\lambda^*$ denotes a tuned regularization strength where applicable.
    ``Local'' denotes local training (clients train and keep their own models).
    ``IFCA (10\%)'' denotes forming clusters for only first 10\% of training rounds due to privacy overhead (\cref{fig:privacy-overhead}).
  }
  \vspace{-1.5em}   %
  \label{fig:pareto-merged}
\end{figure}

\begin{wrapfigure}[15]{Tr}[0pt]{0.3\textwidth}
  \centering
  \includegraphics[width=1\linewidth]{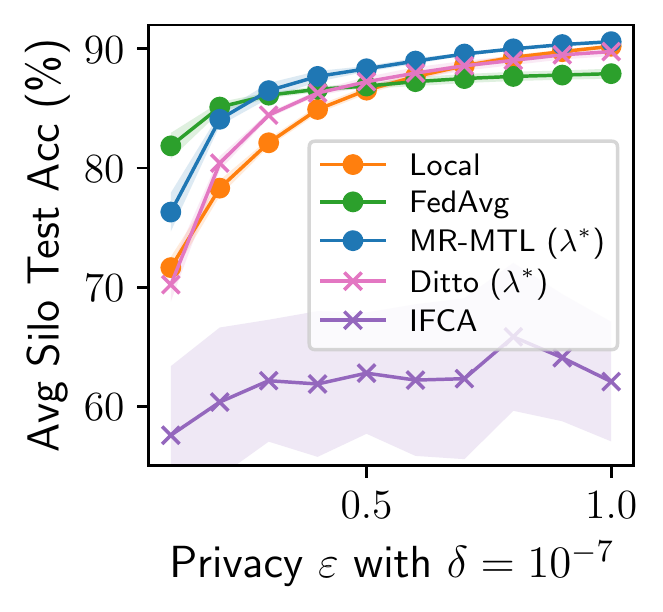}
  \vspace{-1.8em}
  \caption{
    The privacy overhead of cluster selection (IFCA) and extra iterations (Ditto) compared to local, FedAvg, and \mrmtl under silo-specific sample-level DP.
  }
  \label{fig:privacy-overhead}
\end{wrapfigure}

\textbf{Training setup.}\enspace
For all methods, we use minibatch DP-SGD in each silo to satisfy silo-specific sample-level privacy; while certain methods may have more efficient solvers (e.g.\ dual form for~\cite{smith2017federated}), we want compatibility with DP-SGD as well as privacy amplification via sampling on the example level for tight accounting.
For all experiments, silos train for 1 local epoch in every round (except for~\cite{li2021ditto} which runs $\ge$ 2 epochs).
Hyperparameter tuning is done via grid search for all methods.
Importantly, when comparing the benchmark methods, we do not account for the privacy cost of hyperparameter tuning in order to focus on their inherent privacy-utility tradeoff; we revisit this issue in \cref{sec:discussion}.
For simplicity, we use the same privacy budget for all silos, i.e.\ $(\eps_i, \delta_i) = (\eps_j, \delta_j)$ for all $i, j \in [K]$; note that this is not a restrictive assumption since the effect of having varying DP noise scales from different budgets can be attained by varying local dataset sizes.

\textbf{Results.}\enspace
\cref{fig:pareto-merged} shows the privacy-utility tradeoffs across four datasets.
We observe that \mrmtl consistently outperforms a suite of baseline methods, and that it performs at least as good as local training and FedAvg (endpoints of the personalization spectrum), except at high-privacy regimes (possibly different for each dataset).
In particular, there exists a range of $\eps$ values where \textsf{MR-MTL} can give significantly better utility over local training and FedAvg \textit{under the {same} privacy budgets} ($\eps \approx 0.5, 6, 1.5$ for \cref{fig:pareto-merged} (a, b, c) respectively); this is our key regime of interest.

\textbf{Effects of silo-specific sample-level privacy.}\enspace
In \cref{fig:finetune-gap} we saw that local finetuning may not improve utility as expected, motivating us to reconsider the roles of federation and personalization (\cref{sec:method} below). In \cref{fig:privacy-overhead}, we consider the implication of silo-specific sample-level DP from the effects of privacy overhead due to additional dataset queries: (1) If IFCA~\cite{ghosh2020efficient,mansour2020three} performs cluster selection at every round (default behavior), then the extra privacy cost can be prohibitive; (2) Despite its similarity to \mrmtl, Ditto~\cite{li2021ditto}'s privacy overhead makes it less competitive (see also~\cref{fig:mrmtl-bump}).

\vspace{-0.4em}
\section{On the Effectiveness of Mean-Regularized MTL for Private Cross-Silo FL}
\label{sec:method}

Following the observations in \cref{sec:baselines}, we now examine the desirable properties of a good algorithm under silo-specific sample-level privacy and understand why \mrmtl may be an attractive candidate.

\textbf{Federation as noise reduction.}\enspace
A key message from \cref{sec:privacy-model} and \cref{fig:finetune-gap} is that the utility cost from DP can be significantly smaller for FedAvg compared to local training even when the latter spends an identical privacy budget.
This implies that FedAvg may have inherent benefits for DP noise reduction, despite the noise are added in the \textit{gradient} space rather than the parameter space (as in client-level DP).
Consider a simple setting of DP gradient descent: the update rule
$w_k^{(t+1)} = w_k^{(t)} - \frac{\eta}{n_k} \left( z^{(t)} + \sum_{i=1}^{n_k} g_{k, i}^{(t)} \right)$
for silo $k$ recursively expands to $w_k^{(T)} = w_k^{(0)} - \frac{\eta}{n_k} \sum_{t=0}^{T-1} z^{(t)} -  \frac{\eta}{n_k} \sum_{t=0}^{T-1} \sum_{i=1}^{n_k} g_{k, i}^{(t)}$ over $T$ steps,
where
$g_{k, i}^{(t)}$ is the clipped gradient of the $i$-th local example at step $t$ (with norm bound $c$) out of a total of $n_k$ examples, and ${z^{(t)} \sim \mathcal N(0, \sigma^2 \mathbf{I})}$ is the Gaussian noise added to the gradient sum at step $t$ that targets for an overall privacy budget of $(\eps_k, \delta_k)$ over $T$ steps with $\sigma^2 = O\left( c^2 T \ln (1 / \delta_k) / \eps_k^2 \right)$~\cite{abadi2016deep}.
The cumulative noise term
\begin{align}
  Z^{(T)} \triangleq -\frac{\eta}{n_k} \sum_{t} z^{(t)} \sim \mathcal N\left(0, \enspace O\left( \frac{ \eta^2  c^2  T^{2}  \ln (1 / \delta_k) }{ n_k^2 \eps_k^2 }  \right) \cdot \mathbf{I} \right)
\end{align}
indeed implies that each silo's model update has an independent Gaussian random walk component~\cite{rognvaldsson1994langevin, an1996effects, welling2011bayesian} whose variance can be reduced by averaging with other silos' updates, as in FedAvg.\footnote{
    The work of~\cite{kairouz2021practical} examines the benefits of adding negatively correlated (instead of independent) noises $z_t$ across time steps. While this is a potential orthogonal extension to our use of local DP-SGD \textit{within} each silo, it may not be directly applicable to our main focus of reducing noise variance \textit{across} silos, since for each silo $k$ to satisfy its own $(\eps_k, \delta_k)$ requirement, it must add noise independent to other silos.
}
A similar reasoning applies to SGD cases since the additive DP noises are i.i.d.\ across the minibatches.

\textbf{Model personalization for privacy-heterogeneity cost tradeoff.}\enspace
A major downside of FedAvg is that it may underperform simple local training due to data heterogeneity (e.g.\ \cite{yu2020salvaging} and \cref{fig:finetune-gap}), particularly given that clients in cross-silo FL often have sufficient data to fit reasonable local models.
This suggests an emerging role for model personalization on top of its benefits in terms of utility~\cite{wang2019federated}, robustness~\cite{yu2020salvaging}, or fairness~\cite{li2021ditto} under heterogeneity:
our privacy model allows local training and FedAvg to be viewed as two endpoints of a \textit{personalization spectrum} that respectively mitigate the utility costs of heterogeneity and privacy noise with identical privacy budgets (recall \cref{sec:privacy-model}); this means that personalization methods could be viewed as interpolating between these endpoints and that various personalization methods essentially do so in different ways.
However, our empirical observations motivate the following key properties of a good personalization algorithm:
\begin{enumerate}[leftmargin=*,itemsep=2pt,parsep=0pt,topsep=0pt,partopsep=0pt]
  \item \textbf{Noise reduction}: The effect of noise reduction is present throughout training so that the utility costs from DP can be consistently mitigated. Local finetuning is a counter-example (\cref{fig:finetune-gap}).
  \item \textbf{Minimal privacy overhead}: There are little to no additional local dataset queries to prevent extra noise for DP-SGD under a fixed privacy budget. In effect, such privacy overhead can shift the utility tradeoff curve downwards, and Ditto~\cite{li2021ditto} may be viewed as a counter-example (\cref{fig:privacy-overhead}).
  \item \textbf{Smooth interpolation along the personalization spectrum}: The interpolation between local training and FedAvg should be fine-grained (if not continuous) such that an optimal tradeoff should be attainable. Clustering~\cite{ghosh2020efficient,mansour2020three} may be viewed as a counter-example when there are no clear heterogeneity strucutre across clients.
\end{enumerate}
These properties are rather restrictive and they render many promising algorithms less attractive.
For example, model mixture~\cite{liang2020think, bietti2022personalization, deng2020adaptive} and local adaptation~\cite{yu2020salvaging,cheng2021fine} methods can incur linear overhead in dataset iterations, and so can multi-task learning~\cite{li2021ditto, smith2017federated, sattler2020clustered} methods that benefit from additional training. Clustering methods~\cite{ghosh2020efficient, cho2021personalized, mansour2020three, sattler2020clustered} can also incur overhead with cluster selection~\cite{ghosh2020efficient,mansour2020three}, distillation~\cite{cho2021personalized}, or training restarts~\cite{sattler2020clustered}, and they discretize the personalization spectrum in a way that depends on external parameters (e.g., the number of clients, clusters, or top-down partitions).

\textbf{The case for mean-regularization.}\enspace
These considerations point to {mean-regularized multi-task learning} (\textsf{MR-MTL}) as one of the simplest yet particularly suitable forms of personalization.
\textsf{MR-MTL} has manifested in various forms in the literature~\cite{evgeniou2004regularized,zhang2015deep,t2020personalized,hanzely2020federated,cheng2021fine,hu2021private} with the key idea that a personalized model $w_k$ for each silo $k$ should be close to the mean of all personalized models $\bar w$ via a regularization penalty $\nicefrac\lambda2 \| w_k - \bar w \|_2^2$ (see \cref{alg:mr-mtl} for a typical instantiation).
The hyperparameter $\lambda$ serves as a smooth knob between local training and FedAvg, with $\lambda = 0$ recovering local training and a larger $\lambda$ forces the personalized models $w_k$ to be closer to each other (``federate more'').
However, it is an imperfect knob as $\lambda \to \infty$ may \textit{not} recover FedAvg under a typical optimization setup as the regularization term may dominate the gradient step $w_k^{(t+1)} = w_k^{(t)} - \eta \left( g_t + \lambda \left(w_k^{(t)} - \bar w^{(t)} \right) \right)$ where $g_t$ is the noisy clipped gradient, and \mrmtl may thus underperform FedAvg in high-privacy regimes that necessitate a large $\lam$ to mitigate DP noise (\cref{fig:pareto-merged}).

\textsf{MR-MTL} has the attractive properties that:
(1) noise reduction is achieved throughout training via a soft constraint that personalized models are close to an averaged model;
(2) for fixed $\lambda$ it has \textit{zero} additional privacy cost compared to local training/FedAvg as it does not involve extra dataset queries; and
(3) $\lambda$ provides a smooth interpolation along the personalization spectrum.
Moreover, compared to other regularization-based MTL methods, it adds only one hyperparameter $\lambda$ (cf.\ \cite{jalali2010dirty,zhou2011clustered,gong2012robust}); this has important practical implications as will be discussed in \cref{sec:discussion}.
It also has fast convergence~\cite{zhang2021survey} and easily extends to deep learning with good empirical performance in the primal~\cite{t2020personalized,hu2021private} (cf.\ \cite{barzilai2015convex,smith2017federated,liu2017distributed}). It is also sufficently extensible to handle structured heterogeneity (discussed below).
We argue through the following empirical analyses that these properties make \textsf{MR-MTL} a strong baseline under silo-specific sample-level DP.

\begin{wrapfigure}[17]{r}[0pt]{0.47\textwidth}
  \centering
  \vspace{-1.5em}
  \includegraphics[width=1.05\linewidth]{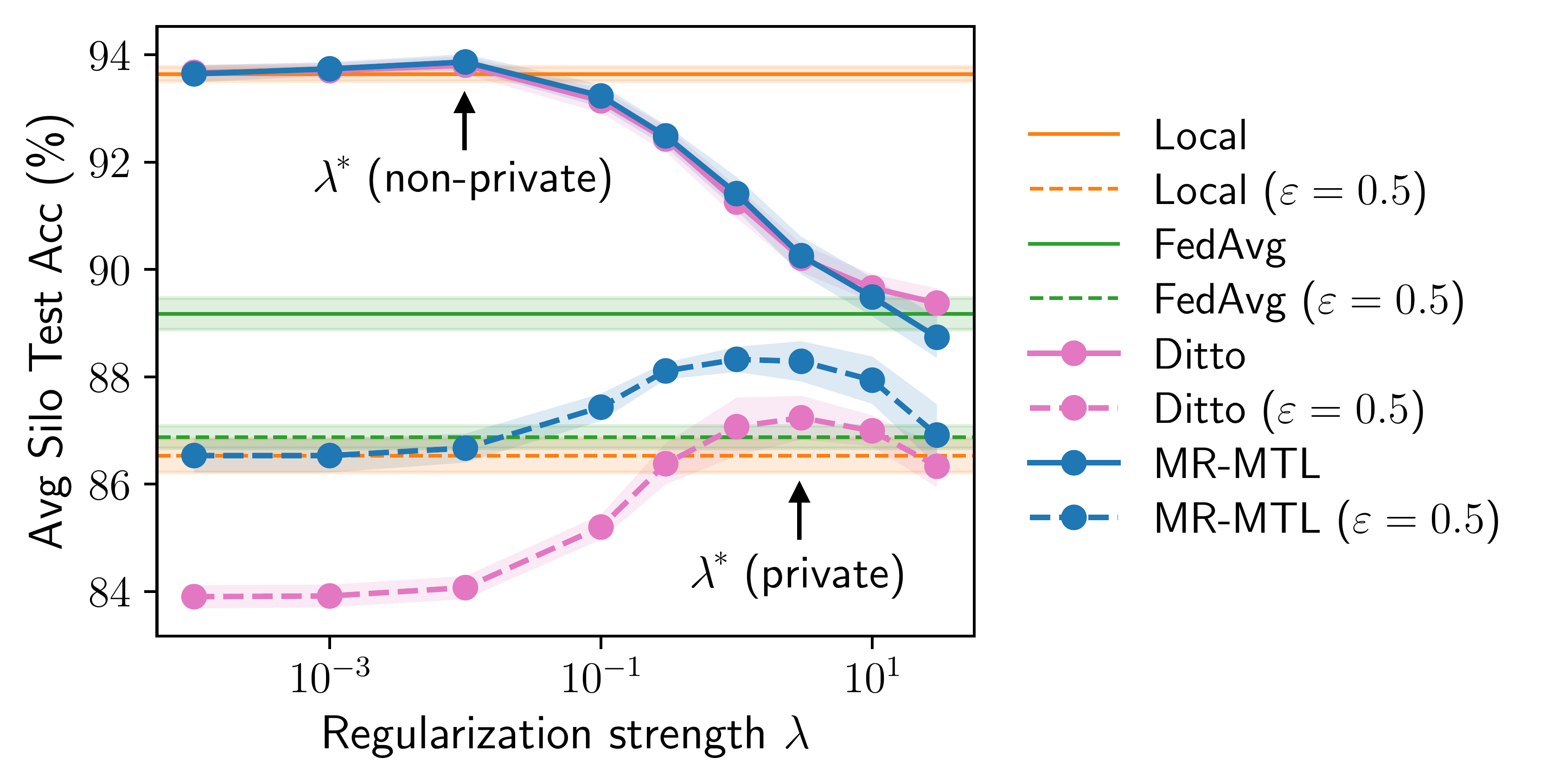}
  \vspace{-1.8em}
  \caption{
    \textbf{Test acc $\pm$ std of \mrmtl with varying $\lam$} (corresponding to $\eps$ = 0.5 in \cref{fig:pareto-merged} (a)).
    Optimal points ($\lambda^*$) exist where it outperforms both ends of the spectrum under the \textit{same} privacy.
    Ditto~\cite{li2021ditto} gives a similar interpolation but has strictly worse
    privacy-utility tradeoffs due to its privacy overhead.
    See \cref{supp-sec:bump-extension} for extensions of this figure to other privacy settings and datasets.
  }
  \label{fig:mrmtl-bump}
\end{wrapfigure}

\textbf{Navigating the emerging privacy-heterogeneity cost tradeoff.}
In \cref{fig:mrmtl-bump} we study the effect of the regularization strength $\lam$ on the model utility directly.
There are several notable observations:
(1) In both private and non-private settings, $\lam$ serves to roughly interpolate between local training and FedAvg.
(2) The utility at the best $\lam^*$ may outperform both endpoints. This is significant for the private setting since \mrmtl achieves an \textit{identical} privacy guarantee as the endpoints.
(3) Moreover, the \textit{advantage} of \mrmtl over the best of the endpoints are also larger under privacy.
(4) The value of $\lam^*$ also increases under privacy, indicating that the personalized silo models are closer to each other (i.e. silos are encouraged to ``federate more'') for noise reduction.
We will characterize these behaviors in \cref{sec:analysis}.
We also consider Ditto~\cite{li2021ditto} in \cref{fig:mrmtl-bump}, a state-of-the-art personalization method that resembles \mrmtl and exhibits similar behaviors, to illustrate the effect of privacy overhead from its extra local training iterations.

\begin{figure}[t]
  \centering
  \includegraphics[width=\linewidth]{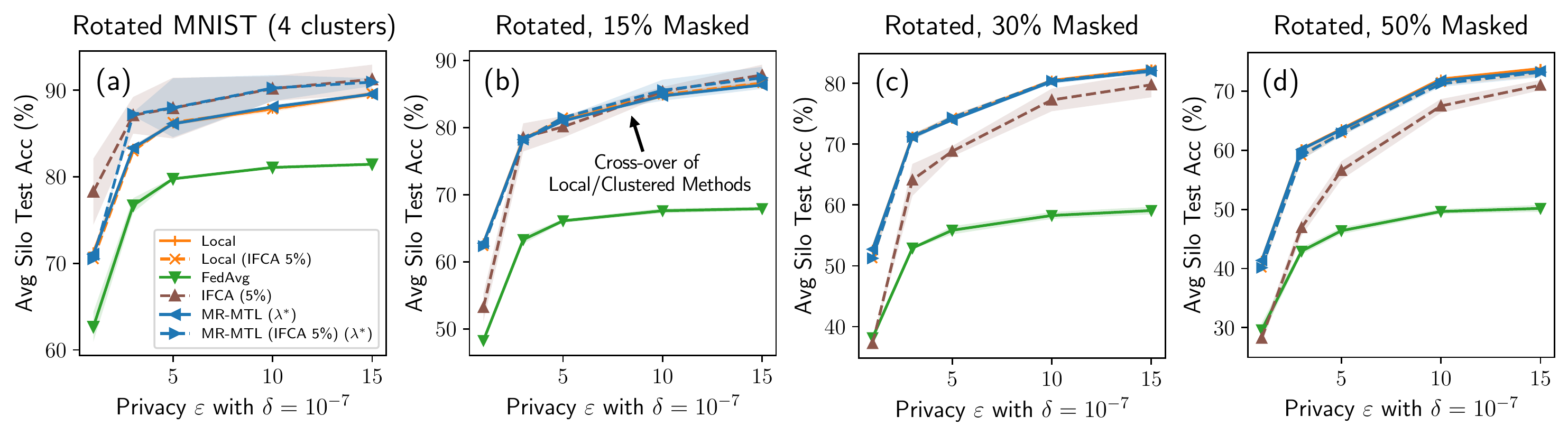}
  \vspace{-1.5em}
  \caption{
    \textbf{Test acc $\pm$ std on Rotated \& Masked MNIST}.
    (IFCA 5\%) denotes warm-starting the method by running IFCA~\cite{ghosh2020efficient,mansour2020three} for first 5\% of rounds followed by running the method within the cluster structures.
  }
  \vspace{-1.5em}
  \label{fig:rot-mnist}
\end{figure}

\textbf{\mrmtl under structured heterogeneity.}\enspace
We further study (1) the extensibility of \mrmtl as a strong baseline method to handle clustering structures of silo data distributions and (2) its flexibility to handle varying heterogeneity levels, by manually introducing two layers of heterogeneity to the MNIST dataset~\cite{lecun1998gradient}.
The first layer is 4-way rotations: train/test images are evenly split into 4 groups of 10 silos, with each group applying $\{0^{\circ}, 90^{\circ}, 180^{\circ}, 270^{\circ}\}$ of rotation to their images. The second layer is \textit{silo-specific} masking: each silo generates and applies its unique random mask of $2\times2$ white patches to its images, with varying masking probability. Together, the 1st layer creates 4 well-defined silo clusters, and the 2nd layer (gradually) adds \textit{intra-cluster} heterogeneity.
Importantly, our goal is not to contrive a utility advantage of \mrmtl (in fact, the added heterogeneity is disadvantageous to mean-regularization), but to examine its extensibility and flexibility as a strong baseline method to match the best methods by construction.
Under the 1st layer of heterogeneity, clustered FL methods~\cite{ghosh2020efficient,mansour2020three} should be optimal if the correct clusters are formed since there is no intra-cluster heterogeneity; with increasing silo-specific heterogeneity in the 2nd layer, local training should be increasingly more attractive. See \cref{supp-sec:dataset-models} for more details on the setup and examples of images.

We propose a simple heuristic to precondition or ``warm-start'' \mrmtl with a small number of training rounds by running private clustering (with IFCA~\cite{ghosh2020efficient,mansour2020three}) followed by mean-regularized training \textit{within each formed cluster} (see \cref{supp-sec:algorithms} for details). We find that this simple heuristic, with convergence properties carried forward from its components~\cite{ghosh2020efficient,evgeniou2004regularized,zhang2021survey}, enables \mrmtl to excel at all levels of heterogeneity: in \cref{fig:rot-mnist}~(a), the preconditioning allows \mrmtl to match IFCA (optimal by construction) while local training (full personalization) does not benefit from the same preconditioning; in \cref{fig:rot-mnist}~(b, c, d), \mrmtl remains optimal across different levels of silo-specific heterogeneity (the 2nd layer) while the gains from warm-start gradually drop.
We argue that extensibility and flexibility are good properties that make \mrmtl a strong baseline, as heterogeneity in practical settings is likely less adversarial than what we presented.

\section{Analysis}
\label{sec:analysis}

\newcommand{\sigmadp}{\sigma_{\mathrm{DP}}}
\newcommand{\sigmaloc}{\sigma_{\mathrm{loc}}}

In this section we provide a theoretical analysis of \mrmtl under mean estimation as a simplified proxy for (single-round) FL using a Bayesian framework extending on~\cite{li2021ditto}.
We provide expressions for the Bayes optimal estimator $\mrmtl~(\lam^*)$ and describe how \mrmtl behaves with varying $\lam$ in relation to the personalization spectrum to characterize our observations from \cref{fig:mrmtl-bump}. Proofs and extensions are deferred to \cref{supp-sec:proofs}.

\textbf{Setup.}\enspace
We start with a total of $K$ silos where the $k$-th silo holds $n$ training samples $X_k \triangleq \{x_{k, i} \in \mathbb R\}_{i \in [n]}$, each normally distributed around a hidden center $w_k$ with variance $\sigma^2$; i.e.\ $x_{k, i} = w_k + z_{k, i}$ with $z_{k, i} \sim \mathcal N(0, \sigma^2)$.
To quantify heterogeneity, the silo centers $\{w_k\}_{k \in [K]}$ are also normally distributed around some unknown fixed meta-center $\theta$ with variance $\tau^2$; i.e.\ $w_k = \theta + z$ with $z \sim \mathcal N(0, \tau^2)$.
A large $\tau$ means that the silo centers are distant from each other and thus their local objectives are heterogeneous, and contrarily for a small $\tau$.
Our goal is for each silo $k$ to compute a sample-level private estimate of $w_k$ that minimizes the \textit{generalization} error (i.e.\ on unseen points from the same local distribution).
Each silo targets $(\eps, \delta)$ sample-level DP and runs the Gaussian mechanism with noise scale $\sigmadp= c\sqrt{2 \ln (1.25 / \delta)} / \eps$ and clipping bound $c$.\footnote{
  For simplicity, we start with the same $n$, $\sigma$, $\sigmadp$ for all silos and extend to silo-specific values in \cref{supp-sec:proofs}.
}
Under this setting, the \mrmtl objective for the $k$-th silo is
\begin{align} \label{eq:mrmtl-objective}
  h_k(w) = \tilde F_k(w) + \frac{\lambda}{2} \| w - \bar w \|_2^2.
\end{align}
Here, $\tilde F_k(w) \triangleq \frac{1}{2} (w - \frac{1}{n} (\xi_k + \sum_{i=1}^n x_{k, i} \cdot \min(1, c/ \| x_{k, i} \|_2 ) ) )^2$ is the local objective to privately estimate the mean of the local data points with privacy noise $\xi_k \sim \mathcal N(0, \sigmadp^2)$. Since the data are (sub-)Gaussian, we assume one can choose $c$ such that no clipping error is introduced w.h.p., so $\hat w_k \triangleq \argmin \tilde F_k(w) = \frac{1}{n} \left(\xi_k + \sum_i x_{k, i}\right)$ is the best local estimator. $\bar w = \frac{1}{K} \sum_{k} \hat w_k $ is the average estimator across silos, which is the same as the FedAvg estimator under mean estimation.
We also consider the \textit{external} average local estimators for silo $k$, defined as $\hat w_{\setminus k} \triangleq \frac{1}{K-1} \sum_{j \neq k} \hat w_j$.
The following lemma gives the best \textsf{MR-MTL} estimator $\hat w_k(\lambda)$ as a function of $\lam$.
\begin{lemma} \label{lem:wk-lambda}
  Let $\lambda \ge 0$ and ${\alpha = \frac{K+\lambda}{(1+\lambda) K} \in (1/K, 1]}$.
  The minimizer of $h_k(w)$ is given by
  \begin{align}
    \hat w_k(\lambda) = \alpha \cdot \hat w_k + (1 - \alpha) \cdot \hat w_{\setminus k}.
  \end{align}
\end{lemma}
\vspace{-0.5em}
Note that the best $\lambda$ is always 0 for \textit{training} error (i.e.\ estimating the empirical mean of the local data $\{x_{k, i}\}$); our hope is that with some $\lambda > 0$, $\hat w_k(\lambda)$ yields a better \textit{generalization} error.

We now present the main takeaways.
At a high level, the basis of our analysis relies on expressing the true center $w_k$ in terms of $\hat w_k$ and $\hat w_{\setminus k}$ conditioned on the local datasets $\smash{\{ X_k \}_{k \in [K]}}$.
Let
\begin{align}
  \sigmaloc^2 \triangleq \frac{\sigma^2}{n} + \frac{\sigmadp^2}{ n^2}
\end{align}
denote the ``local variance'' of $\hat w_k$ around $w_k$ due to both data sampling and privacy noise.

\textbf{Behavior of \mrmtl at optimal $\lam^*$.}\enspace
We first derive the following lemma using Lemma 11 of~\cite{mahdavifar2017global}.
\begin{lemma} \label{lem:wk-as-hatwk-hatwkminus}
  Given $\hat w_k$, $\hat w_{\setminus k}$, and $\smash{\{ X_k \}_{k \in [K]}}$, we can express $w_k = \mu_k + \zeta_k$, where $ \zeta_k \sim \mathcal N(0, \sigma_w^2)$,
  \begin{align}
    \sigma_w^2 &\triangleq \left( \frac{1}{\sigmaloc^2}+\frac{K-1}{K \tau^{2}+ \sigmaloc^2}  \right)^{-1}  \text{ and} \quad
    \mu_k \triangleq \sigma_{w}^{2} \left( \frac{1}{\sigmaloc^2 } \cdot \hat{w}_{k}+\frac{K-1 }{K \tau^{2} + \sigmaloc^2} \cdot \hat{w}_{\setminus k} \right).
  \end{align}
\end{lemma}
\vspace{-0.5em}

\cref{lem:wk-as-hatwk-hatwkminus} expresses the unobserved true silo centers $w_k$ in terms of the (private) empirical estimators $\hat w_k$ and $\hat w_{\setminus k}$. This expression requires conditioning on the datasets $X_k$ as they form the Markov blankets of $\hat w_k$.
Combining \cref{lem:wk-lambda} and \cref{lem:wk-as-hatwk-hatwkminus} gives the optimal $\lam$.
\begin{theorem}[Optimal \mrmtl estimate] \label{thm:optimal-lambda}
  The best $\lam^*$ for the generalization error is given by
  \begin{align}
    \lambda^{*}= \argmin_{\lambda} \mathbb{E} \left[ \left(w_{k}-\hat{w}_{k}(\lambda)\right)^{2} \mid \hat{w}_{k}, \hat{w}_{\setminus k}, \{ X_k \}_{k \in [K]} \right]
    = \frac{1}{n \tau^2} \left( \sigma^2 + \frac{\sigmadp^2}{n} \right).
  \end{align}
\end{theorem}
\vspace{-0.5em}
\cref{thm:optimal-lambda} suggests that there indeed exists an optimal point $\hat w(\lambda^*)$ on the personalization spectrum.
Moreover, $\lambda^*$ grows smoothly with stronger privacy ($\sigmadp^2 \to \infty$) to encourage silos to ``federate more'' with others. This was empirically observed in \cref{fig:mrmtl-bump}.
We now characterize the utility of $\hat w(\lambda^*)$.
\begin{corollary}[Optimal error with $\hat w(\lambda^*)$] \label{lem:optimal-error}
  The MSE of the optimal estimator $\hat w(\lambda^*)$ is given by
  \begin{align}
    \mathcal E^* \triangleq \mathbb{E} \left[ (w_{k}-\hat{w}_{k}(\lambda^*))^{2} \mid \hat{w}_{k}, \hat{w}_{\setminus k}, \{ X_k \}_{k \in [K]} \right] = \sigma_w^2 = \frac{\sigmaloc^2 (\sigmaloc^2 + K \tau^2) }{K(\sigmaloc^2 + \tau^2)}.
  \end{align}
  \vspace{-0.5em}
\end{corollary}
Note also that $\hat w(\lambda^*)$ is the MMSE estimator of $w_k$.
Using \cref{lem:optimal-error}, we can compare $\hat w_k(\lambda^*)$ against the endpoints of the personalization spectrum (local training and FedAvg) with the following propositions.
\begin{proposition}[Optimal error gap to local training] \label{thm:optimal-local-error-gap}
  Let $\mathcal E_{\mathrm{loc}} \triangleq \mathbb{E} \left[ (w_{k}-\hat{w}_{k} )^{2} \mid  X_k \right] = \sigmaloc^2$ be the error of the local estimate. Then, compared to the optimal estimator $\hat w(\lambda^*)$ (\cref{lem:optimal-error}), the local estimator incurs an additional error of
  \begin{align}
    { \Delta_\mathrm{loc} \triangleq \mathcal E_{\mathrm{loc}} - \mathcal E^* = \left( 1 - \frac{1}{K} \right)  \cdot \frac{\sigmaloc^4}{ \sigmaloc^2 + \tau^2} }.
  \end{align}
  \vspace{-0.25em}
\end{proposition}
\begin{proposition}[Optimal error gap to FedAvg] \label{thm:optimal-global-error-gap}
  Let $\mathcal E_{\mathrm{fed}} \triangleq \mathbb{E} \left[ (w_k - \bar w)^{2} \mid \{X_k\}_{k \in [K]} \right]$ be the error under FedAvg.
  Then, compared to the optimal estimator $\hat w(\lambda^*)$ (\cref{lem:optimal-error}), the FedAvg estimator incurs an additional error of
  \begin{align}
    \Delta_\mathrm{fed} \triangleq \mathcal E_{\mathrm{fed}} - \mathcal E^* = \left( 1 - \frac{1}{K} \right)  \cdot \frac{ \tau^4}{\sigmaloc^2 + \tau^2}.
  \end{align}
  \vspace{-0.5em}
\end{proposition}
Together,
\cref{thm:optimal-local-error-gap,thm:optimal-global-error-gap}
suggest that the effects of stronger privacy ($\sigmadp^2, \sigmaloc^2 \to \infty$) on how \mrmtl compares against the personalization endpoints are mixed, with the benefit of $\mrmtl$ \textit{increasing} against local training and \textit{diminishing} against FedAvg.
They also suggest that \mrmtl has an optimal utility advantage over both the endpoints when $\sigmaloc^2 \approx \tau^2$ and local training performs on par with FedAvg, and
the utility ``bump'' under privacy observed in \cref{fig:mrmtl-bump} can be viewed as a result of this balance.
It is worth noting that since the data variance $\sigma^2$ and heterogeneity $\tau^2$ are often fixed in practice, the freedom for silos to vary their privacy targets ($\eps$ and $\sigmadp^2$) makes the utility advantage of \mrmtl more flexible compared to non-private settings.

\textbf{Behavior of $\mrmtl$ as a function of $\lambda$.}\enspace
The above captures how $\mrmtl$ behaves at its optimum, but in \cref{fig:mrmtl-bump} we also observed that $\mrmtl$ has the desirable property that the utility cost from DP \textit{shrinks smoothly} with larger $\lambda$ (\cref{sec:method}).
\cref{lem:general-lambda-error} and \cref{thm:dp-error-gap} below provides a characterization.
\begin{lemma}[Error of $\hat w_k(\lambda)$] \label{lem:general-lambda-error}
  Let $\mathcal E(\lam) \triangleq \mathbb E\left[ (w_k - \hat w_k(\lambda))^2 \mid \hat{w}_{k}, \hat{w}_{\setminus k}, \{ X_k \}_{k \in [K]} \right]$ be the error of $\mrmtl$ as a function of $\lambda$. Then,
  \ifnum\neurips=1
  $\mathcal E(\lam)
  = \left( 1 - \frac{1}{K} \right) \frac{\sigmaloc^2 + \lambda^2 \tau^2}{(\lambda + 1)^2} + \frac{\sigmaloc^2}{K}$.
  \else
  \begin{align}
    \mathcal E(\lam) = \left( 1 - \frac{1}{K} \right) \frac{\sigmaloc^2 + \lambda^2 \tau^2}{(\lambda + 1)^2} + \frac{\sigmaloc^2}{K}.
  \end{align}
  \fi
  \vspace{-0.5em}
\end{lemma}
Using \cref{lem:general-lambda-error} we can now characterize how $\lam$ affects the utility cost from DP (recall from \cref{fig:finetune-gap,fig:mrmtl-bump} that federation helps with noise reduction).
As a side note, \cref{lem:general-lambda-error} also suggests that $\mrmtl$'s utility as a function of $\lambda$ would have a quasi-concave shape, as was empirically observed in \cref{fig:mrmtl-bump}. This could potentially help make heuristic or automated search over $\lambda$ easier.
\begin{theorem}[Private utility gap] \label{thm:dp-error-gap}
  Let $\hat w_k(\lambda)$ and $\hat w_k^\mathrm{DP}(\lambda)$ be the non-private and private estimate of $w_k$ with $\sigmaloc^2 \gets \sigma^2 / n$ and $\sigmaloc^2 \gets \sigma^2 / n + \sigmadp^2 / n^2$, respectively.
  Let $\mathcal E(\lam)$ and $\mathcal E^\mathrm{DP}(\lam)$ be the error of $\hat w_k(\lambda)$ and $\hat w_k^\mathrm{DP}(\lambda)$ respectively as in \cref{lem:general-lambda-error}.
  Let $\Delta_\mathrm{DP}(\lambda) \triangleq \mathcal E^\mathrm{DP}(\lambda) - \mathcal E(\lam)$ be the utility cost due to privacy as a function of $\lam$. Then,
  \ifnum\neurips=1
  ${\Delta_\mathrm{DP}(\lambda) = \left( 1 - \frac{1}{K} \right) \frac{1}{(\lambda + 1)^2}  \frac{\sigmadp^2}{n^2} + \frac{\sigmadp^2}{Kn^2}}$.
  \else
  \begin{align}
    {\Delta_\mathrm{DP}(\lambda) = \left( 1 - \frac{1}{K} \right) \frac{1}{(\lambda + 1)^2}  \frac{\sigmadp^2}{n^2} + \frac{\sigmadp^2}{Kn^2}}.
  \end{align}
  \fi
  \vspace{-0.7em}
\end{theorem}
\cref{thm:dp-error-gap} suggests that with a larger $\lam$, the utility cost from privacy can be smoothly mitigated by up to a factor of $K$, matching the empirical observation in \cref{fig:mrmtl-bump}.

\vspace{-0.4em}
\section{Discussions}
\label{sec:discussion}
\vspace{-0.4em}

\begin{figure}[t]
  \centering
  \includegraphics[width=\linewidth]{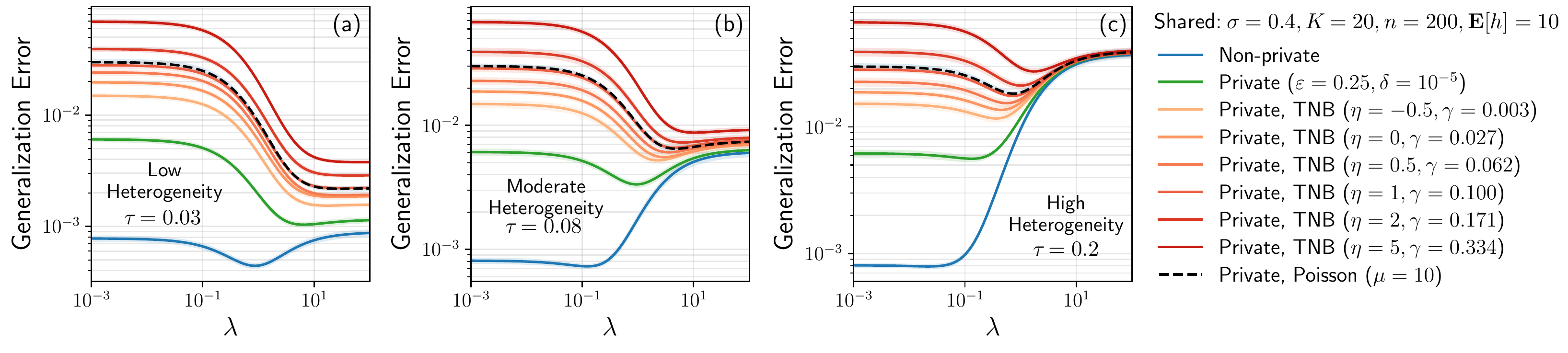}
  \vspace{-1.75em}
  \captionof{figure}{
    \textbf{Privacy costs of tuning $\lam$ on mean estimation} (setup follows \cref{sec:analysis}). Labels ``\textbf{Private}'' and ``\textbf{Non-Private}'' denote the errors of varying $\lam$ with and without silo-specific sample-level DP (the privacy cost of tuning $\lam$ is \textit{not} included). ``\textbf{Private, TNB/Poisson}''~\cite{papernot2022hyperparameter} denotes the same errors but accounts for the privacy cost of trying on average $\mathbf E[h] = 10$ values of $\lam$, with $h$ sampled from the truncated negative binomial distribution with parameters $\eta,\gamma$ / the Poisson distribution with parameter $\mu$ to arrive at the same $\mathbf E[h]$.
    To interpret, observe that the lowest points of ``\textbf{Private, TNB/Poisson}'' may be still higher than one of the endpoints of ``\textbf{Private}''.
  }
  \vspace{-1em}
  \label{fig:hparam-cost}
\end{figure}

In previous sections, we empirically and theoretically studied the benefits of the best personalization hyperparameter $\lam^*$ for $\mrmtl$, but it remains open as to how such $\lam^*$ may be obtained.
In this section, we take an honest look at the complications of deploying $\mrmtl$ through the lens of the \textit{privacy cost} of finding $\lam^*$.
There are in general several approaches: (1) a non-adaptive search (e.g.\ grid/random search~\cite{bergstra2012random}); (2) an adaptive search (e.g.\ grad student descent); or (3) an online estimation during training (e.g.\ \cite{van2018three,andrew2021differentially,pichapati2019adaclip}).
Here, we focus on approach (1) since it is generic to all personalization methods and is a setting for which we have the best privacy accounting tools~\cite{liu2019private,papernot2022hyperparameter} to our knowledge.
We defer technical details and further discussions to \cref{supp-sec:hparam-tuning}.
Note that while we focus on \mrmtl, our reasoning in principle extends to all personalization methods whose advantage depends on having the best hyperparameter(s).

Recall that for a typical tuning procedure, a baseline algorithm $M$ is executed $h$ times with different hyperparameters and the best result is recorded.
The work of~\cite{liu2019private,papernot2022hyperparameter} shows that, with a constant $h$, there exists $M$ that satisfies $(\eps, \delta = 0)$-DP where the output of tuning is not $(\tilde\eps, 0)$-DP for any $\tilde\eps < h\eps$, with analogous negative results for R\'enyi DP (thus also for $\delta > 0$). This implies that naive tuning (as done in practice) can incur a prohibitive privacy overhead and obliterates the utility advantage of \mrmtl ($\lam^*$) over local training/FedAvg.
Instead, by making $h$ \textit{random}, we can make $\tilde\eps$ \textit{constant} w.r.t.\ $h$ or at most $\tilde\eps \le O( \log \mathbf E[h])$~\cite{liu2019private,papernot2022hyperparameter}. However, using the simplified setting of mean estimation (\cref{sec:analysis}),
we find that even with this improved randomized protocol,
there exist scenarios (\cref{fig:hparam-cost}) where the realistic cost of trying a moderate $\mathbf E[h] = 10$ values of $\lam$ may significantly diminish, or even \textit{outweigh}, the utility advantage of $\lam^*$ over local training and FedAvg---that is, we might be better off by \textit{not privately tuning $\lam$ at all}.

The above has several important implications.
On the negative side, it suggests that the true efficacy of \mrmtl can be smaller in practice.
Moreover, it raises the broader open question of whether the emerging privacy-heterogeneity cost tradeoff is best balanced by model personalization, as many existing methods including \mrmtl inherently require at least one hyperparameter to specify ``how much to personalize'' for general utility improvements over local training and FedAvg.
Alternatively, the hyperparameter(s) may be estimated \textit{during} training (approach (3) in the first paragraph), though such procedures may not be general and/or scalable and may need to be tailored to the specific personalization method.
On the positive side, it is unclear whether the choice of $\lam$ can meaningfully leak privacy in practice. \mrmtl may also be viewed favorably as a strong baseline since it only needs one hyperparameter to attain its benefits, while other existing methods that require more tuning will incur even larger privacy costs from hyperparameter tuning.

\section{Concluding Remarks}

In this work, we revisit the application of differential privacy in cross-silo FL. We examine silo-specific sample-level DP as a more appropriate privacy notion for cross-silo FL, and we point out several meaningful ways in which it differs from client-level DP commonly studied under the cross-device setting, particularly when analyzing tensions between privacy, utility, and heterogeneity.
We explore and establish baselines under this privacy setting and identify desirable properties for a personalization method for balancing an emerging tradeoff between utility costs from privacy and heterogeneity.
We then analyze a simple, promising method (\mrmtl) and discuss key open questions for the area at large.
Some future directions include (1) extending the privacy model to cases where data subjects have multiple records across silos, (2) extending our theoretical characterization to deep learning cases or performing a large-scale empirical study, and (3) developing auto-tuning algorithms for personalization hyperparameters with minimal privacy overhead.

\textbf{Acknowledgements.}
We thank Sebastian Caldas, Tian Li, Yash Savani, Amrith Setlur, and Peter Kairouz for helpful discussions and feedback and Thomas Steinke for guidance on implementing privacy accounting for hyperparameter tuning~\cite{papernot2022hyperparameter}.
This work was supported in part by the NSF Grants IIS1838017 and IIS2145670, a Meta Faculty Award, an Apple Faculty Award, the Intel Private AI Center, and the CONIX Research Center. ZSW was supported in part by the NSF Award CNS2120667.
Any opinions, findings, and conclusions or recommendations expressed herein are those of the author(s) and do not necessarily reflect the NSF or any other funding agency.

\clearpage

{\small
\bibliographystyle{plain}
\bibliography{main.bib}

\begin{thebibliography}{100}

\bibitem{abadi2016deep}
Martin Abadi, Andy Chu, Ian Goodfellow, H~Brendan McMahan, Ilya Mironov, Kunal
  Talwar, and Li~Zhang.
\newblock Deep learning with differential privacy.
\newblock In {\em Proceedings of the 2016 ACM SIGSAC conference on computer and
  communications security}, pages 308--318, 2016.

\bibitem{adni-dataset}
The Alzheimer’s Disease Neuroimaging~Initiative (ADNI).
\newblock The alzheimer’s disease neuroimaging initiative (adni).
\newblock adni.loni.usc.edu, 05 2022.
\newblock \url{https://adni.loni.usc.edu/}.

\bibitem{agarwal2020federated}
Alekh Agarwal, John Langford, and Chen-Yu Wei.
\newblock Federated residual learning.
\newblock {\em arXiv preprint arXiv:2003.12880}, 2020.

\bibitem{agarwal2021skellam}
Naman Agarwal, Peter Kairouz, and Ziyu Liu.
\newblock The skellam mechanism for differentially private federated learning.
\newblock {\em Advances in Neural Information Processing Systems}, 34, 2021.

\bibitem{alaggan2016heterogeneous}
Mohammad Alaggan, S{\'e}bastien Gambs, and Anne-Marie Kermarrec.
\newblock Heterogeneous differential privacy.
\newblock {\em Journal of Privacy and Confidentiality}, 7(2), 2016.

\bibitem{aldaghri2021feo2}
Nasser Aldaghri, Hessam Mahdavifar, and Ahmad Beirami.
\newblock Feo2: Federated learning with opt-out differential privacy.
\newblock In {\em NeurIPS 2021 Workshop on New Frontiers in Federated Learning:
  Privacy, Fairness, Robustness, Personalization and Data Ownership}, 2021.

\bibitem{an1996effects}
Guozhong An.
\newblock The effects of adding noise during backpropagation training on a
  generalization performance.
\newblock {\em Neural computation}, 8(3):643--674, 1996.

\bibitem{andrew2021differentially}
Galen Andrew, Om~Thakkar, Brendan McMahan, and Swaroop Ramaswamy.
\newblock Differentially private learning with adaptive clipping.
\newblock {\em Advances in Neural Information Processing Systems}, 34, 2021.

\bibitem{argyriou2008convex}
Andreas Argyriou, Theodoros Evgeniou, and Massimiliano Pontil.
\newblock Convex multi-task feature learning.
\newblock {\em Machine learning}, 73(3):243--272, 2008.

\bibitem{barzilai2015convex}
Aviad Barzilai and Koby Crammer.
\newblock Convex multi-task learning by clustering.
\newblock In {\em Artificial Intelligence and Statistics}, pages 65--73. PMLR,
  2015.

\bibitem{bassily2014private}
Raef Bassily, Adam Smith, and Abhradeep Thakurta.
\newblock Private empirical risk minimization: Efficient algorithms and tight
  error bounds.
\newblock In {\em IEEE Symposium on Foundations of Computer Science}, 2014.

\bibitem{bergstra2012random}
James Bergstra and Yoshua Bengio.
\newblock Random search for hyper-parameter optimization.
\newblock {\em Journal of machine learning research}, 13(2), 2012.

\bibitem{bietti2022personalization}
Alberto Bietti, Chen-Yu Wei, Miroslav Dudik, John Langford, and Zhiwei~Steven
  Wu.
\newblock Personalization improves privacy-accuracy tradeoffs in federated
  optimization.
\newblock In {\em International Conference on Machine Learning}. PMLR, 2022.

\bibitem{jax2018github}
James Bradbury, Roy Frostig, Peter Hawkins, Matthew~James Johnson, Chris Leary,
  Dougal Maclaurin, George Necula, Adam Paszke, Jake Vander{P}las, Skye
  Wanderman-{M}ilne, and Qiao Zhang.
\newblock {JAX}: composable transformations of {P}ython+{N}um{P}y programs,
  2018.

\bibitem{bun2016concentrated}
Mark Bun and Thomas Steinke.
\newblock Concentrated differential privacy: Simplifications, extensions, and
  lower bounds.
\newblock In {\em Theory of Cryptography Conference}, pages 635--658. Springer,
  2016.

\bibitem{caldas2018leaf}
Sebastian Caldas, Sai Meher~Karthik Duddu, Peter Wu, Tian Li, Jakub
  Kone{\v{c}}n{\`y}, H~Brendan McMahan, Virginia Smith, and Ameet Talwalkar.
\newblock Leaf: A benchmark for federated settings.
\newblock In {\em NeurIPS 2019 Workshop on Federated Learning for Data Privacy
  and Confidentiality}, 2019.

\bibitem{canonne2020discrete}
Cl{\'e}ment~L Canonne, Gautam Kamath, and Thomas Steinke.
\newblock The discrete gaussian for differential privacy.
\newblock {\em Advances in Neural Information Processing Systems},
  33:15676--15688, 2020.

\bibitem{chen2022poisson}
Wei-Ning Chen, Ayfer Ozgur, and Peter Kairouz.
\newblock The poisson binomial mechanism for unbiased federated learning with
  secure aggregation.
\newblock In {\em International Conference on Machine Learning}, pages
  3490--3506. PMLR, 2022.

\bibitem{cheng2021fine}
Gary Cheng, Karan Chadha, and John Duchi.
\newblock Fine-tuning is fine in federated learning.
\newblock {\em arXiv preprint arXiv:2108.07313}, 2021.

\bibitem{cheu2021differential}
Albert Cheu.
\newblock Differential privacy in the shuffle model: A survey of separations.
\newblock {\em arXiv preprint arXiv:2107.11839}, 2021.

\bibitem{cheu2019distributed}
Albert Cheu, Adam Smith, Jonathan Ullman, David Zeber, and Maxim Zhilyaev.
\newblock Distributed differential privacy via shuffling.
\newblock In {\em Annual International Conference on the Theory and
  Applications of Cryptographic Techniques}, pages 375--403. Springer, 2019.

\bibitem{cho2021personalized}
Yae~Jee Cho, Jianyu Wang, Tarun Chiruvolu, and Gauri Joshi.
\newblock Personalized federated learning for heterogeneous clients with
  clustered knowledge transfer.
\newblock {\em arXiv preprint arXiv:2109.08119}, 2021.

\bibitem{chowdhery2022palm}
Aakanksha Chowdhery, Sharan Narang, Jacob Devlin, Maarten Bosma, Gaurav Mishra,
  Adam Roberts, Paul Barham, Hyung~Won Chung, Charles Sutton, Sebastian
  Gehrmann, et~al.
\newblock Palm: Scaling language modeling with pathways.
\newblock {\em arXiv preprint arXiv:2204.02311}, 2022.

\bibitem{deng2020adaptive}
Yuyang Deng, Mohammad~Mahdi Kamani, and Mehrdad Mahdavi.
\newblock Adaptive personalized federated learning.
\newblock {\em arXiv preprint arXiv:2003.13461}, 2020.

\bibitem{gaussian-dp}
Jinshuo Dong, Aaron Roth, and Weijie~J. Su.
\newblock Gaussian differential privacy.
\newblock {\em Journal of the Royal Statistical Society: Series B (Statistical
  Methodology)}, 84(1):3--37, 2022.

\bibitem{duarte2004vehicle}
Marco~F Duarte and Yu~Hen Hu.
\newblock Vehicle classification in distributed sensor networks.
\newblock {\em Journal of Parallel and Distributed Computing}, 64(7):826--838,
  2004.

\bibitem{dwork2006calibrating}
Cynthia Dwork, Frank McSherry, Kobbi Nissim, and Adam Smith.
\newblock Calibrating noise to sensitivity in private data analysis.
\newblock In {\em Theory of cryptography conference}, pages 265--284. Springer,
  2006.

\bibitem{dwork2014algorithmic}
Cynthia Dwork, Aaron Roth, et~al.
\newblock The algorithmic foundations of differential privacy.
\newblock {\em Found. Trends Theor. Comput. Sci.}, 9(3-4):211--407, 2014.

\bibitem{erlingsson2019amplification}
{\'U}lfar Erlingsson, Vitaly Feldman, Ilya Mironov, Ananth Raghunathan, Kunal
  Talwar, and Abhradeep Thakurta.
\newblock Amplification by shuffling: From local to central differential
  privacy via anonymity.
\newblock In {\em Proceedings of the Thirtieth Annual ACM-SIAM Symposium on
  Discrete Algorithms}, pages 2468--2479. SIAM, 2019.

\bibitem{evgeniou2004regularized}
Theodoros Evgeniou and Massimiliano Pontil.
\newblock Regularized multi--task learning.
\newblock In {\em Proceedings of the tenth ACM SIGKDD international conference
  on Knowledge discovery and data mining}, pages 109--117, 2004.

\bibitem{fallah2020personalized}
Alireza Fallah, Aryan Mokhtari, and Asuman Ozdaglar.
\newblock Personalized federated learning with theoretical guarantees: A
  model-agnostic meta-learning approach.
\newblock {\em Advances in Neural Information Processing Systems},
  33:3557--3568, 2020.

\bibitem{geyer2017differentially}
Robin~C Geyer, Tassilo Klein, and Moin Nabi.
\newblock Differentially private federated learning: A client level
  perspective.
\newblock In {\em NIPS 2017 Workshop: Machine Learning on the Phone and other
  Consumer Devices}, 2017.

\bibitem{ghosh2020efficient}
Avishek Ghosh, Jichan Chung, Dong Yin, and Kannan Ramchandran.
\newblock An efficient framework for clustered federated learning.
\newblock {\em Advances in Neural Information Processing Systems},
  33:19586--19597, 2020.

\bibitem{girgis2021shuffled}
Antonious Girgis, Deepesh Data, Suhas Diggavi, Peter Kairouz, and
  Ananda~Theertha Suresh.
\newblock Shuffled model of differential privacy in federated learning.
\newblock In {\em International Conference on Artificial Intelligence and
  Statistics}, pages 2521--2529. PMLR, 2021.

\bibitem{goldstein1991multilevel}
Harvey Goldstein.
\newblock Multilevel modelling of survey data.
\newblock {\em Journal of the Royal Statistical Society. Series D (The
  Statistician)}, 40(2):235--244, 1991.

\bibitem{gong2012robust}
Pinghua Gong, Jieping Ye, and Changshui Zhang.
\newblock Robust multi-task feature learning.
\newblock In {\em Proceedings of the 18th ACM SIGKDD international conference
  on Knowledge discovery and data mining}, pages 895--903, 2012.

\bibitem{hanzely2020lower}
Filip Hanzely, Slavom{\'\i}r Hanzely, Samuel Horv{\'a}th, and Peter
  Richt{\'a}rik.
\newblock Lower bounds and optimal algorithms for personalized federated
  learning.
\newblock {\em Advances in Neural Information Processing Systems},
  33:2304--2315, 2020.

\bibitem{hanzely2020federated}
Filip Hanzely and Peter Richt{\'a}rik.
\newblock Federated learning of a mixture of global and local models.
\newblock {\em arXiv preprint arXiv:2002.05516}, 2020.

\bibitem{harris2020array}
Charles~R. Harris, K.~Jarrod Millman, St{\'{e}}fan~J. van~der Walt, Ralf
  Gommers, Pauli Virtanen, David Cournapeau, Eric Wieser, Julian Taylor,
  Sebastian Berg, Nathaniel~J. Smith, Robert Kern, Matti Picus, Stephan Hoyer,
  Marten~H. van Kerkwijk, Matthew Brett, Allan Haldane, Jaime~Fern{\'{a}}ndez
  del R{\'{i}}o, Mark Wiebe, Pearu Peterson, Pierre G{\'{e}}rard-Marchant,
  Kevin Sheppard, Tyler Reddy, Warren Weckesser, Hameer Abbasi, Christoph
  Gohlke, and Travis~E. Oliphant.
\newblock Array programming with {NumPy}.
\newblock {\em Nature}, 585(7825):357--362, September 2020.

\bibitem{he2021masked}
Kaiming He, Xinlei Chen, Saining Xie, Yanghao Li, Piotr Doll{\'a}r, and Ross
  Girshick.
\newblock Masked autoencoders are scalable vision learners.
\newblock In {\em Proceedings of the IEEE conference on computer vision and
  pattern recognition}, 2022.

\bibitem{heikkila2020differentially}
Mikko~A Heikkil{\"a}, Antti Koskela, Kana Shimizu, Samuel Kaski, and Antti
  Honkela.
\newblock Differentially private cross-silo federated learning.
\newblock In {\em Privacy Preserving Machine Learning (PPML) and Privacy in
  Machine learning (PriML) Joint Workshop at NeurIPS 2020}, 2020.

\bibitem{haiku2020github}
Tom Hennigan, Trevor Cai, Tamara Norman, and Igor Babuschkin.
\newblock {H}aiku: {S}onnet for {JAX}, 2020.

\bibitem{hsu2019measuring}
Tzu-Ming~Harry Hsu, Hang Qi, and Matthew Brown.
\newblock Measuring the effects of non-identical data distribution for
  federated visual classification.
\newblock In {\em Workshop on Federated Learning for Data Privacy and
  Confidentiality, NeurIPS 2019}, 2019.

\bibitem{hu2021private}
Shengyuan Hu, Zhiwei~Steven Wu, and Virginia Smith.
\newblock Private multi-task learning: Formulation and applications to
  federated learning.
\newblock {\em arXiv preprint arXiv:2108.12978}, 2021.

\bibitem{jalali2010dirty}
Ali Jalali, Sujay Sanghavi, Chao Ruan, and Pradeep Ravikumar.
\newblock A dirty model for multi-task learning.
\newblock {\em Advances in neural information processing systems}, 23, 2010.

\bibitem{jiang2019improving}
Yihan Jiang, Jakub Kone{\v{c}}n{\`y}, Keith Rush, and Sreeram Kannan.
\newblock Improving federated learning personalization via model agnostic meta
  learning.
\newblock {\em arXiv preprint arXiv:1909.12488}, 2019.

\bibitem{jorgensen2015conservative}
Zach Jorgensen, Ting Yu, and Graham Cormode.
\newblock Conservative or liberal? personalized differential privacy.
\newblock In {\em 2015 IEEE 31St international conference on data engineering},
  pages 1023--1034. IEEE, 2015.

\bibitem{kairouz2021distributed}
Peter Kairouz, Ziyu Liu, and Thomas Steinke.
\newblock The distributed discrete gaussian mechanism for federated learning
  with secure aggregation.
\newblock In {\em International Conference on Machine Learning}, pages
  5201--5212. PMLR, 2021.

\bibitem{kairouz2021practical}
Peter Kairouz, Brendan McMahan, Shuang Song, Om~Thakkar, Abhradeep Thakurta,
  and Zheng Xu.
\newblock Practical and private (deep) learning without sampling or shuffling.
\newblock In {\em International Conference on Machine Learning}, pages
  5213--5225. PMLR, 2021.

\bibitem{kairouz2021advances}
Peter Kairouz, H~Brendan McMahan, Brendan Avent, Aur{\'e}lien Bellet, Mehdi
  Bennis, Arjun~Nitin Bhagoji, Kallista Bonawitz, Zachary Charles, Graham
  Cormode, Rachel Cummings, et~al.
\newblock Advances and open problems in federated learning.
\newblock {\em Foundations and Trends{\textregistered} in Machine Learning},
  14(1--2):1--210, 2021.

\bibitem{kanani2021private}
Pallika Kanani, Virendra~J Marathe, Daniel Peterson, Rave Harpaz, and Steve
  Bright.
\newblock Private cross-silo federated learning for extracting vaccine adverse
  event mentions.
\newblock In {\em Joint European Conference on Machine Learning and Knowledge
  Discovery in Databases}, pages 490--505. Springer, 2021.

\bibitem{karimireddy2020scaffold}
Sai~Praneeth Karimireddy, Satyen Kale, Mehryar Mohri, Sashank Reddi, Sebastian
  Stich, and Ananda~Theertha Suresh.
\newblock Scaffold: Stochastic controlled averaging for federated learning.
\newblock In {\em International Conference on Machine Learning}, pages
  5132--5143. PMLR, 2020.

\bibitem{krizhevsky2009learning}
A~Krizhevsky.
\newblock Learning multiple layers of features from tiny images.
\newblock {\em Master's thesis, University of Toronto}, 2009.

\bibitem{lecun1998gradient}
Yann LeCun, L{\'e}on Bottou, Yoshua Bengio, and Patrick Haffner.
\newblock Gradient-based learning applied to document recognition.
\newblock {\em Proceedings of the IEEE}, 86(11):2278--2324, 1998.

\bibitem{levy2021learning}
Daniel Levy, Ziteng Sun, Kareem Amin, Satyen Kale, Alex Kulesza, Mehryar Mohri,
  and Ananda~Theertha Suresh.
\newblock Learning with user-level privacy.
\newblock {\em Advances in Neural Information Processing Systems}, 34, 2021.

\bibitem{li2019differentially}
Jeffrey Li, Mikhail Khodak, Sebastian Caldas, and Ameet Talwalkar.
\newblock Differentially private meta-learning.
\newblock In {\em International Conference on Learning Representations}, 2020.

\bibitem{li2021ditto}
Tian Li, Shengyuan Hu, Ahmad Beirami, and Virginia Smith.
\newblock Ditto: Fair and robust federated learning through personalization.
\newblock In {\em International Conference on Machine Learning}, pages
  6357--6368. PMLR, 2021.

\bibitem{li2020federated}
Tian Li, Anit~Kumar Sahu, Ameet Talwalkar, and Virginia Smith.
\newblock Federated learning: Challenges, methods, and future directions.
\newblock {\em IEEE Signal Processing Magazine}, 37(3):50--60, 2020.

\bibitem{fedprox}
Tian Li, Anit~Kumar Sahu, Manzil Zaheer, Maziar Sanjabi, Ameet Talwalkar, and
  Virginia Smith.
\newblock Federated optimization in heterogeneous networks.
\newblock {\em Proceedings of Machine Learning and Systems}, 2:429--450, 2020.

\bibitem{Li2020On}
Xiang Li, Kaixuan Huang, Wenhao Yang, Shusen Wang, and Zhihua Zhang.
\newblock On the convergence of fedavg on non-iid data.
\newblock In {\em International Conference on Learning Representations}, 2020.

\bibitem{liang2020think}
Paul~Pu Liang, Terrance Liu, Liu Ziyin, Nicholas~B Allen, Randy~P Auerbach,
  David Brent, Ruslan Salakhutdinov, and Louis-Philippe Morency.
\newblock Think locally, act globally: Federated learning with local and global
  representations.
\newblock In {\em NeurIPS 2019 Workshop on Federated Learning}, 2020.

\bibitem{liu2019private}
Jingcheng Liu and Kunal Talwar.
\newblock Private selection from private candidates.
\newblock In {\em Proceedings of the 51st Annual ACM SIGACT Symposium on Theory
  of Computing}, pages 298--309, 2019.

\bibitem{proj-fedavg}
Junxu Liu, Jian Lou, Li~Xiong, Jinfei Liu, and Xiaofeng Meng.
\newblock Projected federated averaging with heterogeneous differential
  privacy.
\newblock In {\em International Conference on Very Large Databases}. VLDB
  Endowment, 2022.

\bibitem{liu2017distributed}
Sulin Liu, Sinno~Jialin Pan, and Qirong Ho.
\newblock Distributed multi-task relationship learning.
\newblock In {\em Proceedings of the 23rd ACM SIGKDD International Conference
  on Knowledge Discovery and Data Mining}, pages 937--946, 2017.

\bibitem{liu2022convnet}
Zhuang Liu, Hanzi Mao, Chao-Yuan Wu, Christoph Feichtenhofer, Trevor Darrell,
  and Saining Xie.
\newblock A convnet for the 2020s.
\newblock In {\em Proceedings of the IEEE/CVF Conference on Computer Vision and
  Pattern Recognition}, pages 11976--11986, 2022.

\bibitem{lowy2021private}
Andrew Lowy and Meisam Razaviyayn.
\newblock Private federated learning without a trusted server: Optimal
  algorithms for convex losses.
\newblock {\em arXiv preprint arXiv:2106.09779}, 2021.

\bibitem{lu2020multi}
Linpeng Lu and Ning Ding.
\newblock Multi-party private set intersection in vertical federated learning.
\newblock In {\em 2020 IEEE 19th International Conference on Trust, Security
  and Privacy in Computing and Communications (TrustCom)}, pages 707--714.
  IEEE, 2020.

\bibitem{mahdavifar2017global}
Hessam Mahdavifar, Ahmad Beirami, Behrouz Touri, and Jeff~S Shamma.
\newblock Global games with noisy information sharing.
\newblock {\em IEEE Transactions on Signal and Information Processing over
  Networks}, 4(3):497--509, 2017.

\bibitem{mansour2020three}
Yishay Mansour, Mehryar Mohri, Jae Ro, and Ananda~Theertha Suresh.
\newblock Three approaches for personalization with applications to federated
  learning.
\newblock {\em arXiv preprint arXiv:2002.10619}, 2020.

\bibitem{mckenna2020permute}
Ryan McKenna and Daniel~R Sheldon.
\newblock Permute-and-flip: A new mechanism for differentially private
  selection.
\newblock {\em Advances in Neural Information Processing Systems}, 33:193--203,
  2020.

\bibitem{mcmahan2017communication}
Brendan McMahan, Eider Moore, Daniel Ramage, Seth Hampson, and Blaise~Aguera
  y~Arcas.
\newblock Communication-efficient learning of deep networks from decentralized
  data.
\newblock In {\em Artificial intelligence and statistics}, pages 1273--1282.
  PMLR, 2017.

\bibitem{mcmahan2018general}
H~Brendan McMahan, Galen Andrew, Ulfar Erlingsson, Steve Chien, Ilya Mironov,
  Nicolas Papernot, and Peter Kairouz.
\newblock A general approach to adding differential privacy to iterative
  training procedures.
\newblock In {\em NeurIPS 2018 Privacy Preserving Machine Learning Workshop},
  2018.

\bibitem{mcmahan2018learning}
H~Brendan McMahan, Daniel Ramage, Kunal Talwar, and Li~Zhang.
\newblock Learning differentially private recurrent language models.
\newblock In {\em International Conference on Learning Representations}, 2018.

\bibitem{brendan2018learning}
H.~Brendan McMahan, Daniel Ramage, Kunal Talwar, and Li~Zhang.
\newblock Learning differentially private recurrent language models.
\newblock In {\em International Conference on Learning Representations}, 2018.

\bibitem{mcsherry2007mechanism}
Frank McSherry and Kunal Talwar.
\newblock Mechanism design via differential privacy.
\newblock In {\em 48th Annual IEEE Symposium on Foundations of Computer Science
  (FOCS'07)}, pages 94--103. IEEE, 2007.

\bibitem{mcsherry2009privacy}
Frank~D McSherry.
\newblock Privacy integrated queries: an extensible platform for
  privacy-preserving data analysis.
\newblock In {\em Proceedings of the 2009 ACM SIGMOD International Conference
  on Management of data}, pages 19--30, 2009.

\bibitem{mironov2017renyi}
Ilya Mironov.
\newblock R{\'e}nyi differential privacy.
\newblock In {\em 2017 IEEE 30th computer security foundations symposium
  (CSF)}, pages 263--275. IEEE, 2017.

\bibitem{mironov2019r}
Ilya Mironov, Kunal Talwar, and Li~Zhang.
\newblock R\'enyi differential privacy of the sampled gaussian mechanism.
\newblock {\em arXiv preprint arXiv:1908.10530}, 2019.

\bibitem{papernot2022hyperparameter}
Nicolas Papernot and Thomas Steinke.
\newblock Hyperparameter tuning with renyi differential privacy.
\newblock In {\em International Conference on Learning Representations}, 2022.

\bibitem{pichapati2019adaclip}
Venkatadheeraj Pichapati, Ananda~Theertha Suresh, Felix~X Yu, Sashank~J Reddi,
  and Sanjiv Kumar.
\newblock Adaclip: Adaptive clipping for private sgd.
\newblock {\em arXiv preprint arXiv:1908.07643}, 2019.

\bibitem{qu2021experimental}
Liangqiong Qu, Niranjan Balachandar, and Daniel~L Rubin.
\newblock An experimental study of data heterogeneity in federated learning
  methods for medical imaging.
\newblock {\em arXiv preprint arXiv:2107.08371}, 2021.

\bibitem{rahman2015unintrusive}
Shah~Atiqur Rahman, Christopher Merck, Yuxiao Huang, and Samantha Kleinberg.
\newblock Unintrusive eating recognition using google glass.
\newblock In {\em 2015 9th International Conference on Pervasive Computing
  Technologies for Healthcare (PervasiveHealth)}, pages 108--111. IEEE, 2015.

\bibitem{ramaswamy2020training}
Swaroop Ramaswamy, Om~Thakkar, Rajiv Mathews, Galen Andrew, H~Brendan McMahan,
  and Fran{\c{c}}oise Beaufays.
\newblock Training production language models without memorizing user data.
\newblock {\em arXiv preprint arXiv:2009.10031}, 2020.

\bibitem{dalle}
Aditya Ramesh, Prafulla Dhariwal, Alex Nichol, Casey Chu, and Mark Chen.
\newblock Hierarchical text-conditional image generation with clip latents.
\newblock Technical report, OpenAI, 2022.

\bibitem{reddi2021adaptive}
Sashank~J Reddi, Zachary Charles, Manzil Zaheer, Zachary Garrett, Keith Rush,
  Jakub Kone{\v{c}}n{\`y}, Sanjiv Kumar, and Hugh~Brendan McMahan.
\newblock Adaptive federated optimization.
\newblock In {\em International Conference on Learning Representations}, 2021.

\bibitem{reith2020application}
F~Reith, ME~Koran, G~Davidzon, and G~Zaharchuk.
\newblock Application of deep learning to predict standardized uptake value
  ratio and amyloid status on 18f-florbetapir pet using adni data.
\newblock {\em American Journal of Neuroradiology}, 41(6):980--986, 2020.

\bibitem{DPorg-exponential-mechanism-bounded-range}
Ryan Rogers and Thomas Steinke.
\newblock A better privacy analysis of the exponential mechanism.
\newblock DifferentialPrivacy.org, 07 2021.
\newblock
  \url{https://differentialprivacy.org/exponential-mechanism-bounded-range/}.

\bibitem{rognvaldsson1994langevin}
Thorsteinn R{\"o}gnvaldsson.
\newblock On langevin updating in multilayer perceptrons.
\newblock {\em Neural computation}, 6(5):916--926, 1994.

\bibitem{sattler2020clustered}
Felix Sattler, Klaus-Robert M{\"u}ller, and Wojciech Samek.
\newblock Clustered federated learning: Model-agnostic distributed multitask
  optimization under privacy constraints.
\newblock {\em IEEE transactions on neural networks and learning systems},
  32(8):3710--3722, 2020.

\bibitem{shamsian2021personalized}
Aviv Shamsian, Aviv Navon, Ethan Fetaya, and Gal Chechik.
\newblock Personalized federated learning using hypernetworks.
\newblock In {\em International Conference on Machine Learning}, pages
  9489--9502. PMLR, 2021.

\bibitem{smith2017federated}
Virginia Smith, Chao-Kai Chiang, Maziar Sanjabi, and Ameet~S Talwalkar.
\newblock Federated multi-task learning.
\newblock {\em Advances in neural information processing systems}, 30, 2017.

\bibitem{song2013stochastic}
Shuang Song, Kamalika Chaudhuri, and Anand~D Sarwate.
\newblock Stochastic gradient descent with differentially private updates.
\newblock In {\em IEEE Global Conference on Signal and Information Processing},
  2013.

\bibitem{subramani2021enabling}
Pranav Subramani, Nicholas Vadivelu, and Gautam Kamath.
\newblock Enabling fast differentially private sgd via just-in-time compilation
  and vectorization.
\newblock {\em Advances in Neural Information Processing Systems}, 34, 2021.

\bibitem{t2020personalized}
Canh T~Dinh, Nguyen Tran, and Josh Nguyen.
\newblock Personalized federated learning with moreau envelopes.
\newblock {\em Advances in Neural Information Processing Systems},
  33:21394--21405, 2020.

\bibitem{truex2019hybrid}
Stacey Truex, Nathalie Baracaldo, Ali Anwar, Thomas Steinke, Heiko Ludwig, Rui
  Zhang, and Yi~Zhou.
\newblock A hybrid approach to privacy-preserving federated learning.
\newblock In {\em Proceedings of the 12th ACM workshop on artificial
  intelligence and security}, pages 1--11, 2019.

\bibitem{Vaid2020.08.11.20172809}
Akhil Vaid, Suraj~K Jaladanki, Jie Xu, Shelly Teng, Arvind Kumar, Samuel Lee,
  Sulaiman Somani, Ishan Paranjpe, Jessica K~De Freitas, Tingyi Wanyan, Kipp~W
  Johnson, Mesude Bicak, Eyal Klang, Young~Joon Kwon, Anthony Costa, Shan Zhao,
  Riccardo Miotto, Alexander~W Charney, Erwin B{\"o}ttinger, Zahi~A Fayad,
  Girish~N Nadkarni, Fei Wang, and Benjamin~S Glicksberg.
\newblock Federated learning of electronic health records improves mortality
  prediction in patients hospitalized with covid-19.
\newblock {\em medRxiv}, 2020.

\bibitem{van2018three}
Koen~Lennart van~der Veen, Ruben Seggers, Peter Bloem, and Giorgio Patrini.
\newblock Three tools for practical differential privacy.
\newblock In {\em Privacy Preserving Machine Learning (PPML) Workshop at
  NeurIPS 2018}, 2018.

\bibitem{vaswani2017attention}
Ashish Vaswani, Noam Shazeer, Niki Parmar, Jakob Uszkoreit, Llion Jones,
  Aidan~N Gomez, {\L}ukasz Kaiser, and Illia Polosukhin.
\newblock Attention is all you need.
\newblock {\em Advances in neural information processing systems}, 30, 2017.

\bibitem{wang2019federated}
Kangkang Wang, Rajiv Mathews, Chlo{\'e} Kiddon, Hubert Eichner, Fran{\c{c}}oise
  Beaufays, and Daniel Ramage.
\newblock Federated evaluation of on-device personalization.
\newblock {\em arXiv preprint arXiv:1910.10252}, 2019.

\bibitem{welling2011bayesian}
Max Welling and Yee~W Teh.
\newblock Bayesian learning via stochastic gradient langevin dynamics.
\newblock In {\em Proceedings of the 28th international conference on machine
  learning (ICML-11)}, pages 681--688. Citeseer, 2011.

\bibitem{wen2022fishing}
Yuxin Wen, Jonas Geiping, Liam Fowl, Micah Goldblum, and Tom Goldstein.
\newblock Fishing for user data in large-batch federated learning via gradient
  magnification.
\newblock {\em arXiv preprint arXiv:2202.00580}, 2022.

\bibitem{yin2021see}
Hongxu Yin, Arun Mallya, Arash Vahdat, Jose~M Alvarez, Jan Kautz, and Pavlo
  Molchanov.
\newblock See through gradients: Image batch recovery via gradinversion.
\newblock In {\em Proceedings of the IEEE/CVF Conference on Computer Vision and
  Pattern Recognition}, pages 16337--16346, 2021.

\bibitem{yu2019differentially}
Lei Yu, Ling Liu, Calton Pu, Mehmet~Emre Gursoy, and Stacey Truex.
\newblock Differentially private model publishing for deep learning.
\newblock In {\em 2019 IEEE Symposium on Security and Privacy (SP)}, pages
  332--349. IEEE, 2019.

\bibitem{yu2020salvaging}
Tao Yu, Eugene Bagdasaryan, and Vitaly Shmatikov.
\newblock Salvaging federated learning by local adaptation.
\newblock {\em arXiv preprint arXiv:2002.04758}, 2020.

\bibitem{zhang2015deep}
Sixin Zhang, Anna~E Choromanska, and Yann LeCun.
\newblock Deep learning with elastic averaging sgd.
\newblock {\em Advances in neural information processing systems}, 28, 2015.

\bibitem{zhang2021survey}
Yu~Zhang and Qiang Yang.
\newblock A survey on multi-task learning.
\newblock {\em IEEE Transactions on Knowledge and Data Engineering}, 2021.

\bibitem{zhao2020efficient}
Han Zhao, Otilia Stretcu, Alexander~J Smola, and Geoffrey~J Gordon.
\newblock Efficient multitask feature and relationship learning.
\newblock In {\em Uncertainty in Artificial Intelligence}, pages 777--787.
  PMLR, 2020.

\bibitem{zhao2018federated}
Yue Zhao, Meng Li, Liangzhen Lai, Naveen Suda, Damon Civin, and Vikas Chandra.
\newblock Federated learning with non-iid data.
\newblock {\em arXiv preprint arXiv:1806.00582}, 2018.

\bibitem{zheng2021federated}
Qinqing Zheng, Shuxiao Chen, Qi~Long, and Weijie Su.
\newblock Federated f-differential privacy.
\newblock In {\em International Conference on Artificial Intelligence and
  Statistics}, pages 2251--2259. PMLR, 2021.

\bibitem{zhou2011clustered}
Jiayu Zhou, Jianhui Chen, and Jieping Ye.
\newblock Clustered multi-task learning via alternating structure optimization.
\newblock {\em Advances in neural information processing systems}, 24, 2011.

\bibitem{zhou2011malsar}
Jiayu Zhou, Jianhui Chen, and Jieping Ye.
\newblock Malsar: Multi-task learning via structural regularization.
\newblock {\em Arizona State University}, 21:1--50, 2011.

\end{thebibliography}
}

\clearpage

\appendix

\renewcommand{\contentsname}{Appendix}

\etocdepthtag.toc{mtappendix}
\etocsettagdepth{mtchapter}{none}
\etocsettagdepth{mtappendix}{subsection}

{\hypersetup{linkcolor=darkbluecolor}
\tableofcontents
}
\clearpage

\setcounter{figure}{0}
\renewcommand\thefigure{A\arabic{figure}}
\setcounter{algorithm}{0}
\renewcommand\thealgorithm{A\arabic{algorithm}}
\setcounter{table}{0}
\renewcommand\thetable{A\arabic{table}}

\section{Additional Background}
\label{supp-sec:background-extra}

\paragraph{R\'enyi Differential Privacy (RDP).}

In this work, we make use of a relaxation of different privacy known as R\'enyi Differential Privacy~\cite{mironov2017renyi} for tight privacy accounting.

\begin{definition}[R\'enyi Differential Privacy (RDP)~\cite{mironov2017renyi}]
    A randomized algorithm $ M: \mathcal X^n \to \mathcal Y$ is $(\alpha, \eps)$-RDP with order $\alpha > 1$ if for any adjacent datasets $x, x' \in \mathcal X^n$,
    \begin{align}
      D_{\alpha}(M(x) \| M(x')) \le \eps,
    \end{align}
    where $D_\alpha(P \| Q)$ is the R\'enyi divergence\footnote{The R\'enyi divergence at $\alpha = 1$ is defined as $D_{1}(P \| Q) \triangleq \underset{x \sim P}{\mathbb{E}}\left[\log \left(\frac{P(x)}{Q(x)}\right)\right]=\lim _{\alpha \rightarrow 1} D_{\alpha}(P \| Q)$, which is also the KL divergence.} between distributions $P$ and $Q$:
    \begin{align}
        D_\alpha(P \| Q) \triangleq \frac{1}{\alpha-1} \log \underset{x \sim P}{\mathbb{E}} \left[\left(\frac{P(x)}{Q(x)}\right)^{\alpha-1}\right].
      \end{align}
\end{definition}
Under R\'enyi DP, the privacy composition is simple: if every step of an algorithm satisfies $(\alpha, \eps)$-RDP, then over $T$ steps the algorithm satisfies $(\alpha, T\eps)$-RDP.
The following lemma from~\cite{bun2016concentrated, canonne2020discrete} provides a conversion from RDP to standard $(\eps, \delta)$-DP guarantees.
\begin{lemma}[Conversion from R\'enyi DP to approximate DP~\cite{bun2016concentrated, canonne2020discrete}]
  If a mechanism $M$ satisfies $(\alpha, \eps(\alpha))$-RDP, then for any $\delta > 0$, it also satisfies $(\eps(\delta), \delta)$-DP where
  \begin{align}
    \eps(\delta) = \inf _{\alpha>1} \eps(\alpha) +\frac{1}{\alpha-1} \log \left(\frac{1}{\alpha \delta} \right) +\log \left(1- \frac{1}{\alpha} \right).
  \end{align}
\end{lemma}

\textbf{Zero-Concentrated Differential Privacy (zCDP).}\enspace
A closely related privacy notion is zero-concentrated DP (zCDP~\cite{bun2016concentrated}), where $\rho$-zCDP is equivalent to satisfying $(\alpha, \rho \alpha)$-R\'enyi DP simultaneously for all orders $\alpha$. Thus, algorithms that satisfy zCDP guarantees are compatible with standard RDP accounting routines implemented in open-source libraries (e.g.\ TensorFlow Privacy~\cite{mcmahan2018general}).
In our work, we make use of zCDP and a related result for the Exponential Mechanism~\cite{DPorg-exponential-mechanism-bounded-range} for tight privacy composition when implementing private cluster selection for IFCA~\cite{ghosh2020efficient,mansour2020three} (see below and \cref{supp-sec:benchmark-methods}).

\textbf{Exponential Mechanism for Private Selection.}\enspace
The Exponential Mechanism (EM) is a standard algorithm for making private selection from a set of candidates based on their scores~\cite{mcsherry2007mechanism}. Specifically, there is a dataset $x \in \mathcal X^n$ requiring DP protection, and a scoring function $s: \mathcal X^n \times [G] \to \mathbb R$ that evaluates a set of candidates $g \in [G]$. We want to pick the candidate with the highest score (i.e.\ $\argmax_{g \in [G]} s(x, g)$) subject to $(\eps, 0)$-DP for neighboring datasets $x, x'$. The mechanism $M$ is defined by setting the probability of choosing any $g \in [G]$ as
\begin{align}
  \Pr[M(x) = g] = \frac{\exp\left(\frac{\eps}{2\Delta} \cdot s(x, g) \right)}{ \sum_{g' \in [G]} \exp \left( \frac{\eps}{2\Delta} \cdot s(x, g') \right) },
\end{align}
where $\Delta$ is the sensitivity of the scoring function. EM also satisfies $\frac{1}{8} \eps^2$-zCDP~\cite{DPorg-exponential-mechanism-bounded-range} and thus $(\alpha, \frac{\alpha}{8} \eps^2)$-RDP for all $\alpha$. A variant of EM is the Permute-and-Flip mechanism~\cite{mckenna2020permute}.

The Exponential Mechanism can be implemented as ``Report Noisy Max'' with Gumbel noise: we can add independent noises drawn from the Gumbel distribution with scale $\frac{2\Delta}{\eps}$ to the candidate scores $s(x, g)$ for all $g \in [G]$ and simply report the max noisy score. If the score function is a loss metric (where we want the minimum instead of the maximum), we can similarly implement ``Report Noisy Min'' by subtracting the Gumbel noises from the scores and report the minimum. The latter is used in our implementation for private cluster selection, where clients select clusters with lowest loss or error rate (i.e.\ $1 - \text{accuracy}$); see \cref{supp-sec:benchmark-methods} for more details.

\textbf{Privacy Budgeting.}\enspace
A typical accounting workflow, as used in our experiments, thus involves (1) composing the RDP guarantees of all private operations in the algorithm and (2) trying a list of $\alpha$ values that give the lowest $\eps$ for a target $\delta$ when converting back to $(\eps, \delta)$-DP that captures the overall privacy cost.
For SGD training, we also use existing results on privacy amplification via subsampling~\cite{mironov2019r,abadi2016deep}: if a gradient step is $(\eps, \delta)$ w.r.t.\ the dataset \textit{without} amplification and the gradient is computed with a minibatch (assumed to be a random sample) of batch size $b = q / n$ where $q$ is the sampling ratio and $n$ is the size of the dataset, then the privacy of the gradient step is amplified to $(O(q\eps), \delta)$-DP.
In a silo-specific sample-level DP setup, the size of the dataset $n_k$ at each silo thus determines the extent of the amplification, and thus even if silos target for the same $(\eps, \delta)$ sample-level DP, they may end up adding different amounts of noise when running DP-SGD (mentioned in the \textbf{Training Setup} paragraph of \cref{sec:baselines}). For this reason our experiments primarily focus on having the same privacy target for all silos.

\textbf{Heterogeneous Differential Privacy.}\enspace
A related privacy notion is \textit{heterogeneous DP}~\cite{alaggan2016heterogeneous, jorgensen2015conservative}, where each item within a dataset to be protected by DP may opt for a different $(\eps, \delta)$ target. Our setting primarily focuses on different $(\eps, \delta)$ values for disjoint datasets, and all items within a specific dataset share the same DP target.

See also \cref{supp-sec:hparam-tuning} for additional background relating to Section~\cref{sec:discussion} (private hyperparameter tuning).

\section{Additional Discussions}
\label{supp-sec:non-bijective}

\subsection{Limitations}

We discuss below the limitations of this work in addition to Section~\cref{sec:discussion}.

\textbf{When multiple records map to the same entity.}
In this paper we studied the application of silo-specific sample-level differential privacy in cross-silo federated learning. While this is an important initial step towards a more suitable privacy model for cross-silo FL (in contrast to the commonly studied client-level DP model), we assume that each entity that requires privacy protection has at most one record (training example) across silos (e.g.\ a single patient has one medical record at a hospital).

There are two characteristic cases where this assumption does not hold for all items in a silo:
\begin{itemize}[leftmargin=*]
  \item \textbf{Multiple records within a silo map to the same entity.} One example would be a student re-enrolling at the same school for multiple degree programs, thus creating multiple student records at the same silo. In such cases, the silo curator may need to carefully apply group privacy or other methods for ensuring entity-level privacy~\cite{levy2021learning} to protect the entity rather than its records.
  \item \textbf{Multiple records across silos map to the same entity.} One example would be a person having multiple credit cards at different banks.
  This case is more intricate as it is harder to precisely account for the DP guarantee for this entity without knowing (1) the silos in which this entity has appeared and (2) the specific privacy targets for each of those silos. In this case, the silos may cooperate to run \textit{private set intersection} (e.g. as considered in~\cite{lu2020multi}) to privately identify this scenario, but this would by itself come at a privacy cost.
\end{itemize}

These cases are interesting avenues for future research on private cross-silo learning.\footnote{The case of having individual records corresponding to multiple entities at once (e.g.\ one record for all family members) is slightly less interesting since sample-level privacy would protect all of the entities.}

\textbf{Extending the analysis to deep learning cases.}
In Section~\cref{sec:analysis} we use federated mean estimation as a simplified setting for analyzing the behavior of \mrmtl under silo-specific sample-level privacy.
While the analysis provides adequate insights into the empirical phenomena in \cref{fig:mrmtl-bump}, it is a simple model that does not consider the dynamic aspects of the learning settings, including (1) the Gaussian random walk component of the model updates due to DP noise applied over many training rounds, (2) the effect of communication frequency on the effect of noise reduction, (3) the concept of ``client drifts'' (as considered in~\cite{karimireddy2020scaffold}) as a result of heterogeneity and how it interfaces with the DP noises, and (4) how overparameterization may affect all of the above.

\paragraph{Caveats of cross-silo learning with very large local datasets.}
In contrast to cross-device federated learning, cross-silo federated learning is typically characterized by having a limited number of clients, each with a large local dataset.
The term ``large'' is relative because it describes the sufficiency of the datasets for fitting good local models of a \textit{specific class}; for example, 500 examples are likely sufficient to fit linear regression of 10 parameters, but very likely insufficient to learn a transformer~\cite{vaswani2017attention}.
In this sense, many FL problems in practice -- such as large commercial banks running regression on tabular data -- will in fact have local training to be the \textit{optimal} strategy, as long as there are sufficient local data and the data from other silos are not of the same local distribution.
In these cases, one should expect \mrmtl to opt for $\lam^* \approx 0$ as federated learning is not needed at all, and thus its advantages under privacy will also be minimal.

\subsection{Potential Negative Societal Impact}

Our work studies the empirical behaviors that arise when applying an alternative model of differential privacy to cross-silo federated learning, and we provide a strong baseline method (\mrmtl) that fares well in this setting. In this sense, our work sheds light on and facilitates the development of a previously underexplored area of differentially private federated learning.
However, because \mrmtl requires selecting a good regularization stength $\lam$, one potential negative impact is that users may excessively tune $\lam$ on a private dataset and inadvertently leak privacy via the choice of $\lam$ (perhaps qualitatively rather than quantitatively); for example, if a silo chose a large $\lam$ for better performance, then in principle its data would look somewhat more similar to the ``average'' of the silo datasets. Moreover, our privacy model requires that silos add their own independent noises for their own DP targets, and this requirement may not be followed correctly (either deliberately or inadvertently) to provide vacuous DP guarantees for people's data.

\section{Additional Experimental Details}
\label{supp-sec:exp-details}

\subsection{Datasets and Models}
\label{supp-sec:dataset-models}

\cref{supp-table:datasets} summarizes the datasets, tasks, and models considered in our experiments. In the following, we provide details on each.

\setlength{\tabcolsep}{3pt}
\begin{table}[h!]
\parbox{.95\linewidth}{
\centering
{
\scalebox{0.9}{
\begin{tabular}{@{}l  ccc @{\hspace{6mm}}  cc @{\hspace{6mm}} c @{}}
\toprule[\heavyrulewidth]
    \multirow{2}{*}{\textbf{Dataset}} & \multirow{2}{*}{\textbf{Task}} & \textbf{\# Clients} & \textbf{Input} & \multirow{2}{*}{\textbf{Min $n_k$}} & \multirow{2}{*}{\textbf{Max $n_k$}}  & \multirow{2}{*}{\textbf{Learner}} \\
     & & \textbf{(Silos)} & \textbf{Dim} &  &  &  \\
    \midrule
    Vehicle & Classification & 23 &  100 & 872 & 1933 & SVM \\
    School & Regression & 139  & 28  &  15  &  175    &  Linear  \\
    Google Glass (GLEAM) & Classification & 38  & 180  & 699  & 776  & SVM  \\
    Heterogenous CIFAR-10 & Classification & 30  & $32\times 32 \times 3$  & 1515  & 1839  &  ConvNet  \\
    Rotated \& Masked MNIST & Classification & 40  & $28\times 28 \times 1$ & 1500  & 1500  &  ConvNet  \\
    Subsampled ADNI & Regression & 9 & $32\times 32 \times 1$  & 45  & 2685  &  ConvNet  \\
    \bottomrule
\end{tabular}}
\vspace{0.5em}
\caption{
  Summary of datasets, tasks, and models for our empirical studies. $n_k$ denotes the number of training examples on client $k$.
  }
\label{supp-table:datasets}
}
}
\end{table}

\textbf{Vehicle~\cite{duarte2004vehicle}.}\enspace
The Vehicle Sensor dataset is a binary classification dataset containing $K=23$ data silos.
Each silo (sensor) has acoustic and seismic measurements for a road segment, with each data point being a 100-dimensional feature vector describing the measurements when a vehicle passes through the road segment.
The goal is to predict between two predetermined types of vehicles.
We use a train/test split of 75\%/25\% following previous work~\cite{smith2017federated}, yielding $872 \le n_k^\text{train} \le 1933$ training examples on each client.
We use simple linear SVMs for classification following~\cite{smith2017federated,li2021ditto}.
It is a suitable dataset for cross-silo FL because the number of silos $K$ is small while each silo has sufficient data to fit a good local model, as opposed to cross-device datasets such as FEMNIST~\cite{caldas2018leaf} where $K$ is large but each silo has little data to learn a useful model. Moreover, we can use tight privacy budgets due to reasonably large local datasets (in terms of sufficiency for fitting a good local model) and SVMs (which are relatively noise-tolerant since decision boundaries only depend on support vectors). The dataset is accessible from the original authors.\footnote{\url{https://web.archive.org/web/20110515133717/http://www.ece.wisc.edu:80/~sensit/}}

\textbf{School~\cite{goldstein1991multilevel,argyriou2008convex,zhou2011malsar}.}\enspace
The School dataset originated from the now-defunct Inner London Education Authority.\footnote{\url{https://en.wikipedia.org/wiki/Inner_London_Education_Authority}}
It is a regression dataset for predicting the exam scores of 15,362 students distributed across 139 secondary schools.
Each school has records for between 22 and 251 students, and each student is described by a 28-dimensional feature vector capturing attributions such as the school ranking, student birth year, and whether the school provided free meals.
We perform 80\%/20\% train/test split in \cref{fig:pareto-merged} (with $15 \le n_k^\text{train} \le 175$ training examples in each silo), and additionally consider 50\%/50\% and 20\%/80\% train/test split in \cref{supp-fig:pareto-school-t200}.
We use simple linear regression models following previous work (e.g.\ \cite{argyriou2008convex,zhou2011malsar,gong2012robust}) to predict student scores.
Like the Vehicle dataset, the School dataset is a natural cross-silo FL dataset with a limited number of clients $K$, each with roughly sufficient data to fit a reasonable local model. The dataset is available from~\cite{zhou2011malsar}.

\textbf{Google Glass Eating and Motion (GLEAM)~\cite{rahman2015unintrusive}.}\enspace
We also benchmark on the GLEAM dataset, a real-world head motion tracking dataset for binary classification. The motion data is collected with Google Glass, and the task is to classify the activity of the wearer (eating or not). There are in total $K=38$ silos (wearers) and 27800 data points, with each silo containing $699 \le n_k \le 776$ data points. Each data point is a 180-dimensional feature vector capturing head movement of the wearers.
Linear models yield reasonable utility on GLEAM and thus we use linear SVMs following previous work~\cite{smith2017federated}. Like Vehicle and School, this is a suitable dataset for cross-silo FL given a small $K$ and relatively large local datasets.

\textbf{Heterogeneous CIFAR-10 Dataset.}\enspace
We additionally evaluate on CIFAR-10~\cite{krizhevsky2009learning}, with heterogeneous client data split following previous work~\cite{t2020personalized,shamsian2021personalized} (based on the code provided by~\cite{shamsian2021personalized}). The dataset has $K=30$ clients (silos) in total, and data heterogeneity is generated with each client having a random number of samples from 5 randomly chosen classes out of the 10 classes. Each client has $1515 \le n_k \le 1839$ training examples. We use a simple convolutional network with the following layers: [Conv $3\times 3$ with 32 channels, ReLU, MaxPool $2\times 2$ with stride 2, Conv $3\times 3$ with 64 channels, ReLU, MaxPool $2\times 2$ stride 2, Linear]. No padding is used for convolutional layers.

\textbf{Rotated \& Masked MNIST.}\enspace
We adapt the original MNIST dataset~\cite{lecun1998gradient} to study the effect of structured heterogeneity on \mrmtl. For the 60000/10000 train/test images, we first perform a shuffle and then evenly separate them into $K = 40$ clients (silos), each with 1500/250 train/test images with roughly uniform distribution on the labels. We then randomly separate the silos into 4 groups of 10, and apply rotations of $\{0^{\circ}, 90^{\circ}, 180^{\circ}, 270^{\circ}\}$ to each group respectively; silos within the same group have the same rotations applied to the images, thus forming 4 natural silo clusters.
To add \textit{intra-cluster} heterogeneity, we then apply \textit{silo-specific} random masks of $2\times 2$ white patches; that is, all images in the same silo has the same mask, and no two silos have the same mask with very high probability. The random masks are akin to those considered in~\cite{he2021masked}.
The white patches of the random mask do not overlap, and the mask ratio is the probability of a patch being applied (so the specified percentage of masked area is an expectation).
Examples of generated images are shown in \cref{supp-fig:mnist-examples}.
These image transformations introduce two types of heterogeneity identified by~\cite{kairouz2021advances}: ``covariate shift'' (skew of feature distributions) and ``concept drift'' (same label, different features).
The model architecture is the same as the one used for heterogeneous CIFAR-10.

\begin{figure}
  \centering
  \includegraphics[width=\linewidth]{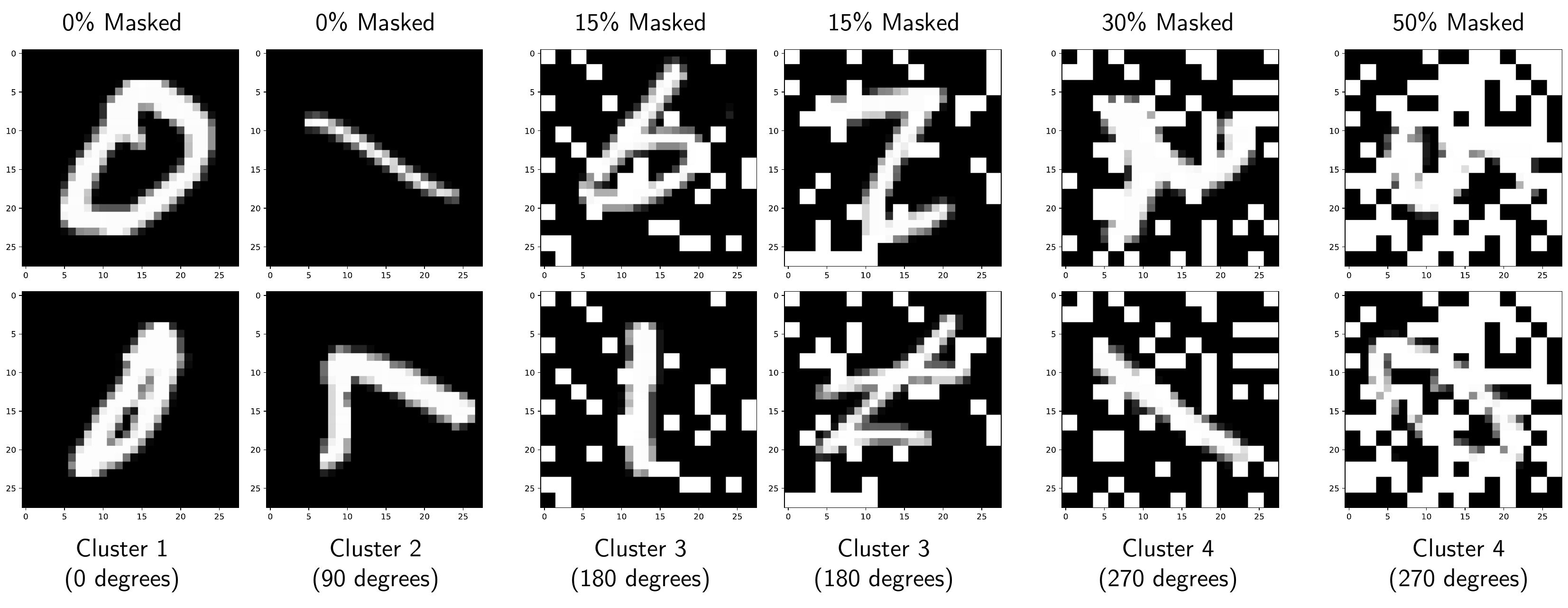}
  \caption{
    \textbf{Example images of the Rotated \& Masked MNIST dataset}.
    Each column corresponds to two images from the same (random) client from the specified cluster (thus they have the same client-specific random mask). Labels for each column from left to right: (0, 0), (1, 7), (9, 1), (2, 7), (2, 1), (5, 0).
  }
  \label{supp-fig:mnist-examples}
  \vspace{-0.5em}
\end{figure}

\textbf{Subsampled Alzheimer's Disease Neuroimaging Initiative (ADNI) Dataset~\cite{adni-dataset}.}\enspace
We additionally benchmark on the ADNI dataset, which is a real-world dataset containing brain PET scans of Alzheimer's disease patients, patients with mild cognitive impairment, and healthy people taken from multiple institutions~\cite{qu2021experimental}.
It is a regression dataset for predicting the SUVR value (a scalar ranged roughly between 0.8 and 2) from the PET scan images of a brain.
We simplified the full dataset for faster training by subsampling the axial slices generated for each brain PET scan (96 slices for scan), turning them into a gray scale image, downsampling them to size $32\times 32$, and randomly splitting them into a 75\%/25\% train/test sets.
There are in total $K=9$ silos containing a total of 11040 images; each silo corresponds to a different equipment that took the PET scans and contains $45 \le n_k^\text{train} \le 2685$ training examples.
See Fig.~4 of \cite{reith2020application} and Fig. 4 of~\cite{qu2021experimental} for sample images.
The model architecture is a simple convolutional network with the following layers: [Conv $5\times 5$ with 32 channels, ReLU, MaxPool $2\times 2$ with stride 2, Conv $5\times 5$ with 64 channels, ReLU, MaxPool $2\times 2$ stride 2, Linear]. No padding is used for convolutional layers.

\subsection{License/Usage Information for Datasets}

\textbf{Vehicle.}\enspace The Vehicle dataset was made publicly available by the original authors as a research dataset~\cite{duarte2004vehicle} and license information was unavailable. It has been subsequently used in many work (e.g.~\cite{smith2017federated}).

\textbf{School.}\enspace The original entity that collected the School dataset~\cite{goldstein1991multilevel} is defunct and license information was unavailable. The dataset has been made publicly available~\cite{zhou2011malsar} and used extensively in previous work (e.g.~\cite{zhou2011malsar,zhao2020efficient}).

\textbf{Google Glass (GLEAM).}\enspace The GLEAM dataset was made publicly available by the original authors and can be used for any non-commercial purposes. See this URL\footnote{\url{http://www.healthailab.org/data.html}} for license and usage information.

\textbf{Heterogeneous CIFAR-10.}\enspace The original CIFAR-10 dataset is available under the MIT license.

\textbf{Rotated \& Masked MNIST.}\enspace The original MNIST dataset is available under the CC BY-SA 3.0 license.

\textbf{Subsampled ADNI.}\enspace
As per the Data Use Agreement of the ADNI dataset:\footnote{\url{https://adni.loni.usc.edu/wp-content/uploads/how_to_apply/ADNI_Data_Use_Agreement.pdf}}
\begin{quote}
  Data used in preparation of this manuscript were obtained from the Alzheimer's Disease Neuroimaging Initiative (ADNI) database (adni.loni.usc.edu).
  As such, the investigators within the ADNI contributed to the design and implementation of ADNI and/or provided data but did not participate in analysis or writing of this manuscript. A complete listing of ADNI investigators can be found at this URL.\footnote{\url{https://adni.loni.usc.edu/wp-content/uploads/how_to_apply/ADNI_Acknowledgement_List.pdf}}
\end{quote}

The data sharing and publication policy of the ADNI dataset can be found at this URL.\footnote{\url{https://adni.loni.usc.edu/wp-content/uploads/how_to_apply/ADNI_DSP_Policy.pdf}}
Access to the dataset must be approved by ADNI.

\subsection{Benchmark Methods and Implementation}
\label{supp-sec:benchmark-methods}

We provide more details on our benchmark personalization methods below.

\textbf{Local finetuning}~\cite{wang2019federated,yu2020salvaging,cheng2021fine} is one of the simplest but most effective personalization methods: once clients obtain a shared model via federated training (e.g.\ FedAvg), they can personalize it with additional training steps over their local dataset. This simple strategy has been shown to work very well empirically~\cite{wang2019federated,yu2020salvaging,cheng2021fine}, with the work of~\cite{cheng2021fine} providing theoretical support that it can asymptotically achieve comparable performance to other more sophisticated methods. Moreover, in contrast to other local adaptation methods like distillation~\cite{yu2020salvaging}, local finetuning's privacy footprint under our privacy model can be easily controlled (by limiting the number of finetuning and total training steps) without qualitatives change in its behavior.
In our experiments, local finetuning is implemented as FedAvg followed by local training, each taking 50\% of the total number of training rounds to ensure an identical privacy budget as other baseline methods.

\textbf{Ditto}~\cite{li2021ditto} is the current state-of-the-art method for personalization with provable benefits of robustness and fairness. It is closely related to $\mrmtl$ because it similarly trains personalized models while $\ell^2$-regularizing them towards a global model, but it differs from \mrmtl in that its global model can be obtained by a standalone solver. In particular, when no privacy is added, Ditto's modularity allows it to perfectly interpolate between local and FedAvg training with its regularization strength $\lambda$.
In our experiments, we implement Ditto with the FedAvg solver and use a minimal number of iterations over the local datasets to avoid excessive privacy overhead.

\textbf{Mocha}~\cite{smith2017federated} is a multi-task learning framework tailored for federated settings. During training, it simultaneously learns the personalized models as well as a client-relationship matrix which can model both positive and negative client relationships (in contrast, clustering methods only focus on positive relationships).
One disadvantage of Mocha is that it applies to convex problems only. In particular, the original paper uses a dual formulation for efficient training, but for fair comparison and compatibility with privacy primitives (especially DP-SGD~\cite{song2013stochastic,bassily2014private,abadi2016deep}), we implemented Mocha in its primal form and trained it with SGD in our experiments. Similarly, we align Mocha to other baselines in terms of the number of training steps to prevent privacy overhead.

\textbf{IFCA}~\cite{ghosh2020efficient} (and the conceptually similar \textbf{HypCluster}~\cite{mansour2020three}) is a simple clustering framework proposed as an extension to FedAvg. In every round of IFCA training, the server sends $k$ models (cluster centroids) to all clients; each client locally evaluates them over its local training data and selects the one with the lowest loss. Each client then locally trains on the selected model and returns updates only for this model, along with its index. Clients that selected the same model indices can be viewed as belonging to the same cluster.
We use IFCA as the representative clustering method due to its performance, simplicity, and practicality.

The privacy overhead of IFCA comes in the form of \textit{private cluster selection}: when each client evaluates the incoming models (cluster centroids) on the local datasets, this procedure must be privatized as the selection itself may leak information about the dataset. Private selection can be implemented via the exponential mechanism~\cite{mcsherry2007mechanism} with sharp accounting via a bounded range analysis~\cite{DPorg-exponential-mechanism-bounded-range} (discussed in~\cref{supp-sec:background-extra}), but one must decide (e.g.\ as a privacy budget $\eps_\mathrm{select}$ for the same $\delta$) how to share the selection cost with DP-SGD under a fixed total privacy budget $\eps_\mathrm{total}$.

\textbf{Mitigating strategies for private cluster selection.}
In our experiments,
we observed that if private selection is implemented naively, it can incur a prohibitive privacy overhead and destroys the final utility (e.g.~\cref{fig:privacy-overhead}).
There are two important reasons: (1) unlike DP-SGD, no privacy amplification applies to private selection, and (2) the sensitivity of the training loss (which is used to select cluster centroids) is unbounded in general and must be clipped to a reasonable value (e.g. $\le 1$).
We propose two mitigation strategies:
\begin{enumerate}[leftmargin=*]
  \item \textbf{Use accuracy instead of loss.} For the cluster selection metric (i.e.\ the score function $s(x, g)$ where $x$ is the local dataset and $g$ is a particular cluster centroid), we use the error rate ($1 - \text{accuracy}$) instead of the loss (which is used by the original authors~\cite{ghosh2020efficient,mansour2020three}).
  The rationale is that accuracy is a low-sensitivity function, particularly in cross-silo settings: one can show that, by enumerating the cases where the differing example between the neighboring datasets $x$ and $x'$ are correctly/incorrectly classified under addition/removal/replacement notions of DP, the sensitivity $\Delta_\mathrm{acc}$ of $s$ is bounded as
  \begin{align}
    \Delta_\mathrm{acc} &= \max_{g} \max_{x, x'} \left| s(x, g) - s(x, g') \right|
                        \le \frac{1}{n-1}
  \end{align}
  where $n$ is the size of the local dataset.
  Since $n$ can be large in cross-silo settings, the sensitivity can be orders of magnitude smaller than that of the loss function.
  With small sensitivity, we heuristically set the per-round selection privacy budget to a very small $\eps_\text{select} = 0.03 \cdot \eps_\text{total}$.
  Despite this, however, the cost of private selection can still grow quickly over the entire training process and considerably eat into DP-SGD's privacy budget.
  \item \textbf{Truncate the number of cluster selection rounds.} Our second strategy is to simply run less rounds of cluster selection (e.g.\ to 10\% of total number of training rounds, as in \cref{fig:pareto-merged}).
  This follows from the empirical analysis of~\cite{ghosh2020efficient} as well as our own experimental observation that clusters tend to converge quickly, though in some cases, clusters may not fully converge within 10\% of training rounds.
\end{enumerate}

Despite these strategies, however, the private selection cost can still lead to a steep utility hit.
Note also that the new hyperparameter $\eps_\text{select}$ can be tuned; for a fixed total budget $\eps_\text{total}$, a small $\eps_\text{select}$ means the budget for DP-SGD $\eps_\text{train}$ is not affected by much, but the selected clusters would be very noisy and inaccurate; a large $\eps_\text{select}$ leads to less noisy clusters (and thus smaller intra-cluster heterogeneity), but DP-SGD will correspondingly use larger noise and hurt optimization.

\subsection{Training Settings}

\textbf{Optimizers.}\enspace
For simplicity of hyperparameter tuning and experimental controls, we use minibatch DP-SGD for client local training without local or server momentum for all experiments (in fact, FedAvgM~\cite{hsu2019measuring} and FedAdam~\cite{reddi2021adaptive} were not found to be helpful on Vehicle and School). While there are more efficient solvers for the Vehicle and School datasets since we are dealing with convex problems, we want compatibility with DP-SGD~\cite{song2013stochastic,bassily2014private,abadi2016deep} as well as tight privacy accounting with privacy amplification via subsampling (i.e.\ via minibatch training).

\textbf{Hyperparameters.}\enspace
For all datasets and all methods, we set silos to train for 1 local epoch in every round (except Ditto~\cite{li2021ditto} which takes 2 local epochs).
For Vehicle, GLEAM, School, Heterogeneous CIFAR-10, Rotated \& Masked MNIST, and subsampled ADNI respectively, the local batch size across all silos are fixed with $B = 64, 64, 32, 100, 100, 64$, and the clipping norm for per-example gradients are heuristically set to $c = 6, 6, 1, 8, 1, 0.5$.
Vehicle uses $T=400$ rounds for most experiments (except \cref{fig:finetune-gap} which trains for $T=200$ rounds); School, Google Glass, Heterogeneous CIFAR-10, and Rotated \& Masked MNIST use $T=200$; and ADNI uses $T=500$.

For multi-task learning methods ($\mrmtl$, Ditto~\cite{li2021ditto}, Mocha~\cite{smith2017federated}), we sweep the regularization strength across a grid of $\lambda \in [0.0001,0.001,0.003,0.01,0.03,0.1,0.3,1,3,10]$ to find the best $\lambda^*$ wherever applicable (e.g.\ \cref{fig:pareto-merged,fig:mrmtl-bump,supp-fig:pareto-school-t200}). To compensate for the change in the gradient magnitude, we also sweep different client learning rates across a grid of $\eta \in [0.001,0.003,0.01,0.03,0.1,0.3]$; for fair comparison, the same grid of $\eta$ is swept for methods that do not involve $\lambda$ (e.g.\ IFCA, local finetuning).\footnote{As discussed in Section~\cref{sec:discussion}, releasing the results from repeated experiments (possibly with different hyperparameters) may technically compromise the privacy of the datasets~\cite{liu2019private,papernot2022hyperparameter}. In our case, we are primarily interested in understanding the behaviors and tradeoffs that emerge under silo-specific sample-level DP, and thus for experimental control and ease of comparison we do not account for the privacy costs from hyperparameter tuning and repetitions. We also use only public datasets in our experiments.}
For \cref{fig:finetune-gap} and \cref{supp-fig:finetune-sweep}, the learning rate is fixed to $\eta = 0.01$. For all datasets, the chosen privacy parameter $\delta$ satisfy $\delta < n_k^{-1.1}$ where $n_k$ is the local training dataset size.

\textbf{Evaluation Protocol.}\enspace
For all datasets, we evaluate methods by the average test metric (accuracy or MSE) across the silos, weighted by their respective test sample counts. Weighted averaging allows the final test metric to reflect a method's performance over the individual test samples of the combined dataset across silos, thus fairer and more aligned (compared to uniform averaging of silo test metrics) with our privacy model where each test sample represents an entity requiring protection.

\subsection{Hardware}

Experiments for Vehicle, School, and Google Glass (GLEAM) are run on commodity CPUs and experiments for Heterogeneous CIFAR-10, Rotated \& Masked MNIST and ADNI are run on four NVIDIA RTX A6000 GPUs.

\subsection{Code}
\label{supp-sec:code}

Our experiments are implemented in Python with NumPy~\cite{harris2020array}, JAX~\cite{jax2018github} and Haiku~\cite{haiku2020github}. For private training, JAX is used to vectorize DP-SGD over per-example gradients~\cite{subramani2021enabling}.
Code is available at \url{https://github.com/kenziyuliu/private-cross-silo-fl}.

\section{Additional Algorithmic Details}
\label{supp-sec:algorithms}

\subsection{Mean-Regularized Multi-Task Learning (\mrmtl)}

\cref{alg:mr-mtl} describes the canonical instantiation of \mrmtl~\cite{t2020personalized}. Its key ingredient is the mean-regularization (Line 6 of \cref{alg:mr-mtl}) that forces the local personalized models $w_k$ to be close to their mean $\bar w$. Silo-specific sample-level privacy is added by privatizing the local gradients as in DP-SGD~\cite{song2013stochastic,bassily2014private,abadi2016deep}.

\textbf{Privacy of \mrmtl.}\enspace
Since the iterates of $\bar w^{(t)}$ are already differentially private (as they are the average of the private iterates $w_k^{(t)}$), the additional regularization term $\frac{\lambda}{2} \| w_k^{(t)} - \bar w_k^{(t-1)} \|_2^2$ (and hence \mrmtl) does not incur privacy overhead compared to local and FedAvg training.

\textbf{Weighted vs Unweighted Model Updates.}\enspace
It is customary to weigh the model updates from each client (silo) by its training example counts, as in the original FedAvg implementation~\cite{mcmahan2017communication}. However, note that under silo-specific sample-level privacy with \textit{addition/removal} notions of DP, the example counts on each silo may itself leak sensitive information (e.g.\ when a silo only has one record). This is less of an issue with the \textit{replacement} notion of DP, since neighboring datasets would have the same example counts. In our experiments we use weighted model updates and thus implicitly assume the replacement notion of DP. Nevertheless, the resulting privacy guarantees are constant factors apart and empirically we did not observe significant changes in performance when using unweighted aggregation.

    \begin{algorithm}[t]
      \caption{Mean-Regularized Multi-Task Learning}
      \label{alg:mr-mtl}
      \begin{algorithmic}[1]
        \State {\bfseries Input:} Initial client models $\{ w_k^{(0)} \}_{k\in [K]}$, and mean model $\bar w^{(0)}$.
        \For{training round $t = 1, ..., T$}
          \State Server sends $\bar w^{({t-1})}$ to every client.
          \For{client $k = 1, ..., K$ {\bfseries in parallel}}
            \State Set model iterate: $w_k^{(t)} \gets w_k^{(t-1)}$.
            \State \Longunderstack[l]{For every batch $(x, y)$, client updates $w_k^{(t)}$ using SGD or DP-SGD with batch loss\\ $\ell(w_k^{(t)}, x, y)$ and gradient}
            \[
              \nabla_{w_k^{(t)}} \left[ \ell\left(w_k^{(t)}, x, y \right) + \frac{\lambda}{2} \left\| w_k^{(t)} - \bar w^{(t-1)} \right\|_2^2 \right].
            \]  \label{alg:line:mrmtl-gradient}
            \State Return model update $\Delta_k^{(t)} = w_k^{(t)} - w_k^{(t - 1)}$.
          \EndFor
          \State Server updates $\bar w^{({t})} = \bar w^{({t-1})} + \frac{1}{K} \sum_{k=1}^{K} \Delta_k^{(t)}$ (may weigh $\Delta_k^{(t)}$ by client example counts).
        \EndFor
        \State {\bfseries Output:} Personalized models $w_k^{(T)}$ for all $i \in [K]$.
      \end{algorithmic}
    \end{algorithm}

\subsection{IFCA Preconditioning / Warm-Starting}
\label{supp-sec:ifca-precond}

In Section~\cref{sec:method}, we considered an extension to \mrmtl where training is ``warm-started'' by a small number of rounds of clustering (via IFCA~\cite{ghosh2020efficient,mansour2020three}), followed by \mrmtl training \textit{within} each formed cluster. Pseudocode for this procedure is shown in \cref{alg:ifca-precond}.

\textbf{Privacy of IFCA Preconditioning.}\enspace
Observe that as with \mrmtl, the gradient steps satisfy silo-specific sample-level DP regardless or whether the steps are made on the cluster models or the personalized models.
IFCA preconditioning introduces privacy overhead in the form of private cluster selection (Line 6 of \cref{alg:ifca-precond}; discussed in \cref{supp-sec:background-extra,supp-sec:benchmark-methods}), which splits the total privacy budget with DP-SGD. As a result the noise scale for DP-SGD would be increased and can be numerically determined.

\begin{algorithm}[t]
  \caption{IFCA-Preconditioned \mrmtl}
  \label{alg:ifca-precond}
  \begin{algorithmic}[1]
    \State {\bfseries Input:} Initial client models $\{ w_k^{(0)} \}_{k\in [K]}$, number of clusters $G$, initial cluster models $\{ \bar w^{(0)}_g \}_{g\in [G]}$, total number of rounds $T$, and the number of initial clustering rounds $T_\text{cluster}$.
    \State {\textcolor{OliveGreen}{\texttt{\# IFCA preconditioning rounds}}}
    \For{IFCA training round $t = 1, ..., T_\text{cluster}$}
      \State Server sends cluster models $\{ \bar w^{(t-1)}_g \}_{g\in [G]}$ to every client.
      \For{client $k = 1, ..., K$ {\bfseries in parallel}}
        \State \Longunderstack[l]{\textbf{Use Exponential Mechanism (\cref{supp-sec:background-extra}) to select best cluster} $\bar w^{(t-1)}_{g^*(k)}$ from\\$\{ \bar w^{(t-1)}_g \}$ with loss/error rate function $s$ and local dataset $(X_k, Y_k)$.}
        \label{alg:line:cluster-selection}
        \State Set model iterate: $w_k^{(t)} \gets \bar w^{(t-1)}_{g^*(k)}$.
        \State \Longunderstack[l]{For every batch $(x, y)$, client updates $w_k^{(t)}$ using SGD or DP-SGD with batch loss\\ $\ell(w_k^{(t)}, x, y)$ and gradient $\nabla_{w_k^{(t)}} \ell\left(w_k^{(t)}, x, y \right)$.}
        \State Return model update $\Delta_{k}^{(t)} = w_k^{(t)} - \bar w^{(t-1)}_{g^*(k)}$ and selected cluster index $g^*(k)$.
      \EndFor
      \For{each cluster $g \in [G]$}
        \State \Longunderstack[l]{Server applies (weighted) model updates to cluster $g$ with the associated client indices:}
        \begin{align}
          \bar w^{({t})}_g = \bar w^{({t-1})}_g + \frac{1}{|\mathcal K_g|} \sum_{k \in \mathcal K_g} \Delta_{k}^{(t)}, \quad \text{where } \mathcal K_g \equiv \{ k \in [K] \mid g^*(k) = g \}.
        \end{align}
      \EndFor
    \EndFor

    \State {\textcolor{OliveGreen}{\texttt{\# \mrmtl rounds with regularization towards frozen cluster centroids}}}
    \For{\mrmtl training round $t = T_\text{cluster} + 1, ..., T$}
      \State Server sends cluster models $\{ \bar w^{(t-1)}_g \}_{g\in [G]}$ to every client.
      \For{client $k = 1, ..., K$ {\bfseries in parallel}}
        \State \textbf{Client retrieves the last selected cluster centroid} $\bar w^{(t-1)}_{g^*(k)}$.
        \State Set model iterate: $w_k^{(t)} \gets w_k^{(t-1)}$ (\textbf{using the personalized model from last round}).
        \State \Longunderstack[l]{For every batch $(x, y)$, client updates $w_k^{(t)}$ using SGD or DP-SGD~\cite{song2013stochastic,bassily2014private,abadi2016deep} with batch loss\\ $\ell(w_k^{(t)}, x, y)$ and gradient}
        \[
          \nabla_{w_k^{(t)}} \left[ \ell\left(w_k^{(t)}, x, y \right) + \frac{\lambda}{2} \left\| w_k^{(t)} - \bar w^{(t-1)}_{g^*(k)} \right\|_2^2 \right].
        \]
        \State Return model update $\Delta_{k}^{(t)} = w_k^{(t)} - \bar w^{(t-1)}_{g^*(k)}$ and the cluster index $g^*$.
      \EndFor
      \For{each cluster $g \in [G]$}
        \State \Longunderstack[l]{Server applies (weighted) model updates to cluster $g$ with the associated client indices:}
        \begin{align}
          \bar w^{({t})}_g = \bar w^{({t-1})}_g + \frac{1}{|\mathcal K_g|} \sum_{k \in \mathcal K_g} \Delta_{k}^{(t)}, \quad \text{where } \mathcal K_g \equiv \{ k \in [K] \mid g^*(k) = g \}.
        \end{align}
      \EndFor
    \EndFor
    \State {\bfseries Output:} Personalized models $w_k^{(T)}$ for all $i \in [K]$.
  \end{algorithmic}
\end{algorithm}

\clearpage

\section{Additional Analysis Details}
\label{supp-sec:proofs}

In this section we provide additional details for the analysis presented in Section~\cref{sec:analysis}. We also extend the analysis to the case with varying $n$, $\sigma$, $\sigma_{\mathrm{DP}}$ for each silo in \cref{supp-sec:analysis-varying-silo-params}.

\subsection{Notations}

Unless otherwise specified, we used the following notations throughout the analysis:
\begin{itemize}[leftmargin=20pt]
  \item $K$ denotes the total number of silos (clients) with indices $k \in [K]$;
  \item $n$ denotes the number of examples on each silo;
  \item $w_k$ denotes the true center of silo $k$'s data distribution;
  \item $X_k \equiv \{ x_{k, i} \}_{i \in [n]}$ denotes the local dataset with $n$ data points;
  \item $\hat w_k$ denotes the best local estimator of $w_k$;
  \item $\bar w$ denotes the average of local estimators;
  \item $\hat w_{\setminus k} \triangleq \frac{1}{K-1} \sum_{j \neq k, j \in [K]} \hat w_j$ denotes the external average estimators from the perspective of $k$;
  \item $\sigma^2$ denotes the sampling variance of the local data $X_k$;
  \item $\tau^2$ denotes the sampling variance of the local data centers $w_k$ (hence a measure of data heterogeneity across silos); and
  \item $\sigmadp^2$ denotes the DP noise variance on each silo to satisfy silo-specific sample-level privacy.
\end{itemize}

\subsection{\mrmtl~Formulation}

The general formulation of the mean-regularized MTL objective may be expressed as
\begin{align} \label{supp:mrmtl-objective}
    \underset{w_k, k\in [K]}{\min}  \frac{1}{K} \sum_{k=1}^K h_k(w_k)  \quad \text{with} \quad h_k(w) \triangleq F_k(w) + \frac{\lambda}{2} \| w - \bar w \|_2^2,
\end{align}
where $F_k(\cdot)$ is the local objective for client $k$, $\bar w \triangleq \frac{1}{K} \sum_{k=1}^K w_k$ is the average model, and $\lambda \ge 0$ is the regularization strength.
A larger $\lambda$ enforces the models to be closer to each other, and $\lambda = 0$ reduces the problem to local training.
In particular, unlike Ditto~\cite{li2021ditto}, $\mrmtl$ may \textit{not} recover FedAvg~\cite{mcmahan2017communication} under SGD as $\lambda \to \infty$, since $\lambda$ essentially changes the ratio between the local objective gradients and the regularization gradients.
For the purposes of our analysis, we assume $F_k$ is strongly convex for all $k \in [K]$.

\subsection{Assumptions}

Our characterization of $\mrmtl$ on makes the following simplifying assumptions.
\begin{assumption} \label{assp:private-scalar-same-params}
  All clients (silos) have the same number of data points $n$, data sampling variance $\sigma^2$, and DP noise variance $\sigmadp^2$.
\end{assumption}
\begin{assumption}  \label{assp:private-scalar-no-clip}
  A sufficiently large clipping $c$ can be selected such that $\| x_{k,i}\|_2 \le c$ with high probability.
\end{assumption}

Note that \cref{assp:private-scalar-same-params} primarily serves to make results cleaner and easily interpretable (see \cref{supp-sec:analysis-varying-silo-params} for extensions). \cref{assp:private-scalar-no-clip} is mild since Gaussians have strong tail decay.

\subsection{Omitted Details and Proofs}

\paragraph{Lemma 11 of \cite{mahdavifar2017global}.}  Our analysis (particularly \cref{lem:wk-as-hatwk-hatwkminus}) makes use of the following lemma to determine the posterior of an unknown parameter given several independent Gaussian observations.
\begin{lemma}[Lemma 11 of \cite{mahdavifar2017global}] \label{lem:ivw-lemma}
  Let $\theta \in \mathbb R$ be a constant (non-informative prior).
  Let $\{\phi_k \triangleq \theta + z_k \}_{k\in [K]}$ with $z_k \sim \mathcal N(0, \sigma_k^2)$ be $m$ independent noisy observations of $\theta$ with variances $\{\sigma_k^2 \}_{k\in [K]}$.
  Then, conditioned on $\{\phi_k \}_{k\in [K]}$, with $\sigma_\theta^2 \triangleq (\sum_{k \in [K]} 1 / \sigma_k^2 )^{-1}$, we have
  \begin{align}
      \theta = \sigma_{\theta}^{2} \sum_{k \in[K]} \frac{\phi_{k}}{\sigma_{k}^{2}} + z, \text{ where } z\sim \mathcal N(0, \sigma_\theta^2 ).
  \end{align}
\end{lemma}
This lemma allows us to express the local true centers $w_k$ conditioned on the local empirical estimates $\hat w_k$, the external empirical estimates $\hat w_{\setminus k}$, and the local datasets $\{ X_k \}_{k \in [K]}$.

Below we present omitted proofs for the main results presented in Section~\cref{sec:analysis}. We restate the lemmas and theorems for convenience.

\textbf{\cref{lem:wk-lambda}}. The minimizer of the silo-specific objective $h_k(w)$ (see \cref{eq:mrmtl-objective}, \cref{supp:mrmtl-objective}) is
\begin{align}
  \hat w_k(\lambda) = \alpha \cdot \hat w_k + (1 - \alpha) \cdot \hat w_{\setminus k},
  \quad \text{where} \quad \alpha \triangleq \frac{K+\lambda}{(1+\lambda) K} \in (1/K, 1].
\end{align}
\begin{proof}[Proof of \cref{lem:wk-lambda}.]
  First note that $\bar w$ is independent of $\lambda$. Since we assume $F_k$ is convex, $h_k$ is also convex, and the proof follows from taking the derivative of $h_k$ and setting it to 0:
  \begin{align}
    \frac{\partial h_k(w)}{\partial w}
        &= \frac{1}{n} \sum_{i=1}^n (w - x_{k, i}) + \lambda (w - \bar w) = (w - \hat w_k) + \lambda (w - \bar w), \\
    \hat w_k(\lambda)
        &= \frac{1}{1 + \lambda} \hat w_k + \frac{\lambda}{1 + \lambda} \bar w = \frac{K+\lambda}{(1+\lambda) K} \cdot \hat{w}_{k} + \frac{\lambda (K-1)}{(1+\lambda) K} \hat{w}_{\setminus k}.
  \end{align}

  Note here that in the non-private case, the local estimator is given by the empirical average: $\hat w_k = \frac{1}{n} \sum_{i=1}^n x_{k, i}$. In the private case, under \cref{assp:private-scalar-no-clip}, the local estimator is $\hat w_k = \frac{1}{n} \left(\xi_k + \sum_{i = 1}^n x_{k, i} \right)$ where $\xi_k \sim \mathcal N(0, \sigmadp^2)$ is the one-shot privacy noise, and we can similarly arrive at the same $\hat w_k(\lambda)$ expression.
\end{proof}

\textbf{\cref{lem:wk-as-hatwk-hatwkminus}}. Let $\sigmaloc^2 \triangleq \sigma^2 / n + \sigmadp^2 / n^2$ denote the local variance on each silo as a result of data sampling and DP noise.
Then, given $\hat w_k$ and $\hat w_{\setminus k}$, $w_k$ can be expressed as
  \begin{align}
    w_k &= \mu_k + \zeta_k \\
    \text{where} \quad \zeta_k &\sim \mathcal N(0, \sigma_w^2), \\
    \text{with} \quad \sigma_w^2 &\triangleq \left( \frac{1}{\sigmaloc^2}+\frac{K-1}{K \tau^{2}+ \sigmaloc^2}  \right)^{-1} \\
    \text{and} \quad \mu_k &\triangleq \sigma_{w}^{2} \left( \frac{1}{\sigmaloc^2 } \cdot \hat{w}_{k}+\frac{K-1 }{K \tau^{2} + \sigmaloc^2} \cdot \hat{w}_{\setminus k} \right).
  \end{align}
\begin{proof}[Proof of \cref{lem:wk-as-hatwk-hatwkminus}.]
  When given $\theta$, we can express $w_k = \theta + z_k$ where $ z_k \sim \mathcal N(0, \tau^2)$, and by symmetry when given $w_k$,
    \begin{align} \label{eq:theta-sim-wk-tau}
        \theta = w_k + z_k \quad \text{where} \quad z_k \sim \mathcal N(0, \tau^2).
    \end{align}
    Note also that when given the local dataset $X_k$, $\hat w_k$ is a noisy observation of $w_k$ with Gaussian noise from data sampling and added privacy:
    \begin{align} \label{eq:wk-hat-sim-wk}
        \hat w_k = w_k + \hat z_k \quad \text{where} \quad \hat z_k \sim \mathcal N(0, \sigmaloc^2).
    \end{align}
    Moreover, when given the local datasets $\{X_j\}_{j\in [K], j\neq k}$, $\hat{w}_{\setminus k}$ can be viewed as a noisy observation of $\theta$ with Gaussian noise from the silo heterogeneity and the empirical mean of the local estimators:
    \begin{align} \label{eq:w-minusk-sim-theta}
        \hat{w}_{\setminus k} = \theta + \hat z_{\setminus k} \quad \text{where} \quad \hat z_{\setminus k} \sim \mathcal N\left(0, \frac{\tau^2 + \sigmaloc^2}{K - 1}\right).
    \end{align}
    Combining (\ref{eq:w-minusk-sim-theta}) with (\ref{eq:theta-sim-wk-tau}) we have, when given $w_k$ and the local datasets $\{X_j\}_{j\in [K], j\neq k}$,
    \begin{align} \label{eq:w-minusk-sim-wk}
        \hat{w}_{\setminus k} = w_k + z_{\setminus k} \quad \text{where} \quad z_{\setminus k} \sim \mathcal N\left(0, \tau^2 + \frac{\tau^2 + \sigmaloc^2}{K - 1}\right) \equiv \mathcal N\left(0, \frac{K \tau^2 + \sigmaloc^2}{K - 1}\right).
    \end{align}
    Invoking Lemma~\ref{lem:ivw-lemma} on \cref{eq:wk-hat-sim-wk} and \cref{eq:w-minusk-sim-wk} we have the desired $\mu_k$ and $\sigma_w^2$.

    A key observation for this derivation is that the local datasets $\{X_k\}_{k \in [K]}$ form the Markov blankets of $\hat w_k$ and $\hat w_{\setminus k}$ (as they are sampled \textit{after} $w_k$ are sampled, and the estimators $\hat w_k$ and $\hat w_{\setminus k}$ are computed based on the datasets). Thus, given $\{X_k\}_{k \in [K]}$, $\hat w_k$ and $\hat w_{\setminus k}$ can be viewed as independent observations of $w_k$ and \cref{lem:ivw-lemma} applies.
\end{proof}

\textbf{\cref{thm:optimal-lambda}} (Optimal $\mrmtl$ estimate).
Let $\lambda^*$ be the optimal $\lambda$ that minimizes the generalization error. Then,
\begin{align} \label{supp-eq:generalization-error-mmse}
  \lambda^{*}= \argmin_{\lambda} \mathbb{E} \left[ \left(w_{k}-\hat{w}_{k}(\lambda)\right)^{2} \mid \hat{w}_{k}, \hat{w}_{\setminus k}, \{ X_k \}_{k \in [K]} \right] = \frac{\sigmaloc^2}{\tau^2}
  = \frac{1}{n \tau^2} \left( \sigma^2 + \frac{\sigmadp^2}{n} \right).
\end{align}
\begin{proof}[Proof of \cref{thm:optimal-lambda}]
  The objective of finding $\lambda^*$ equates to finding an expression for $\lambda$ such that when $w_k$ is viewed as a noisy observation of $\hat w_k(\lambda)$ (which is an interpolation of $\hat w_k$ and $\hat w_{\setminus k}$ from \cref{lem:wk-lambda}), the coefficients for $\hat w_k$ and $\hat w_{\setminus k}$ in $\hat w_k(\lambda)$ matches those of $w_k$ (\cref{lem:wk-as-hatwk-hatwkminus}). That is, we want $\lambda^*$ such that
  \begin{align}
      \frac{K+\lambda^*}{(1+\lambda^*) K}  = \frac{\sigma_{w}^{2}}{\sigmaloc^2} = \frac{1}{\sigmaloc^2} \cdot \left( \frac{1}{\sigmaloc^2}+\frac{K-1}{K \tau^{2}+ \sigmaloc^2}  \right)^{-1} .
  \end{align}
  Simpifying gives $\lambda^* = {\sigmaloc^2} / {\tau^2}$. Note that $\hat w_k(\lambda^*)$ is the (conditional) MMSE estimator of $w_k$.
\end{proof}
Note that \cref{supp-eq:generalization-error-mmse} measures the \textit{generalization} error because $w_k$ is the (unobserved) true center of the local data distribution around which testing data points would be drawn.

\textbf{\cref{thm:optimal-local-error-gap}} (Optimal error gap to local training).
Denote the error of the local estimate as
\begin{align}
  \mathcal E_{\mathrm{loc}} \triangleq \mathbb{E} \left[ (w_{k}-\hat{w}_{k} )^{2} \mid  X_k \right],
\end{align}
and let $\Delta_\mathrm{loc} \triangleq \mathcal E_{\mathrm{loc}} - \mathcal E^*$ be its error gap to the optimal estimate. Then,
\begin{align}
  \Delta_\mathrm{loc} = \left( 1 - \frac{1}{K} \right)  \cdot \frac{\sigmaloc^4}{ \sigmaloc^2 + \tau^2}.
\end{align}
\textbf{\cref{thm:optimal-global-error-gap}} (Optimal error gap to FedAvg).
Denote the error of the global (FedAvg) estimate as
\begin{align}
  \mathcal E_{\mathrm{fed}} \triangleq \mathbb{E} \left[ (w_{k}- \bar{w} )^{2} \mid \{ X_k \}_{k \in [K]} \right],
\end{align}
and let $\Delta_\mathrm{fed} \triangleq \mathcal E_{\mathrm{fed}} - \mathcal E^*$ be its error gap to the optimal estimate. Then,
\begin{align}
  \Delta_\mathrm{fed} = \left( 1 - \frac{1}{K} \right)  \cdot \frac{\tau^4}{ \sigmaloc^2 + \tau^2}.
\end{align}
\begin{proof}[Proofs of \cref{thm:optimal-local-error-gap,thm:optimal-global-error-gap}]
  From \cref{lem:optimal-error} we know that the error of the optimal estimator $\hat w_k(\lambda^*)$ of $w_k$ is $\mathcal E^* = \sigma_w^2 = \frac{\sigmaloc^2 (\sigmaloc^2 + K \tau^2) }{K(\sigmaloc^2 + \tau^2)}$.
  For $\mathcal E_{\mathrm{loc}}$, we know from earlier that the local estimator $\hat w_k = \hat w_k(0)$ of $w_k$ has an MSE/variance of $\mathcal E_{\mathrm{loc}} = \sigmaloc^2$.
  For $\mathcal E_{\mathrm{fed}}$, we can view $\bar w$ as a noisy observation of $w_k$ with MSE/variance $\mathcal E_{\mathrm{fed}} = \frac{\sigmaloc^2 + (K - 1)\tau^2}{K}$.
  The results of \cref{thm:optimal-local-error-gap,thm:optimal-global-error-gap} thus follow from re-arranging terms of $\mathcal E_{\mathrm{loc}} - \mathcal E^*$ and $\mathcal E_{\mathrm{fed}} - \mathcal E^*$ respectively.
\end{proof}

\textbf{\cref{lem:general-lambda-error}} (Error of $\hat w_k(\lambda)$).
Let $\mathcal E_\lambda \triangleq \mathbb E\left[ (w_k - \hat w_k(\lambda))^2 \mid \hat{w}_{k}, \hat{w}_{\setminus k}, \{ X_k \}_{k \in [K]} \right] $ be the error of $\mrmtl$ as a function of $\lambda$. Then,
\begin{align}
  \mathcal E_\lambda
    &= \left( 1 - \frac{1}{K} \right) \frac{\sigmaloc^2 + \lambda^2 \tau^2}{(\lambda + 1)^2} + \frac{\sigmaloc^2}{K}.
\end{align}
\begin{proof}[Proof of \cref{lem:general-lambda-error}]
  The result follows from the variance of $\hat w_k(\lambda)$ when viewed as a noisy observation of $w_k$. Specifically, from \cref{lem:wk-lambda} we know that $\hat w_k(\lambda)$ is an interpolation (parameterized by $\alpha \triangleq \frac{K+ \lambda}{(1 +\lambda)K}$) between $\hat w_k$ and $\hat w_{\setminus k}$. We thus have (with some rearrangement),
  \begin{align}
    \mathcal E_\lambda = \alpha^2 \cdot \sigmaloc^2 + (1 - \alpha)^2 \cdot \frac{K \tau^2 + \sigmaloc^2}{K - 1} = \left( 1 - \frac{1}{K} \right) \frac{\sigmaloc^2 + \lambda^2 \tau^2}{(\lambda + 1)^2} + \frac{\sigmaloc^2}{K}.
  \end{align}
\end{proof}
The proof of \cref{thm:dp-error-gap} follows directly from \cref{lem:general-lambda-error} by subtracting the common terms between private and non-private $\mrmtl$ estimates.

\subsection{The Case of Heterogeneous Privacy Requirements}
\label{supp-sec:analysis-varying-silo-params}

\newcommand{\sigmadpk}{\sigma_{k, \mathrm{DP}}}
\newcommand{\sigmalock}{\tilde{\sigma}_{k}}

In the main analysis, we assumed for simplicity that each silo has the same values of local dataset size $n$, local data sampling variance $\sigma^2$, and DP noise variance $\sigmadp^2$. In this section, we consider the case where each silo $k$ has custom values for the above, denoted as $n_k$, $\sigma_k^2$, and $\sigma_{k, \mathrm{DP}}^2$, to arrive at a silo-specific ``local variance'' $\sigmalock^2$:
\begin{align}
  \sigmalock^2 \triangleq \frac{\sigma_k^2}{n} + \frac{\sigmadpk^2}{n^2}.
\end{align}

With a slight abuse of notation, let us use $C + \mathcal N(\mu, \sigma^2)$ to denote drawing an iid sample of Gaussian noise with mean $\mu$ and variance $\sigma^2$ and add it to some value $C$.
Then, following \cref{lem:wk-as-hatwk-hatwkminus} and the same set of conditions, we can express the local model $\hat w_k$ and the non-local model $\hat w_{\setminus k}$ as:
\begin{align}
  \hat w_{k}
    &= w_k + \mathcal N(0, \tilde{\sigma}_k^2),  \\
  \hat w_{\setminus k}
    &=  \theta + \mathcal N \left(0, \frac{\tau^2}{K-1} + \frac{\sum_{j \neq k} \tilde{\sigma}_j^2 }{(K - 1)^2} \right) \\
    &= w_k + \mathcal N \left(0, \tau^2 + \frac{\tau^2}{K-1} + \frac{\sum_{j \neq k} \tilde{\sigma}_j^2 }{(K - 1)^2} \right) \\
    &= w_k + \mathcal N \left(0, \frac{K \tau^2}{K-1} + \frac{\sum_{j \neq k} \tilde{\sigma}_j^2 }{(K - 1)^2} \right),
\end{align}
and the true local center $w_k$ can be expressed in terms of $\hat w_k$ and $\hat w_{\setminus k}$ with silo-specific local variances:
\begin{lemma} \label{lem:heter-scalar-wk-in-hat-wk-hat-wkminus}
  Given $\hat w_k$ and $\hat w_{\setminus k}$, $w_k$ can be expressed as
  \begin{align}
      w_k = \bar{\mu}_k + \mathcal N(0, \bar{\sigma}_k^2)
  \end{align}
  with
  \begin{align}
      \bar{\sigma}_k^2 &\triangleq \left( \frac{1}{\tilde{\sigma}_k^2}+\frac{(K-1)^2}{(K-1) K \tau^{2}+ \sum_{j \neq k} \tilde{\sigma}_j^2}  \right)^{-1}, \\
      \bar{\mu}_k &\triangleq \frac{\bar{\sigma}_k^2}{\tilde{\sigma}_k^2 } \cdot \hat{w}_{k}+ \frac{(K-1)^2 \bar{\sigma}_k^2 }{(K-1) K\tau^{2}+ \sum_{j \neq k} \tilde{\sigma}_j^2} \cdot \hat{w}_{\setminus k}.
  \end{align}
\end{lemma}
\begin{proof}
  The proof follows similarly from \Cref{lem:wk-as-hatwk-hatwkminus}: we apply \Cref{lem:ivw-lemma} to the expressions of $\hat w_k$ and $\hat w_{\setminus k}$ as noisy observations of $w_k$.
\end{proof}

From here we can derive the best $\lambda^*_k$ \textit{specific to each silo}.
\begin{theorem}[Optimal $\lambda$ for silo $k$]
\label{thm:custom-optimal-lambda}
  The optimal choice of the \mrmtl regularization strength for silo $k$ is given by
  \begin{align}
    \lambda^*_k = \dfrac{\tilde{\sigma}_k^2}{\tau^2 + \dfrac{1}{K} \left( \dfrac{\sum_{j \neq k} \tilde{\sigma}_j^2 }{  K-1 } -\tilde{\sigma}_k^2  \right)}
  \end{align}
\end{theorem}
\begin{proof}
  The result follows from re-arranging and solving for $\lambda^*_k$ in the following equation, as in \Cref{thm:optimal-lambda}:
  \begin{align}
    \frac{K + \lambda^*_k}{(1 + \lambda^*_k)K } = \frac{\bar{\sigma}_k^2}{\tilde{\sigma}_k^2 }
  \end{align}
  where $\mu_k$ is defined in \Cref{lem:heter-scalar-wk-in-hat-wk-hat-wkminus}.
\end{proof}

\cref{thm:custom-optimal-lambda} suggests that if silos have varying local variance (as a result of varying local dataset sizes, inherent data variances, and silo-specific sample-level DP requirements), then it is optimal to have each silo choose its own $\lam^*_k$. In particular, the term $\dfrac{\sum_{j \neq k} \tilde{\sigma}_j^2 }{  K-1 } -\sigmalock^2 $ suggests that the optimal $\lambda^*_k$ depends on how much of an ``outlier'' is silo $k$ compared to the rest of silos:
if it has a smaller local variance, then it benefits from a smaller $\lambda^*_k$ since other silos are noisy; if it has a larger local variance, then $\lambda^*_k$ would be larger to encourage silo $k$ to ``conform''.
Note that under this simple analysis, one could have $\lam^*_k < 0$ (when $\sigmalock^2 > K \tau^2 + \frac{1}{K-1} \sum_{j \neq k} \tilde{\sigma}_j^2$); this case stems from having a gradually larger $\sigmalock^2$ and it suffices to choose $\lam^* \to \infty$ (i.e.\ fall back to FedAvg training).

\clearpage

\section{Additional Details for Private Hyperparameter Tuning \texorpdfstring{(\cref{sec:discussion})}{}}
\label{supp-sec:hparam-tuning}

In Section~\cref{sec:discussion} we took an alternative view on the benefits of personalization by examining the \textit{privacy cost} of tuning the hyperparameter responsible for navigating the tradeoff between costs from heterogeneity and privacy. In this section, we expand on the omitted details and provide more discussions.

\subsection{Additional Background on Privacy Costs of Hyperparameter Tuning}
\label{supp-sec:tnb-hpo}

We consider the typical hyperparameter tuning procedure (denoted as HPO) where a base algorithm $M$, such as one run of a machine learning algorithm to produce a model, is executed $h$ times with different hyperparameters and the best result is recorded. The hyperparameters may be high-dimensional; e.g.\ vector of learning rate and batch size.
Typically, $h$ is a constant (e.g.\ we sweep a fixed grid of learning rates and batch sizes).

On a high level, the works of Liu and Talwar~\cite{liu2019private} and Papernot and Steinke~\cite{papernot2022hyperparameter} show that, with a constant $h$, one can construct $M$ where the privacy guarantee of the HPO output is (roughly) a factor of $h$ weaker than the original privacy guarantee.
Specifically, \cite{papernot2022hyperparameter} gives the following proposition that applies to both pure DP and R\'enyi DP (and thus related to approximate DP).
\begin{proposition}[Proposition 17 of \cite{papernot2022hyperparameter}]
  For any $\eps > 0$ and any $\alpha > 1$, there exist some algorithm $M$ that satisfy $(\eps, 0)$-DP and the corresponding HPO algorithm $A$ that runs $M$ for $h$ times and returns only the best result,  such that
  \begin{enumerate}
    \item $A$ is not $(\tilde \eps, 0)$-DP for any $\tilde \eps < h\eps$, and
    \item $A$ is not $(\alpha, \tilde\eps(\alpha))$-R\'enyi DP for any $\tilde\eps(\alpha) < \hat\eps(\alpha)$ where
    \begin{align} \label{supp-eq:rdp-fixed-hparam-search}
      \hat\eps(\alpha) = h \eps - \frac{h \log(1 + \exp(-\eps))}{\alpha - 1}.
    \end{align}
  \end{enumerate}
\end{proposition}

Both the pure DP and RDP results suggest that running the typical HPO procedure with a fixed $h$ can be as bad as running $M$ and $h$ times and releasing \textit{all} results.

The authors of \cite{liu2019private} and \cite{papernot2022hyperparameter} provided a mitigating strategy where one simply makes $h$ (the number of tuning runs) a random variable. The authors of~\cite{papernot2022hyperparameter} provide an improved analysis and considered two distributions for $h$ in particular: the truncated negative binomial (TNB) distribution and the Poisson distribution.

Here, we focus on TNB as it provides a spectrum of results for privacy-utility tradeoff, allowing us to choose several points on this spectrum that offer small (but realistic) privacy overhead. We combine it our analysis in Section~\cref{sec:analysis} to motivate our discussions on the practical complications of deploying personalization methods.

Specifically, the probability mass function of the TNB distribution is given by
\begin{align}
  f(h) = \begin{cases}
    \frac{(1-\gamma)^{h}}{\gamma^{-\eta}-1} \prod_{\ell=0}^{h-1}\left(\frac{\ell+\eta}{\ell+1}\right) & \text{for } \eta \neq 0, \\
    \frac{(1-\gamma)^{h}}{h \log (1 / \gamma)} & \text{for } \eta = 0,
  \end{cases}
  \text{ with expected value } \mathbf E[h] = \begin{cases}
    \frac{\eta (1-\gamma)}{\gamma (1-\gamma^\eta)}  & \text{for } \eta \neq 0, \\
    \frac{1 / \gamma-1}{\log (1 / \gamma)}  & \text{for } \eta = 0, \\
  \end{cases} \nonumber
\end{align}
where $\eta$ and $\gamma$ are parameters of the TNB distribution. The privacy of the HPO output using $h$ sampled from the TNB distribution is given by the following theorem from~\cite{papernot2022hyperparameter}.
\begin{theorem}[Theorem 2 of \cite{papernot2022hyperparameter}] \label{supp-thm:tnb-rdp}
  For any $\eta > 1$, any $\gamma \in (0, 1)$, and any algorithm $M$ that satisfy both $(\alpha_1, \eps(\alpha_1))$-RDP and $(\alpha_2, \eps(\alpha_2))$-RDP with $\alpha_1 > 1$ and $\alpha_2 \ge 1$, the HPO algorithm $A$ that runs $M$ for $h$ times and returns only the best result with $h \sim \operatorname{TNB}(\eta, \gamma)$ satisfies $(\alpha_1, \tilde\eps)$-RDP with
  \begin{align}
    \tilde\eps = \eps(\alpha_1) + (1 + \eta) \left(1-\frac{1}{\alpha_2}\right) \eps(\alpha_2)+\frac{(1+\eta) \log (1 / \gamma)}{\alpha_2}+\frac{\log \mathbf{E}[h]}{\alpha_1 -1}.
  \end{align}
\end{theorem}

\cref{supp-thm:tnb-rdp} tells us that by drawing $h$ from the TNB distribution, the privacy cost of the HPO procedure is a \textit{constant} factor of the original cost of $M$ for pure DP (when $\alpha_1, \alpha_2 \to \infty$), or up to a factor of $\log \mathbf E[h]$ for approximate DP. This substantially improves the default HPO procedure.

\subsection{Interpretation of \texorpdfstring{\cref{fig:hparam-cost}}{Hyperparameter Tuning Costs for MR-MTL}}
\label{supp-sec:hparam-cost-interpretation}

In \cref{fig:hparam-cost}, we considered 6 pairs of $(\eta, \gamma)$ to arrive at the same $\mathbf E[h] = 10$ (i.e.\ on average we want to try 10 values of $\lam$ when tuning \mrmtl).
Intuitively, for a fixed $\mathbf E[h]$, we can expect a smaller $\eta$ (and consequently a smaller $\gamma$) to give a tighter privacy guarantee from \cref{supp-thm:tnb-rdp} (i.e.\ the privacy cost of hyperparameter tuning is smaller).
At the same time, however, the TNB distribution would be more concentrated around $h=1$ with a smaller $\eta$, which means we may end up trying only 1 hyperparameter even when $\mathbf E[h]$ is large. $(\eta, \gamma)$ thus provides an implicit privacy-utility tradeoff.

In \cref{fig:hparam-cost} we chose a list of $\eta$ values that favors strong privacy while being realistic, since picking $\eta \to -1$ (or just $\eta < -0.5$) would give a even smaller privacy overhead, but it will perform very poorly utility-wise and was not considered by~\cite{papernot2022hyperparameter}.
We then applied \cref{supp-thm:tnb-rdp} with these $(\eta, \gamma)$ values to the the expected error of \mrmtl as a function of $\lam$ under mean estimation, by computing and applying the extra noise needed to account for the privacy cost of tuning.
Here, the closed form of the expected error of \mrmtl is derived in \cref{lem:general-lambda-error}, but in \cref{fig:hparam-cost} we ran numerical simulations with 500 repetitions.

\cref{fig:hparam-cost}~(a) and (c) suggests that there are settings in which, after applying \cref{supp-thm:tnb-rdp}, the best $\lam^*$ one could find would already give an error that is larger than simply running FedAvg and local training \textit{without any tuning}, respectively (the \textcolor{OliveGreen}{\textbf{green}} curves). \cref{fig:hparam-cost}~(b) suggests that in less extreme cases, the private hyperparameter tuning cost would significantly diminish the utility improvement of \mrmtl.

\subsection{Implementing Hyperparameter Tuning under Silo-Specific Sample-Level DP}

Under silo-specific sample-level privacy, it can be intricate to implement the improved, randomized hyperparameter tuning protocol described above in~\cref{supp-sec:tnb-hpo}.
For silo $k$ to satisfy its own $(\eps_k, \delta_k)$ sample-level DP target, it must draw $h$ \textit{independently} to all other silos, and it must only return the best result of hyperparameter tuning at the \textit{end of training} (instead of returning the best model iterate at every round).

Using \mrmtl as an example, the above has several implications:
\begin{enumerate}[leftmargin=*]
  \item Each silo should now maintain a list of personalized models, each corresponding to a choice of $\lam$.
  \item Depending on the specific distribution of $h$, silos may or may not end up trying the same set of $\lam$ values.
  \item During training rounds, silos must return model updates for -- and consequently regularize their personalized models towards -- a ``pivot'' model using some public $\bar \lam$.
\end{enumerate}

The analysis in \cref{supp-sec:analysis-varying-silo-params} suggests that having silo-specific choices of $\lam$ (\#2 above) should not have adverse effects,
but the above adaptations in general may influence final utility or convergence properties in iterative learning settings.

Other more sophisticated implementation (potentially tailored to the specific personalization method) may also be possible.
It would be an interesting future direction to study the best approaches for implementing private hyperparameter tuning under silo-specific sample-level privacy.

\clearpage

\section{Future Directions for Auto-tuning / Online Estimation\texorpdfstring{ of $\lambda$}{}}
\label{supp-sec:future-directions}

Apart from tuning $\lam$ for \mrmtl (or other hyperparameters for other personalization methods) non-adaptively with grid search or random search, another approach is to leverage some form of online estimation as in~\cite{van2018three,andrew2021differentially,pichapati2019adaclip} (mentioned in Section~\cref{sec:discussion}). That is, $\lam$ is adjusted at every training round or iteration such that it gradually converges to a good $\lam^*$ (in expectation).

This line of approaches would be promising if their privacy overhead (if any) can be contained within a small factor of the original privacy guarantee---smaller than that offered by the randomized tuning procedure discussed above in \cref{supp-sec:tnb-hpo}.
However, since the hyperparameters of interest are those that navigate the tradeoff between costs from heterogeneity and privacy (such as $\lam$ for \mrmtl), there are several factors that make auto-tuning particularly challenging in our setting:
\begin{itemize}[leftmargin=*]
  \item \textbf{A universal and practical measure of heterogeneity is generally unavailable}.
  For \mrmtl, our analysis in Section~\cref{sec:analysis} suggests that the optimal $\lam^*$ depends on a measure of data heterogeneity across silos ($\tau^2$).
  However, in practice, the level of heterogeneity is abstract and hard to quantify even with non-private oracle access to the silo datasets.
  Past work has devised custom measures of heterogeneity such as the variance of client (silo) model or gradient updates~\cite{karimireddy2020scaffold,fedprox}, or the difference between the true global loss and the aggregate of the true local losses, but these measures may not be consistent throughout training~\cite{karimireddy2020scaffold,fedprox} such that they cannot provide a useful estimate of the heterogeneity which should be invariant during training. These metrics may also simply not be obtainable in practice~\cite{Li2020On}.
  The work of~\cite{smith2017federated} considers the use of a client-relationship matrix, but such matrix may be too large ($K^2$ elements) and too noisy under privacy as we may need to directly or implicitly noise every coordinate of the matrix.
  Directly estimating the dataset statistics may also be challenging, since silos may apply custom dataset transformations (e.g.\ standardization or augmentation) that can drastically alter such statistics without affecting the learning dynamics.
  \item \textbf{The estimated heterogeneity may have high variance.}
  In addition to the above, heterogeneity would generally be measured on the client distribution (as some form of ``variance'' across clients, as in~\cite{karimireddy2020scaffold}), but there is generally a limited number of clients in cross-silo settings, meaning that this estimated ``variance'' as a measure of heterogeneity may itself have high variance.
  Moreoever, such measures must also be estimated privately, which means that it would have even higher variance from the additional noise.
  This could be problematic since the best hyperparameter for personalization may depend on such heterogeneity measure (as is the case for \mrmtl in simplified settings), and the precision of a potential auto-tuning precedure may suffer as a result.
  \item \textbf{Auto-tuning procedures may need to be developed specifically for each method.}
  For example, local finetuning may use the number of federated training rounds as the hyperparameter for determining ``how much'' to personalize, but such hyperparameter would likely be auto-tuned in a drastically different way compared to auto-tuning for $\lam$ for \mrmtl.
  It is in general an open question as to whether method-specific auto-tuning procedures will always outperform the simple randomized procedure described in \cref{supp-sec:tnb-hpo}, in terms of both precision and privacy overhead.
\end{itemize}

\clearpage

\section{Additional Experimental Results}

\subsection{Local Finetuning \texorpdfstring{(Extension of \cref{fig:finetune-gap})}{}}

In \cref{fig:finetune-gap}, we made the observation that local finetuning~\cite{wang2019federated,yu2020salvaging,cheng2021fine} as a personalization strategy (warm-starting from a global model and continue training using local data) may not always improve utility as expected.
Extending on \cref{fig:finetune-gap}, we examine the behavior of local finetuning starting at different stages of training (under a fixed total number of rounds) as well as under varying privacy budgets in \cref{supp-fig:finetune-sweep}. We observe that the phenomenon illustrated in \cref{fig:finetune-gap} is indeed reflected across different settings: as soon as local finetuning begins, the DP utility gap can quickly widen, and the utility of finetuning under privacy can roughly reduce to the utility of local training.

Note, however, that these observations serve to give insights into an interesting behavior of local finetuning under silo-specific sample-level DP and \textit{do not} preclude the possibility that local finetuning can still outperform both local training and FedAvg. Indeed, local finetuning was observed to outperform local training in, e.g., \cref{fig:pareto-merged}~(d) and \cref{supp-fig:pareto-school-t200}.

\begin{figure}[ht]
  \centering
  \includegraphics[width=0.328\linewidth]{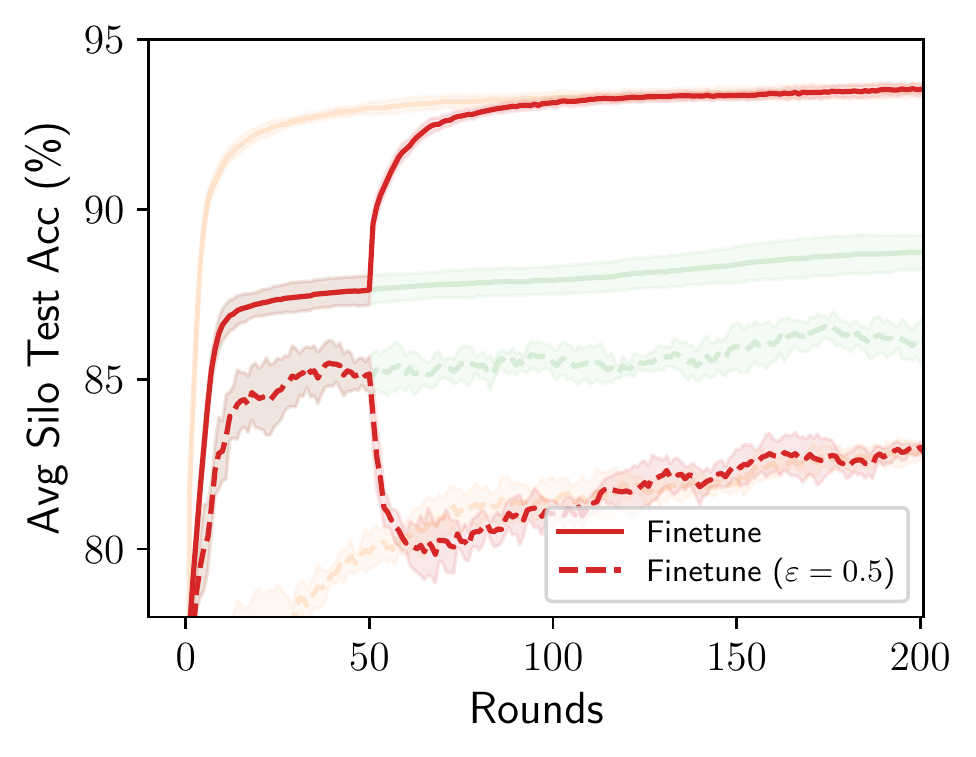}
  \includegraphics[width=0.328\linewidth]{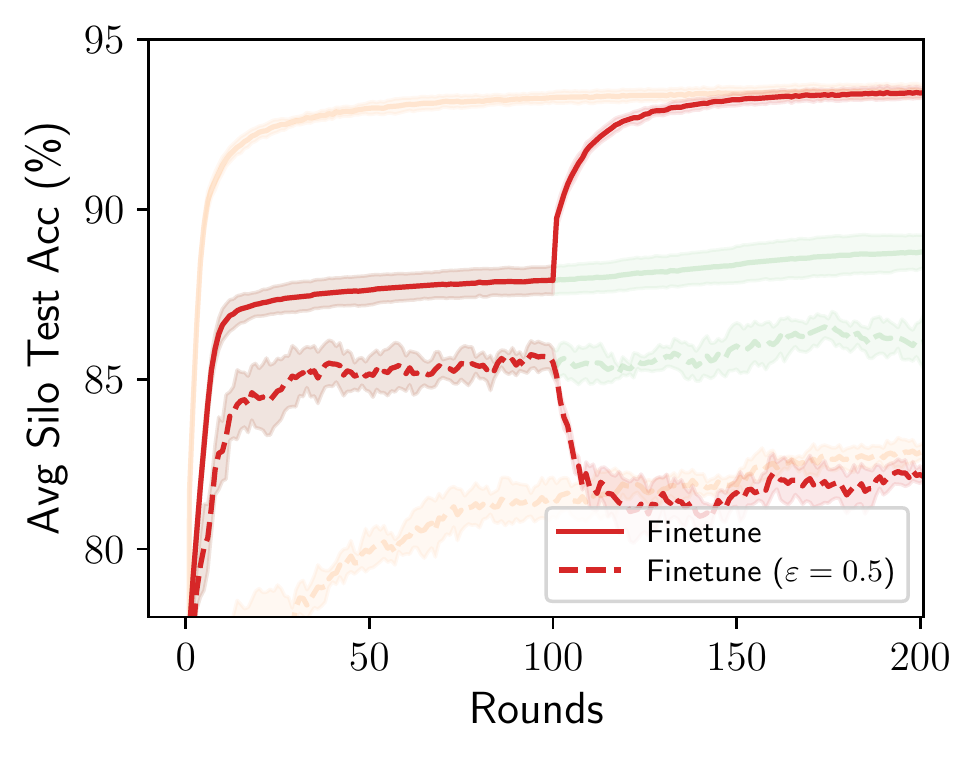}
  \includegraphics[width=0.328\linewidth]{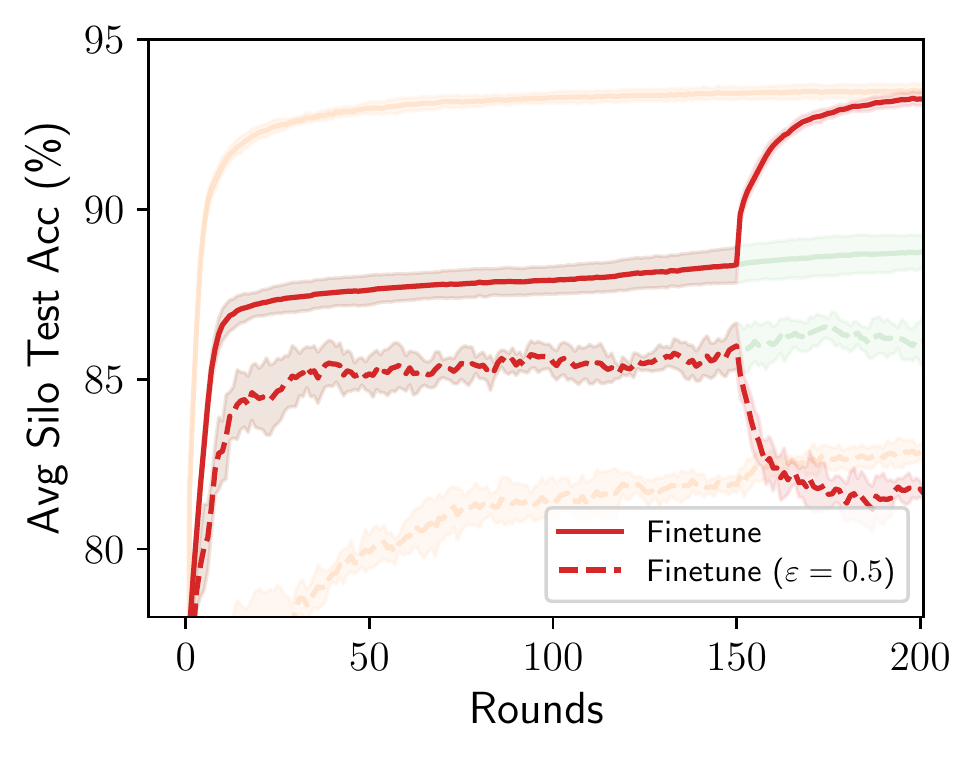}
  \includegraphics[width=0.328\linewidth]{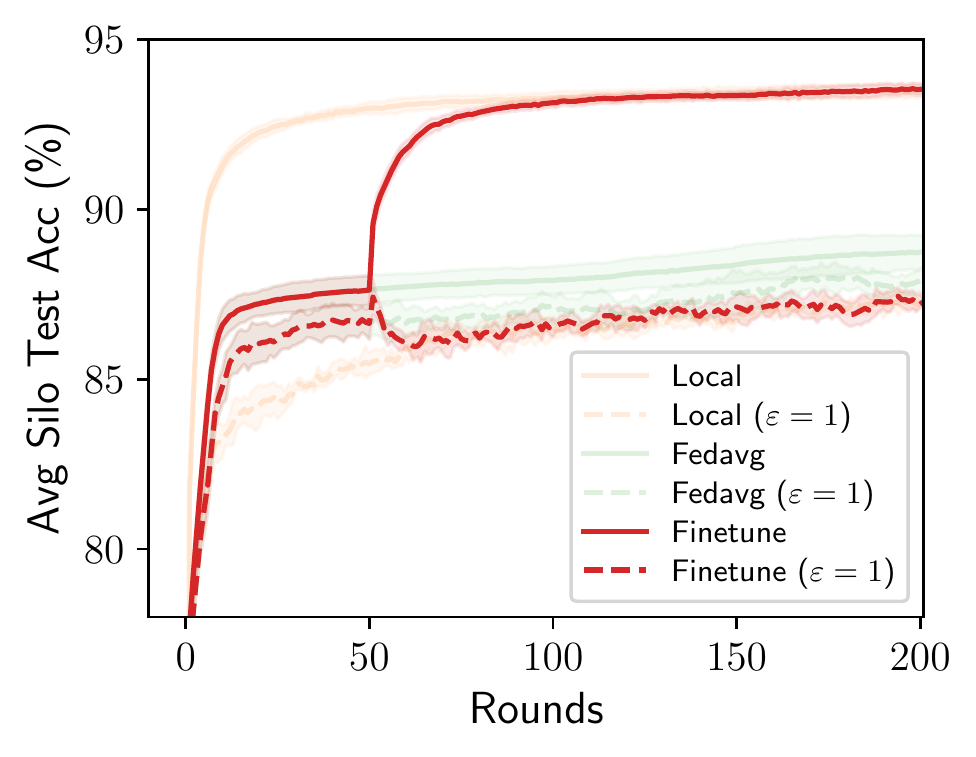}
  \includegraphics[width=0.328\linewidth]{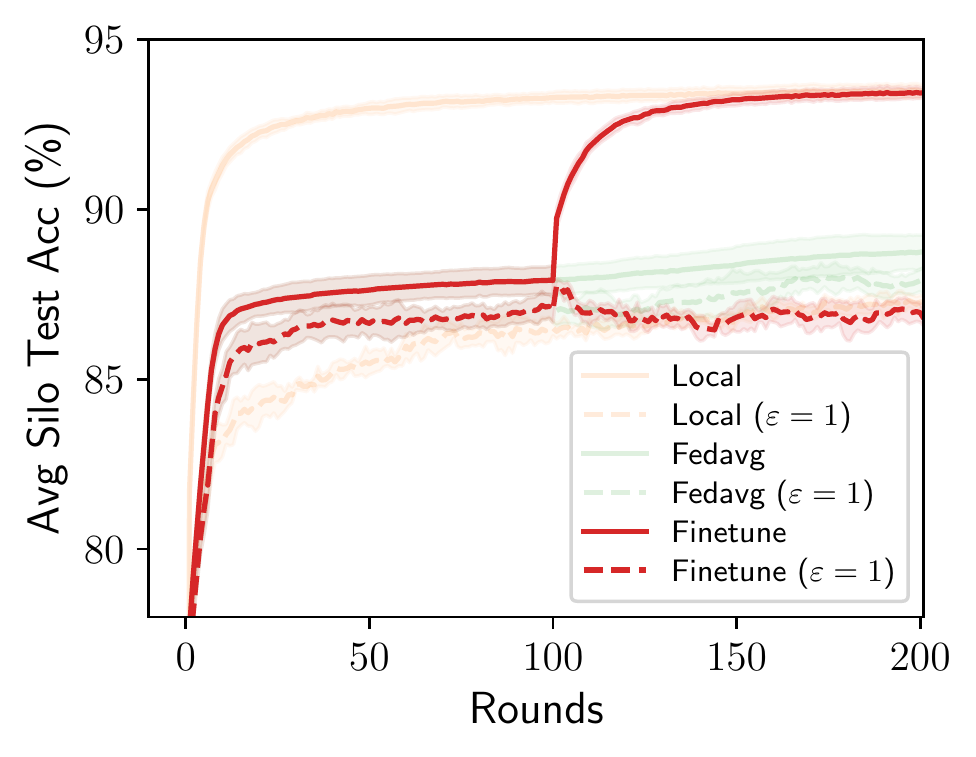}
  \includegraphics[width=0.328\linewidth]{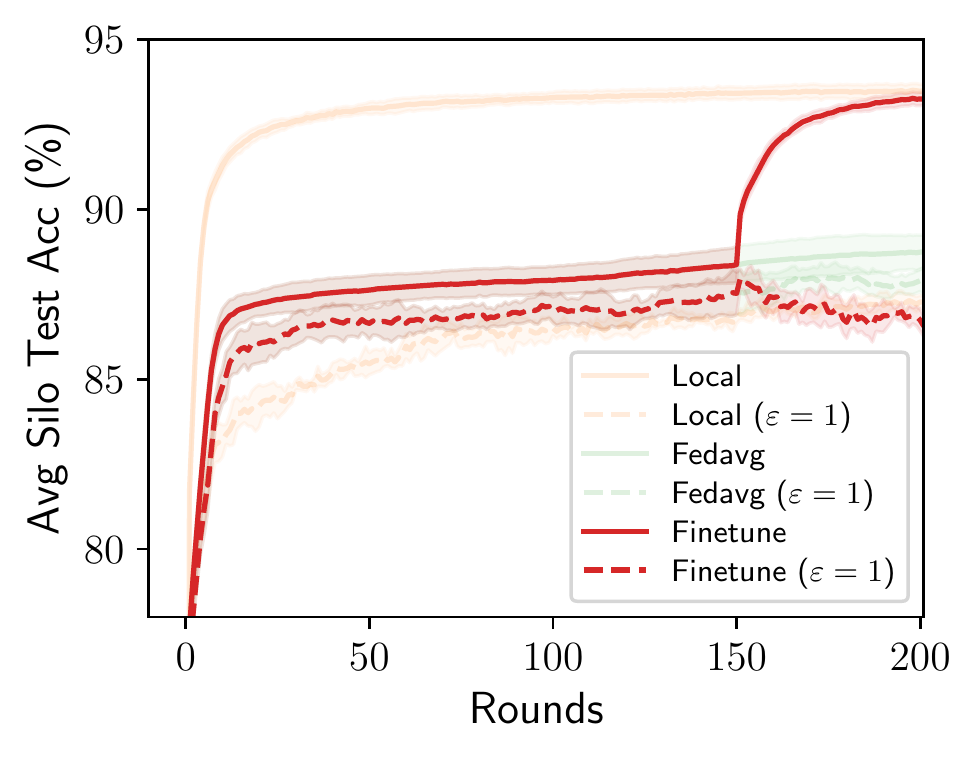}
  \includegraphics[width=0.328\linewidth]{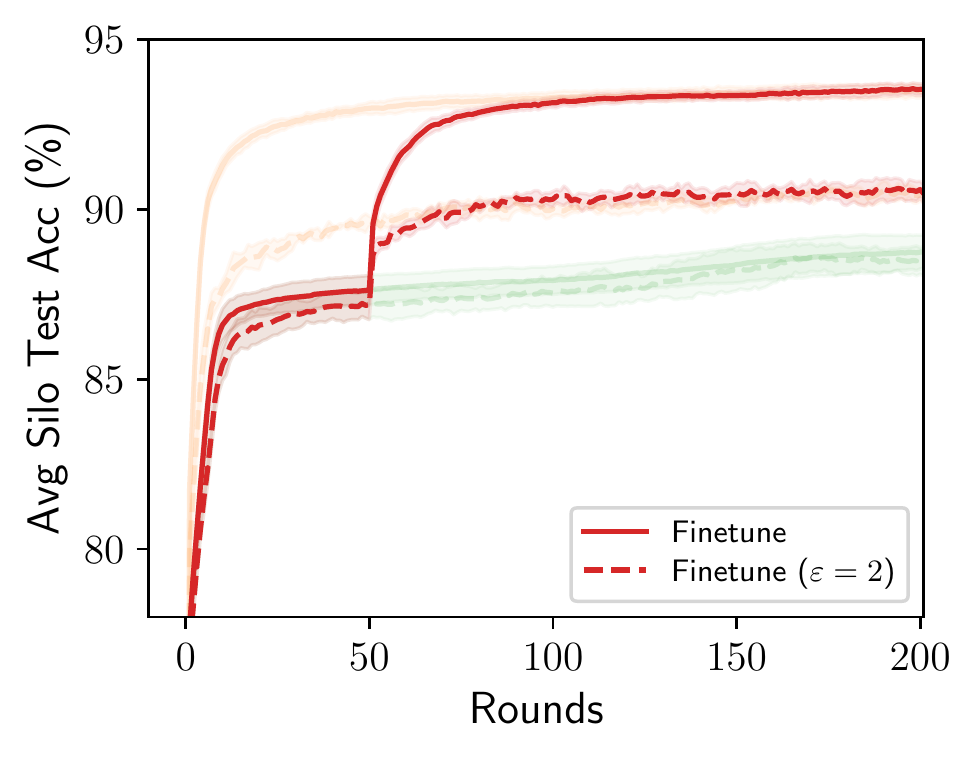}
  \includegraphics[width=0.328\linewidth]{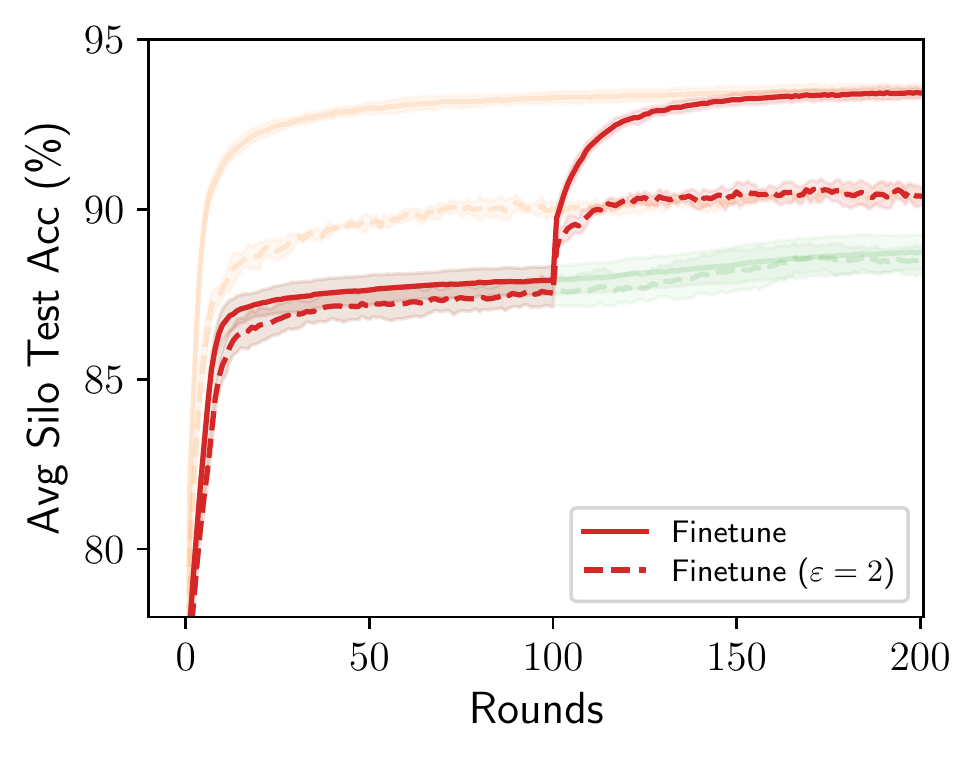}
  \includegraphics[width=0.328\linewidth]{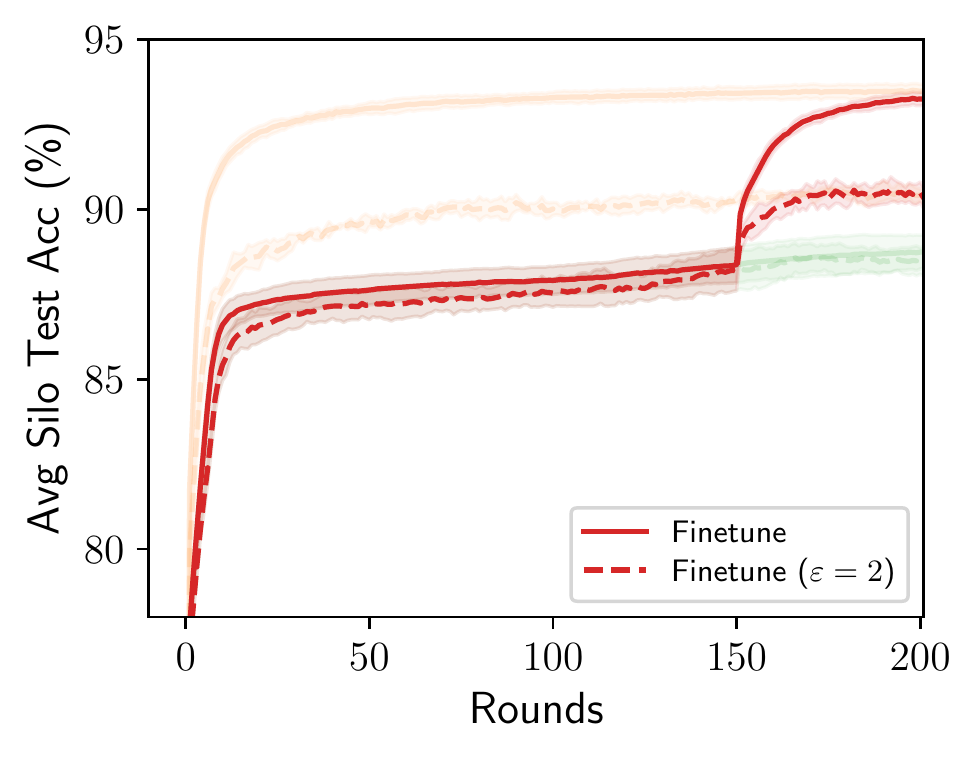}
  \caption{
    (Extension of \cref{fig:finetune-gap}) \textbf{Behavior of local finetuning} starting at different stages of training (25\%, 50\%, 75\%) and under varying privacy budgets ($\eps \in [0.5, 1, 2], \delta=10^{-7}$) for $T=200$ rounds on the \textbf{Vehicle} dataset.
  }
  \label{supp-fig:finetune-sweep}
\end{figure}

\subsection{Utility of \texorpdfstring{$\mrmtl$}{MR-MTL} as a Function of \texorpdfstring{$\lambda$}{Lambda} \texorpdfstring{(Extension of \cref{fig:mrmtl-bump})}{}}
\label{supp-sec:bump-extension}

  We also extend on \cref{fig:mrmtl-bump} to consider varying privacy budgets, and the results are shown in \cref{supp-fig:mrmtl-varying-bumps} (Vehicle), \cref{supp-fig:mrmtl-varying-bumps-school} (School), and \cref{supp-fig:mrmtl-varying-bumps-cifar10} (Heterogeneous CIFAR-10). Note that on School, the test metric is MSE thus lower is better. There are several notable observations:
  \begin{itemize}[leftmargin=*]
    \item \textbf{$\lambda^*$ decreases as $\eps$ grows (weaker privacy)}: With weaker privacy (larger $\eps$), we have a smaller optimal $\lambda^*$ (i.e. the location of the utility ``bump'' under DP gradually shifts to the left). This behavior is characterized by \cref{thm:optimal-lambda}, where under stronger privacy, silos benefit from more federation as a means to reduce the effect of privacy noise.
    \item \textbf{$\mrmtl$ does not always outperform FedAvg} (at high privacy regimes):
    One minor caveat is that \cref{thm:optimal-global-error-gap} says that the utility gap from $\mrmtl$ to FedAvg is always nonnegative under federated mean estimation;
    however, as discussed in Sections~\cref{sec:method} and \cref{sec:analysis}, $\mrmtl$ does not approach FedAvg with larger $\lambda$ in general learning settings since the \mrmtl objective may become too hard to solve via (DP-)SGD.
    Thus, despite its utility advantage over local training at $\lambda^*$, it may never reach the performance of FedAvg. See, e.g., subplots of $\eps \in [0.1, 0.2]$ in \cref{supp-fig:mrmtl-varying-bumps}.
    \item \textbf{Utility advantage of \mrmtl~($\lam^*$) over local/FedAvg changes with $\eps$}: Observe that with larger $\eps$, the utility advantage of \mrmtl~($\lam^*$) is larger compared to FedAvg and is smaller compared to local training. This behavior is characterized by \cref{thm:optimal-local-error-gap,thm:optimal-global-error-gap}.
  \end{itemize}

  Note also that for Heterogeneous CIFAR-10 (\cref{supp-fig:mrmtl-varying-bumps-cifar10}), Ditto~\cite{li2021ditto} exhibits a slightly different interpolation behavior compared to \mrmtl (e.g.\ it has larger optimal $\lambda^*$ under privacy), though at $\lambda^*$ its utility underperforms that of \mrmtl under privacy.

  \begin{figure}[t]
    \centering
    \includegraphics[width=0.329\linewidth]{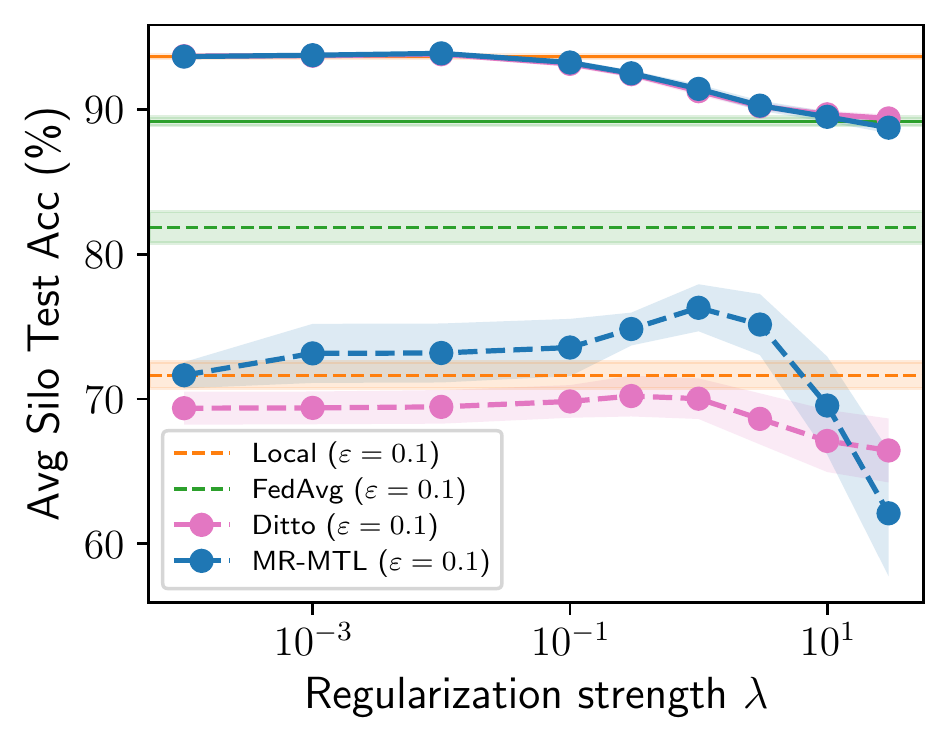}
    \includegraphics[width=0.329\linewidth]{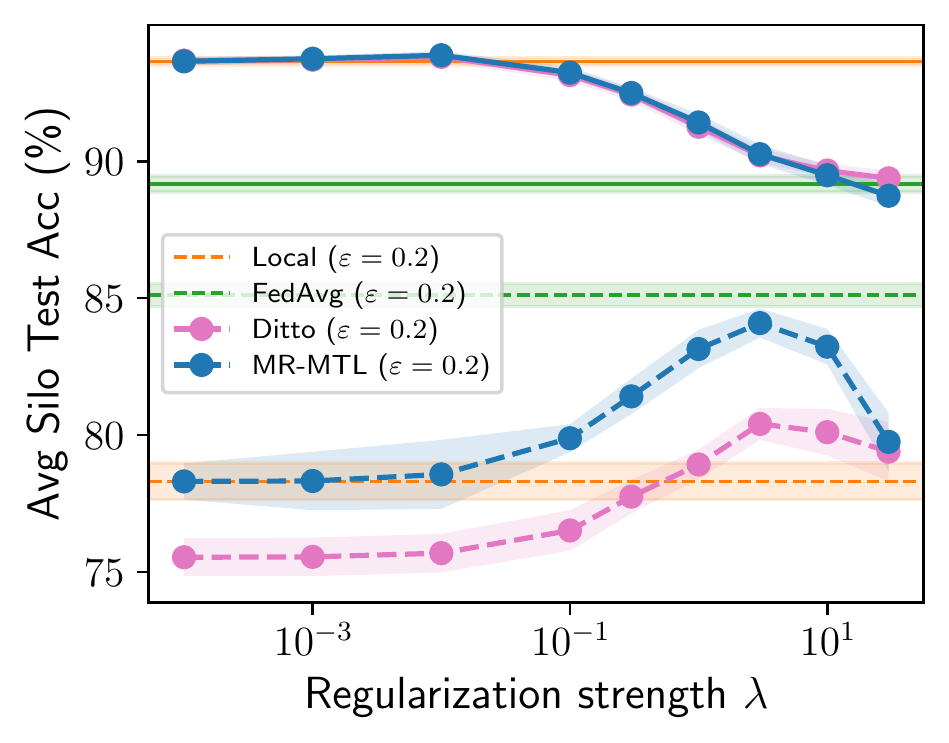}
    \includegraphics[width=0.329\linewidth]{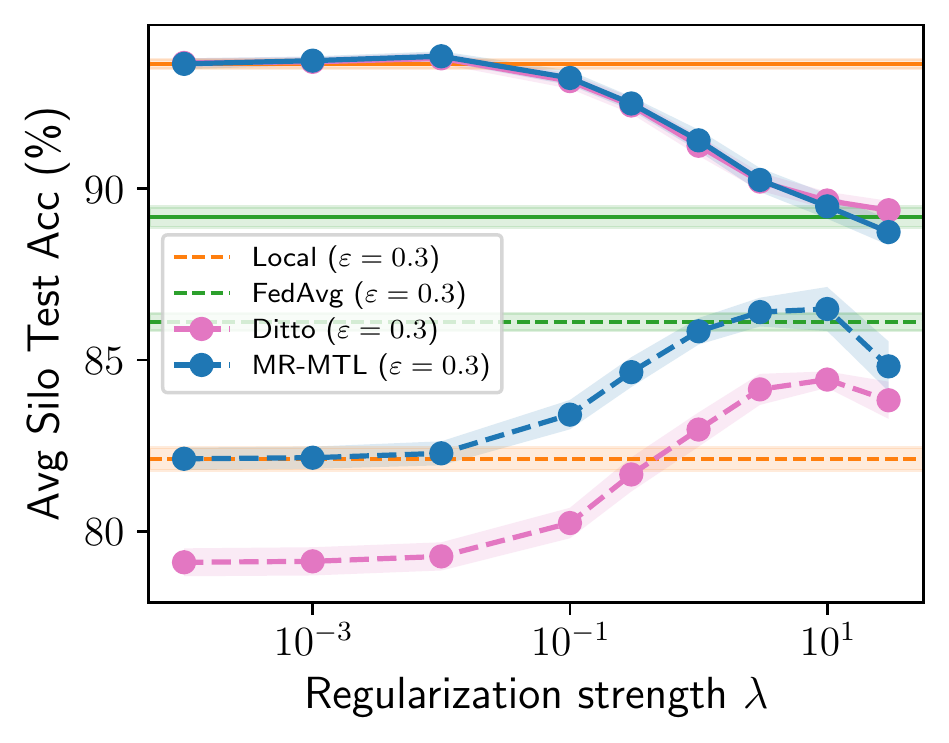}
    \includegraphics[width=0.329\linewidth]{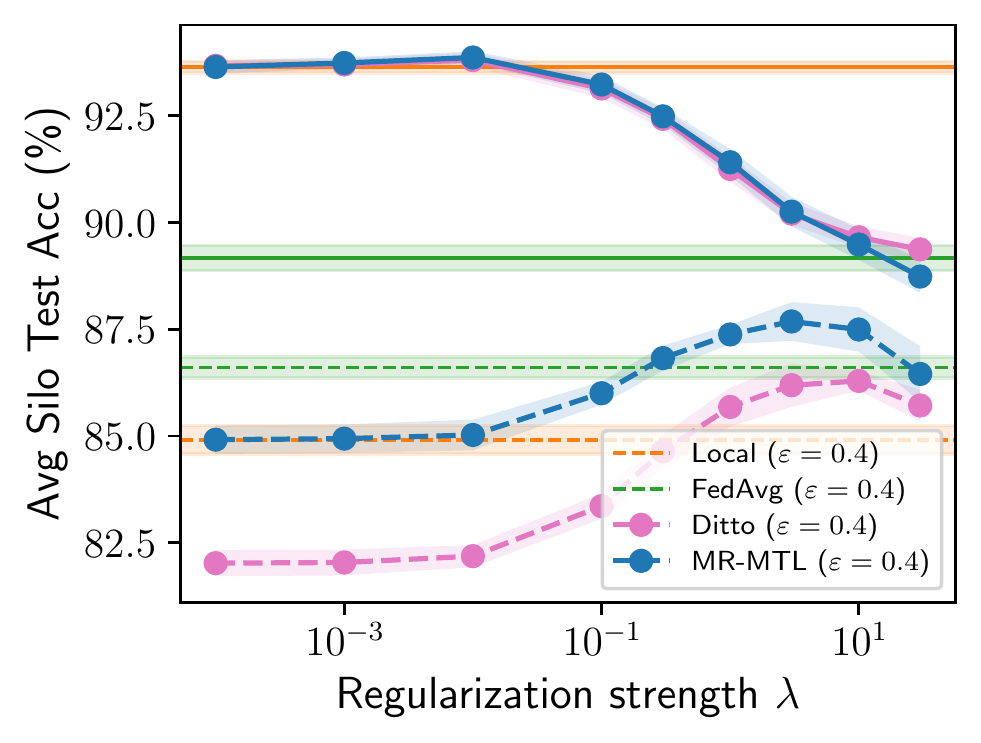}
    \includegraphics[width=0.329\linewidth]{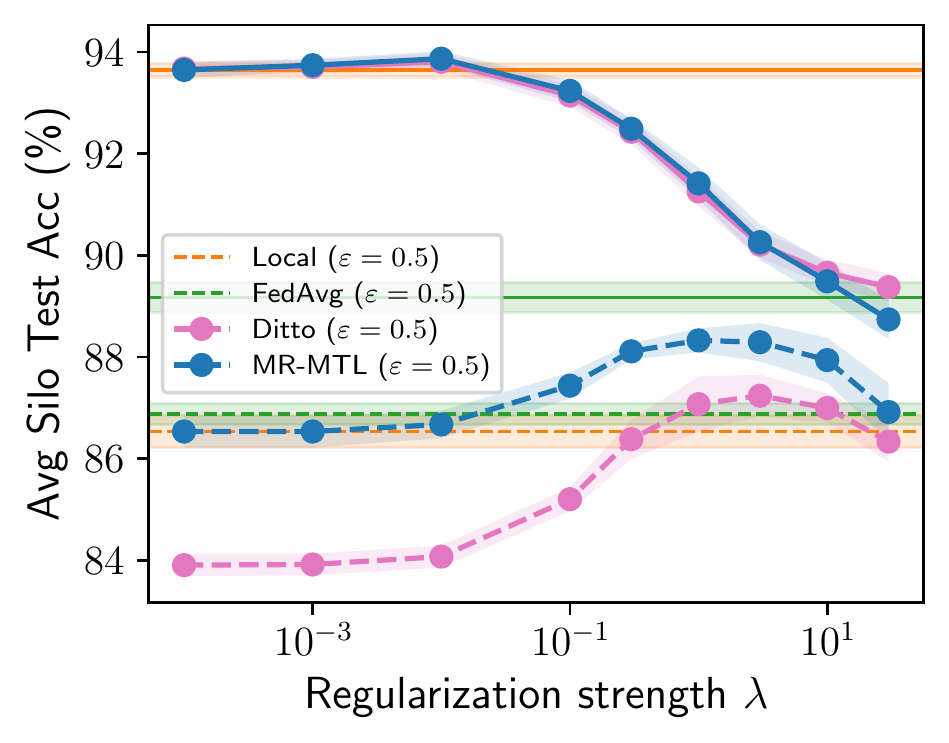}
    \includegraphics[width=0.329\linewidth]{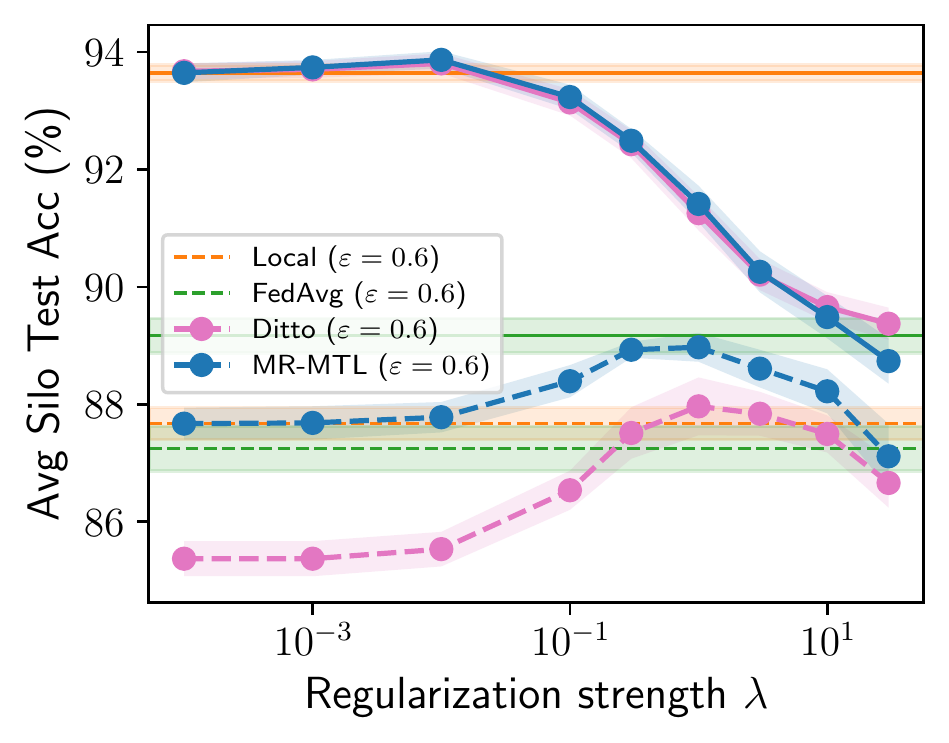}
    \includegraphics[width=0.329\linewidth]{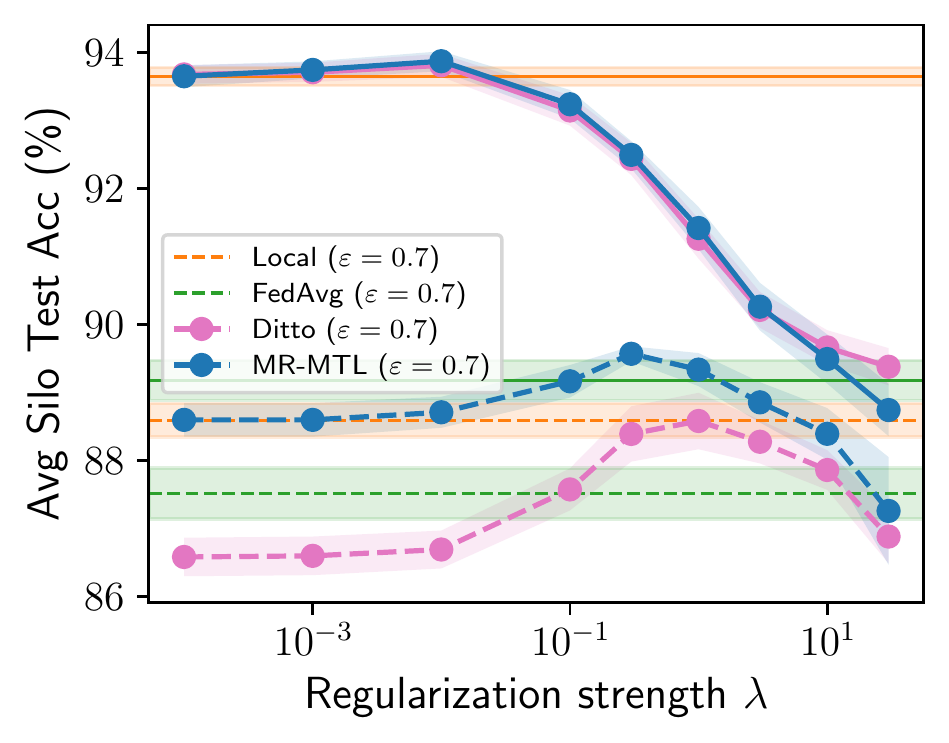}
    \includegraphics[width=0.329\linewidth]{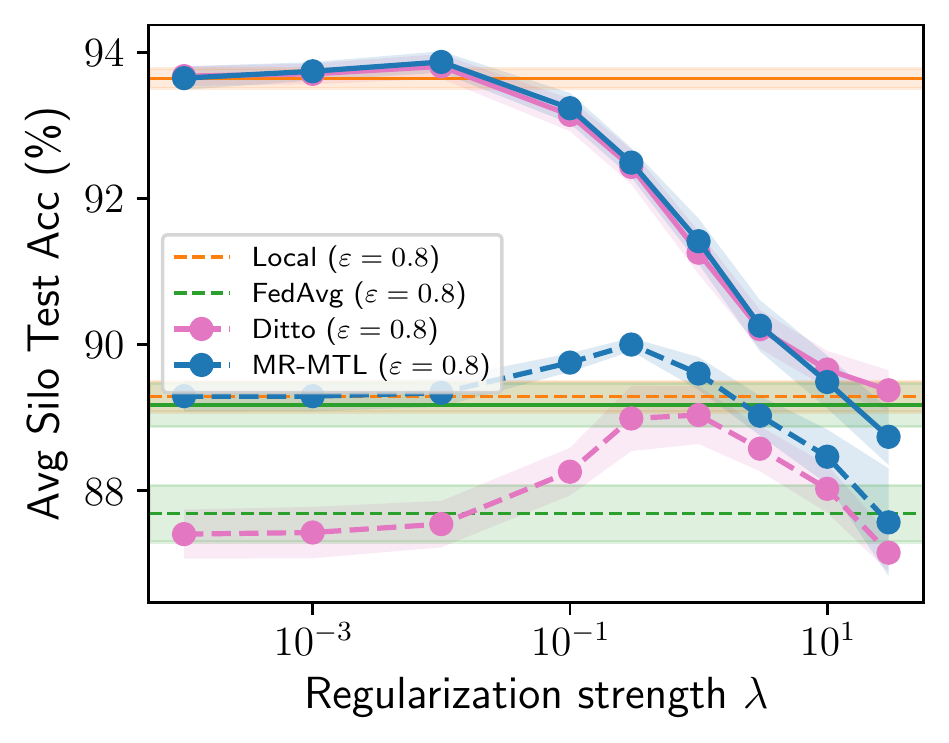}
    \includegraphics[width=0.329\linewidth]{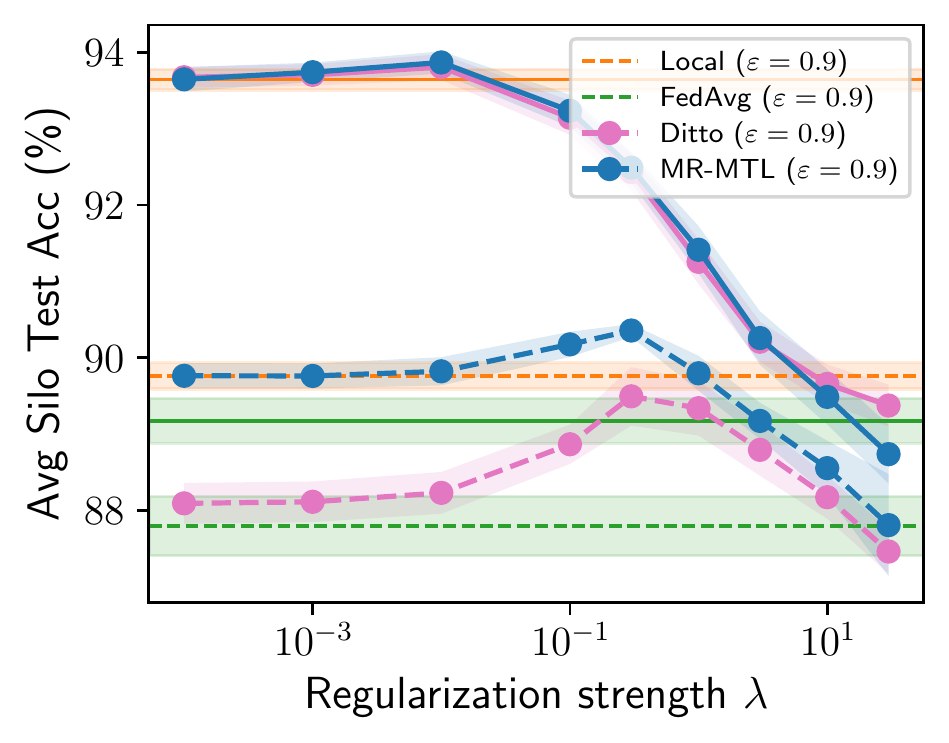}
    \caption{
      (Extension of \cref{fig:mrmtl-bump}) \textbf{Behavior of $\mrmtl$ as a function of $\lambda$ across varying privacy budgets} ($\eps \in [0.1, 0.2, ..., 0.9], \delta=10^{-7}$) for $T=400$ rounds on the \textbf{Vehicle} dataset.
      Solid lines refer to the non-private runs (same across all plots).
    }
    \label{supp-fig:mrmtl-varying-bumps}
  \end{figure}
  \begin{figure}[t]
    \centering
    \includegraphics[width=0.329\linewidth]{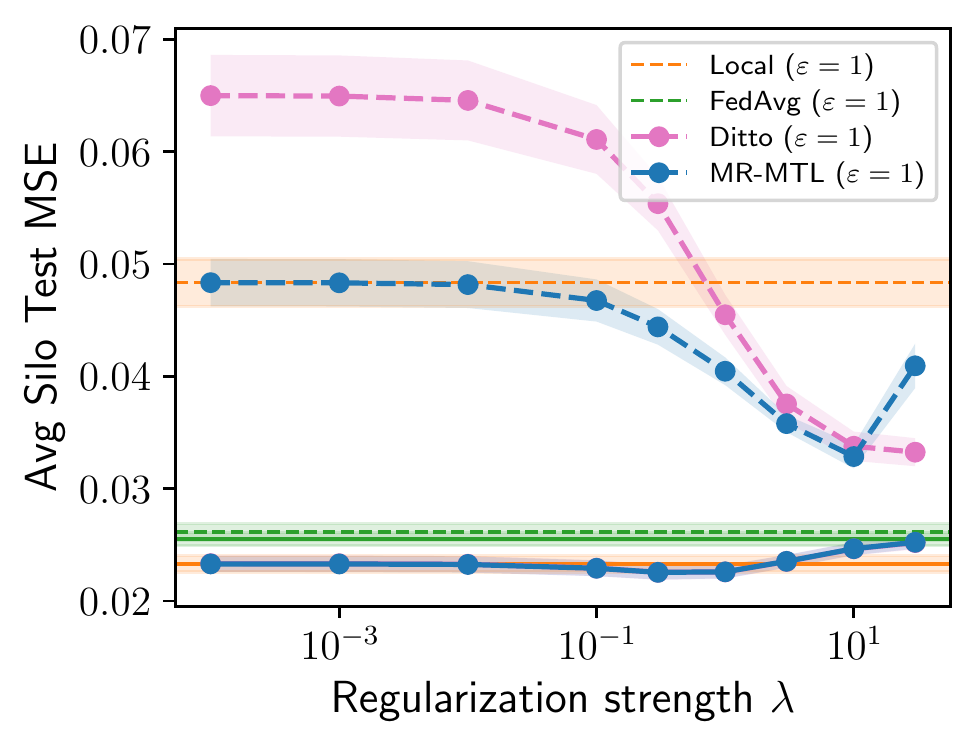}
    \includegraphics[width=0.329\linewidth]{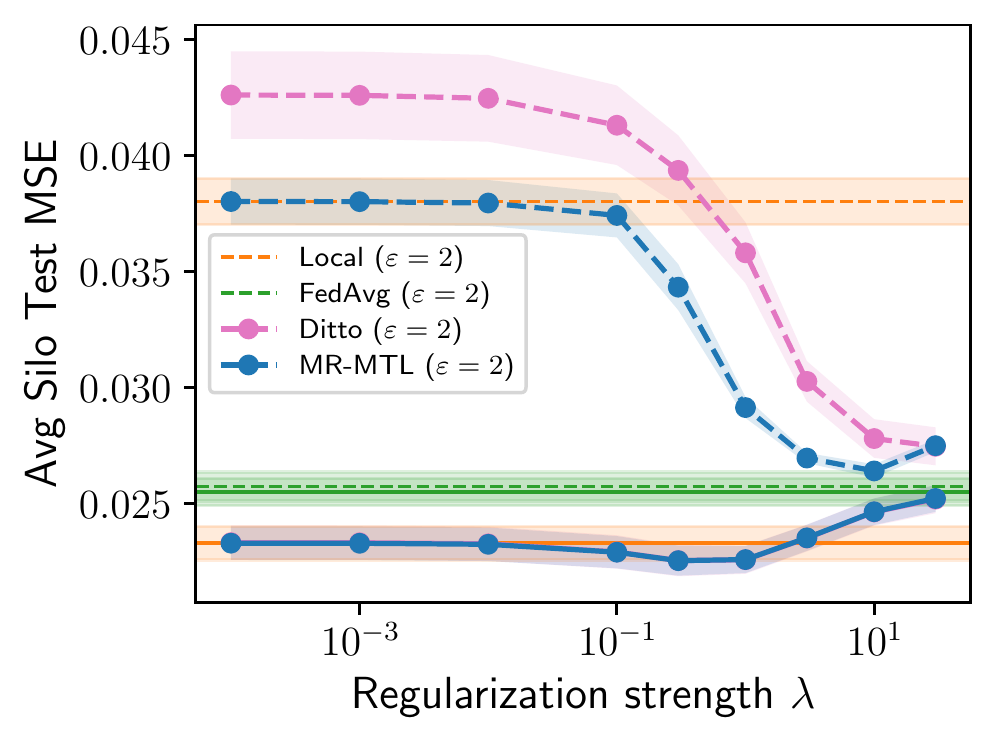}
    \includegraphics[width=0.329\linewidth]{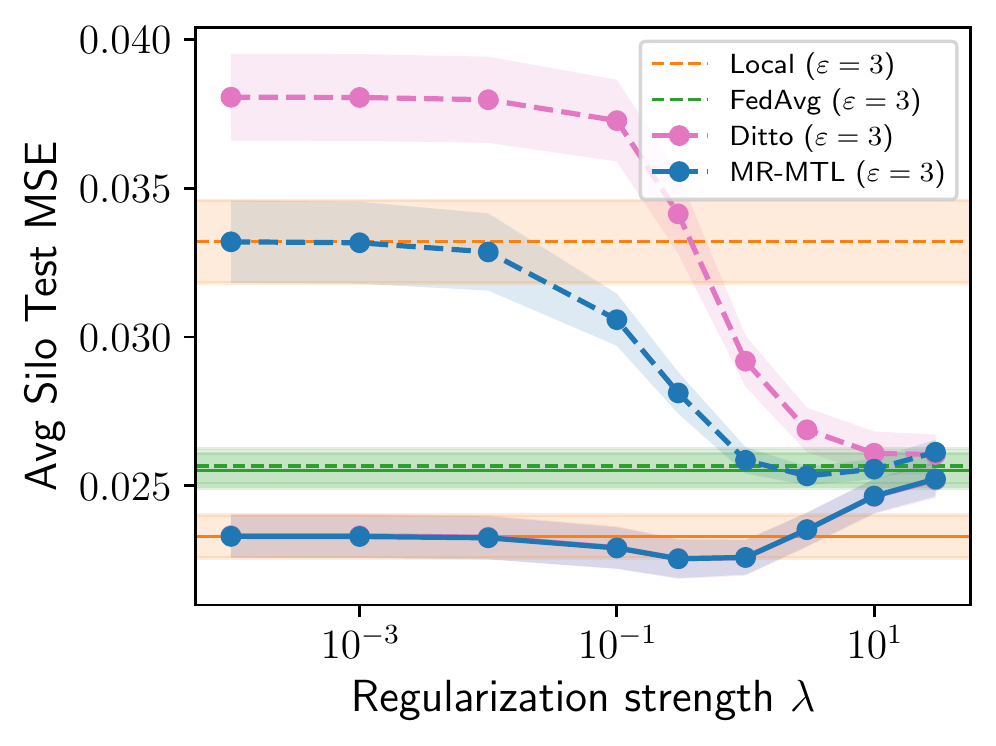}
    \includegraphics[width=0.329\linewidth]{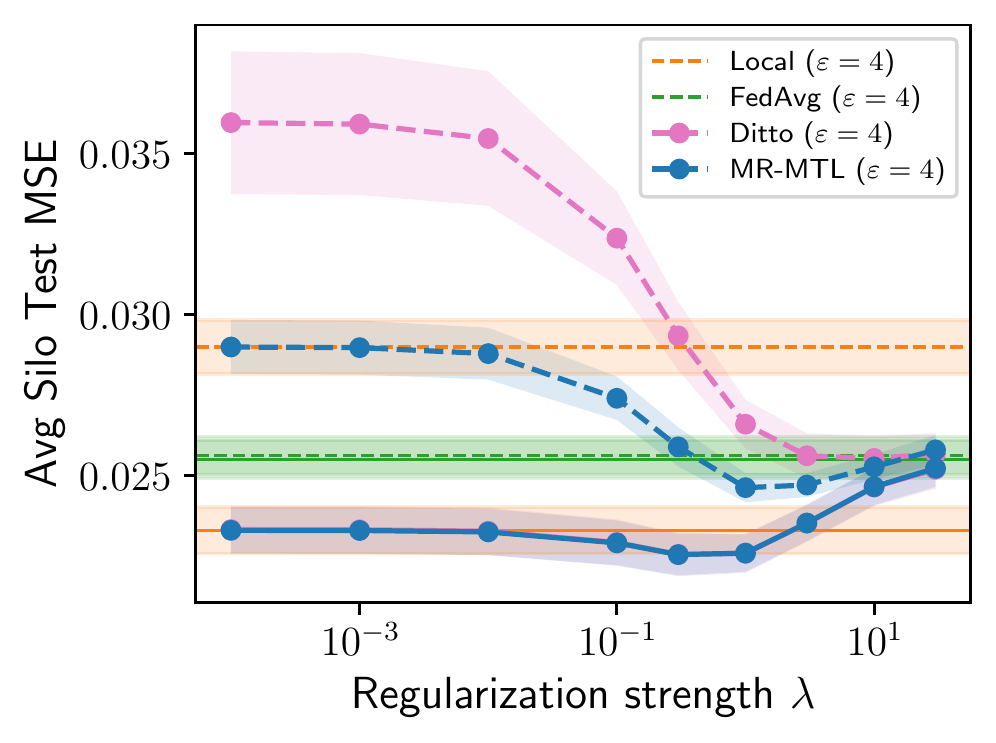}
    \includegraphics[width=0.329\linewidth]{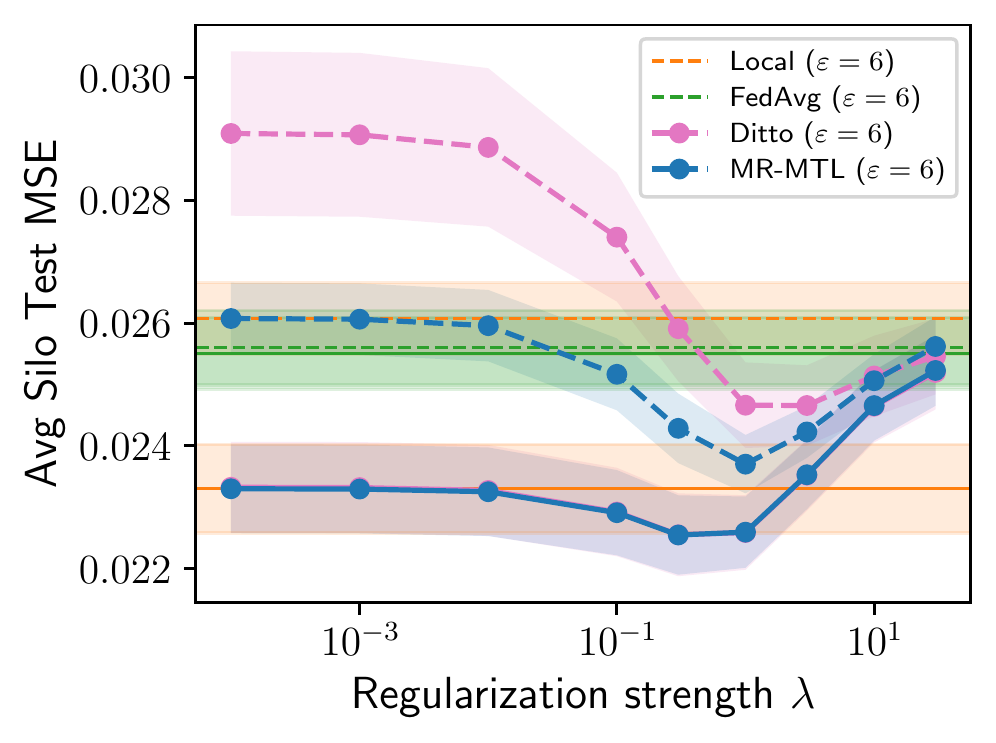}
    \includegraphics[width=0.329\linewidth]{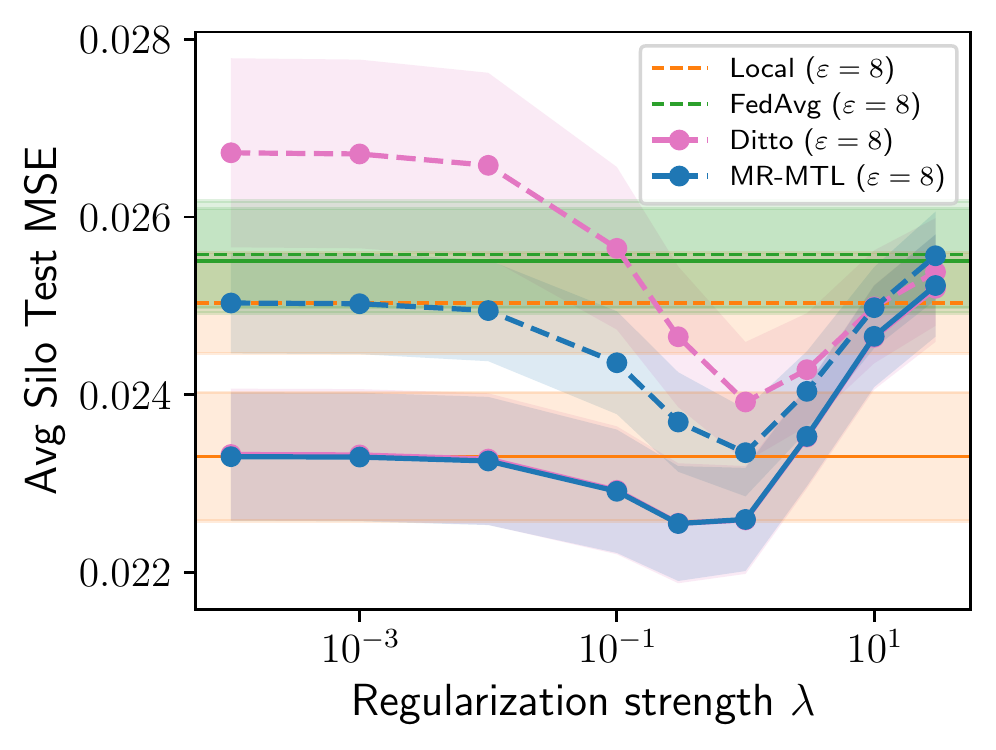}
    \caption{
      (Extension of \cref{fig:mrmtl-bump}) \textbf{Behavior of $\mrmtl$ as a function of $\lambda$ across varying privacy budgets} ($\eps \in [1, 2, 3, 4, 6, 8], \delta=10^{-3}$) for $T=200$ rounds on the \textbf{School} dataset. Cf.~\cref{fig:pareto-merged}~(b) and \cref{supp-fig:pareto-school-t200}, where the privacy regime of interest is around $\eps \approx 6$.
      Solid lines refer to the non-private runs (same across all plots).
    }
    \label{supp-fig:mrmtl-varying-bumps-school}
  \end{figure}
  \begin{figure}[t]
    \centering
    \includegraphics[width=0.329\linewidth]{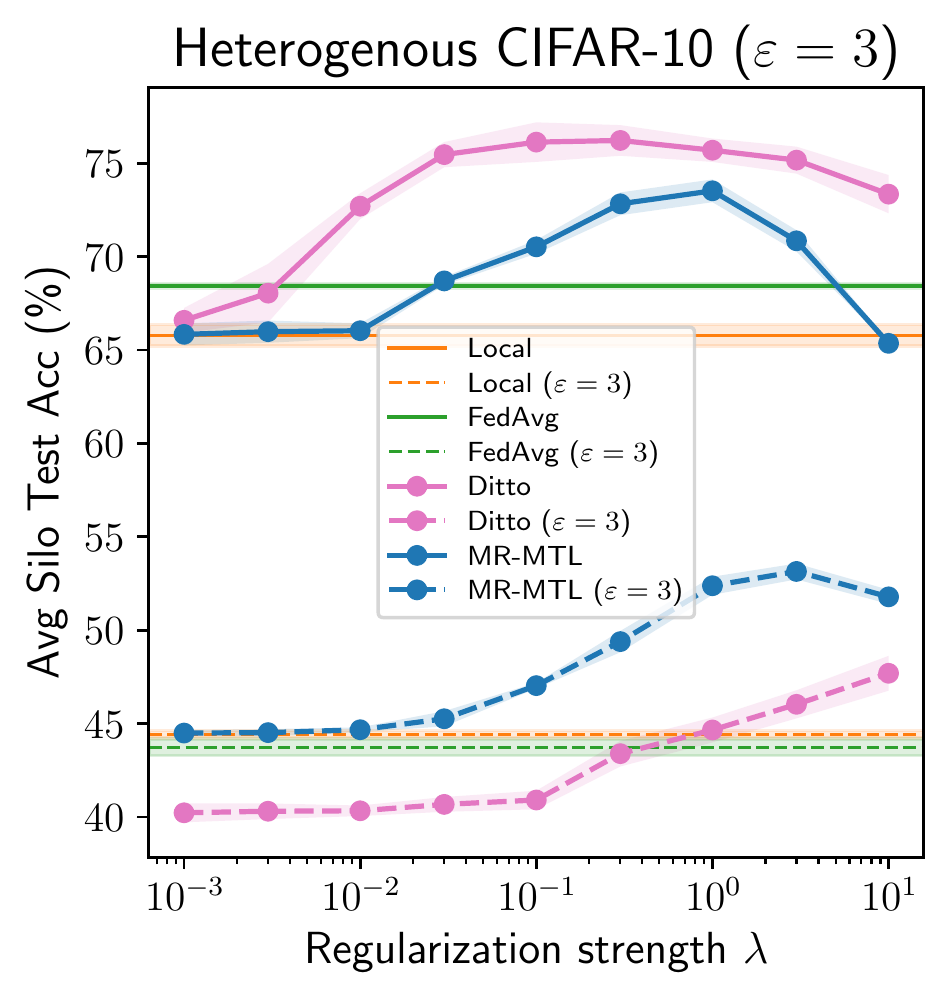}
    \includegraphics[width=0.329\linewidth]{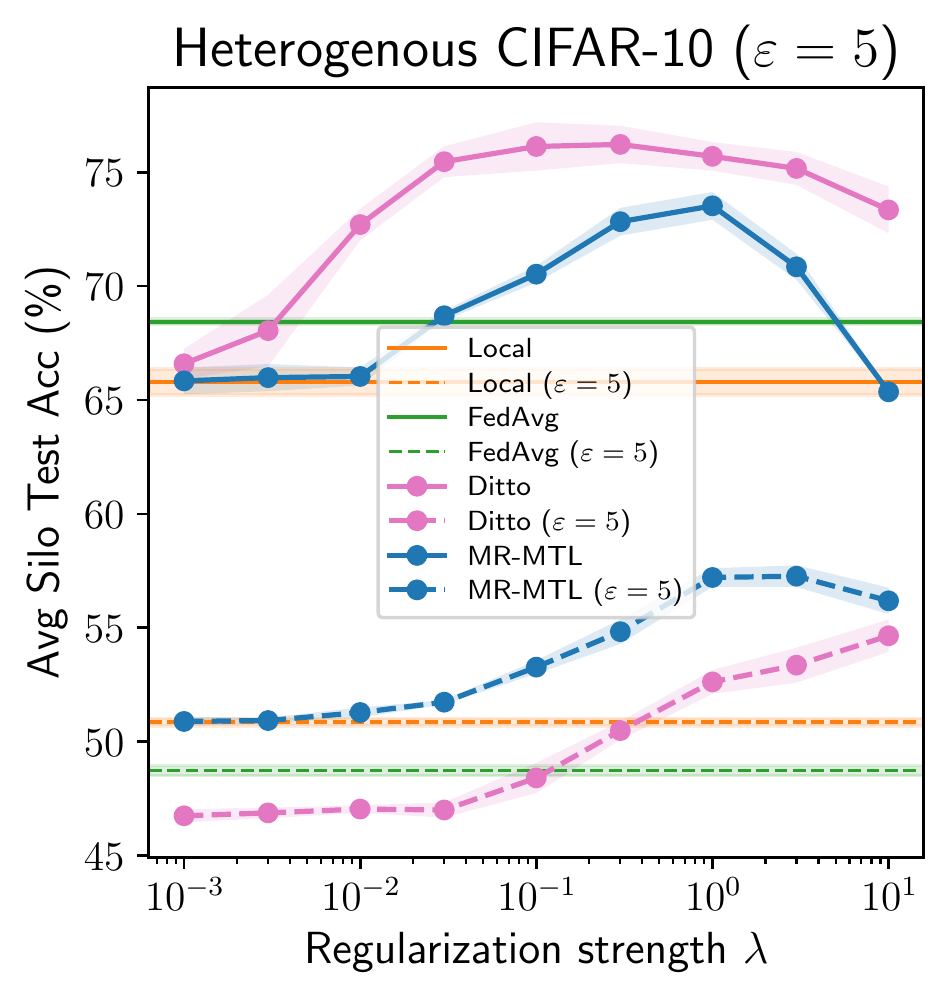}
    \includegraphics[width=0.329\linewidth]{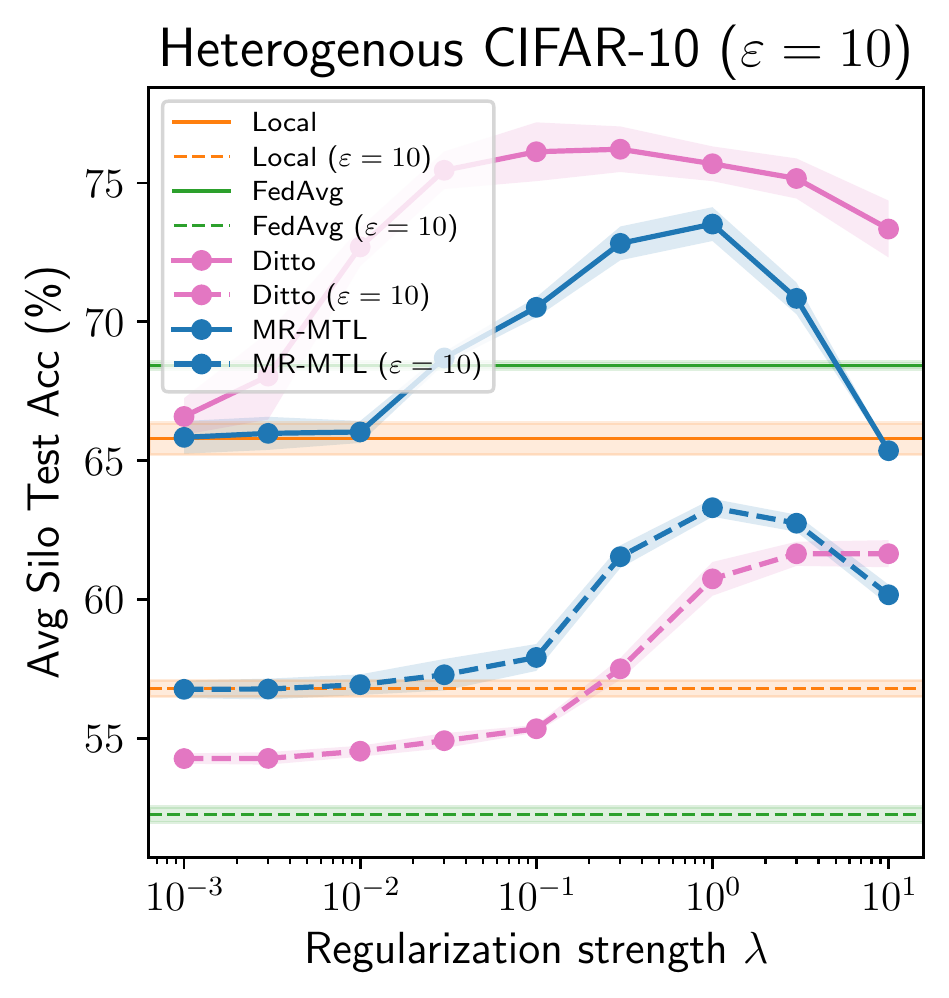}
    \includegraphics[width=0.329\linewidth]{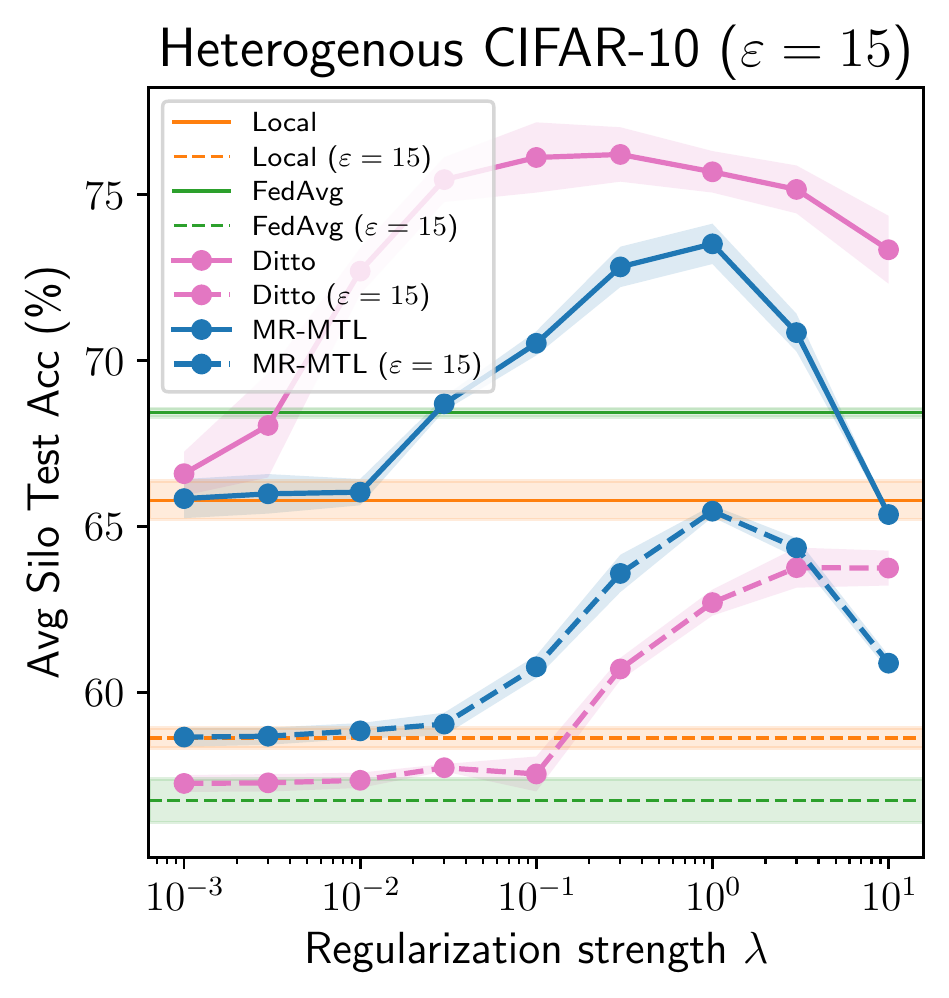}
    \includegraphics[width=0.329\linewidth]{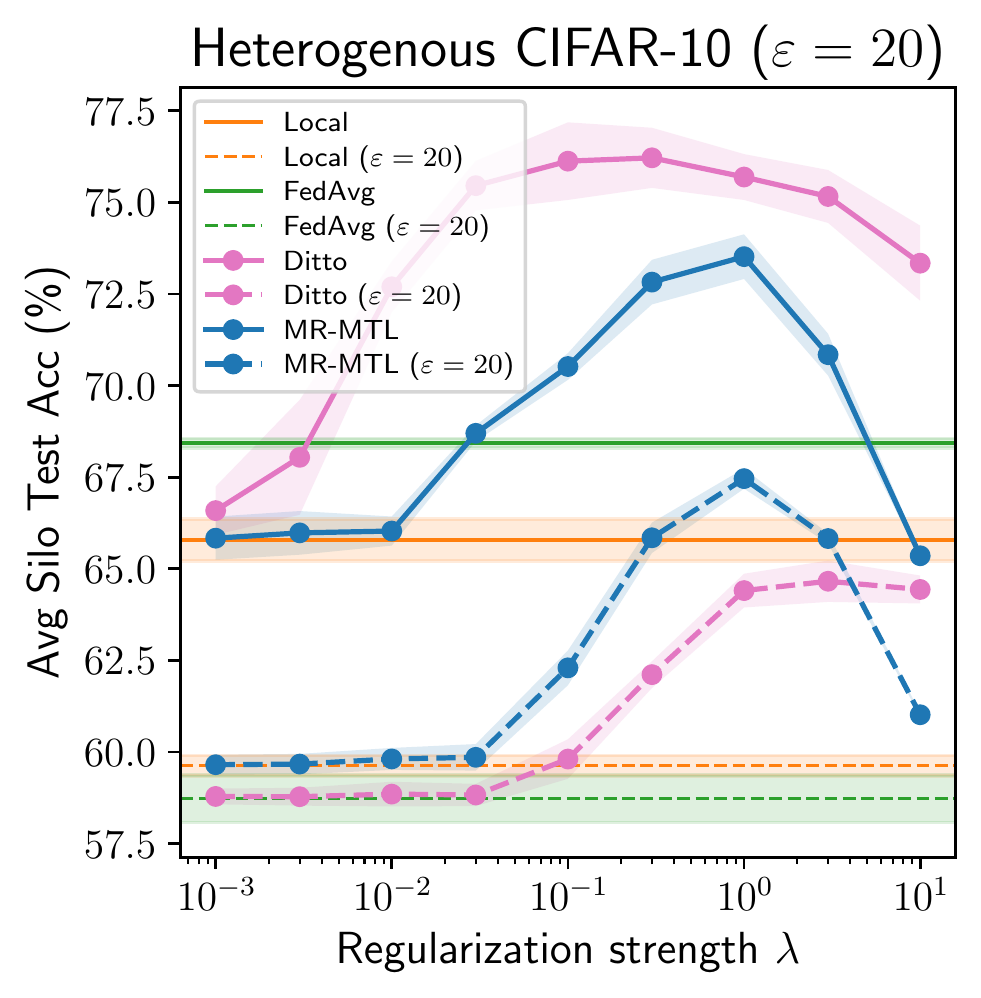}
    \includegraphics[width=0.329\linewidth]{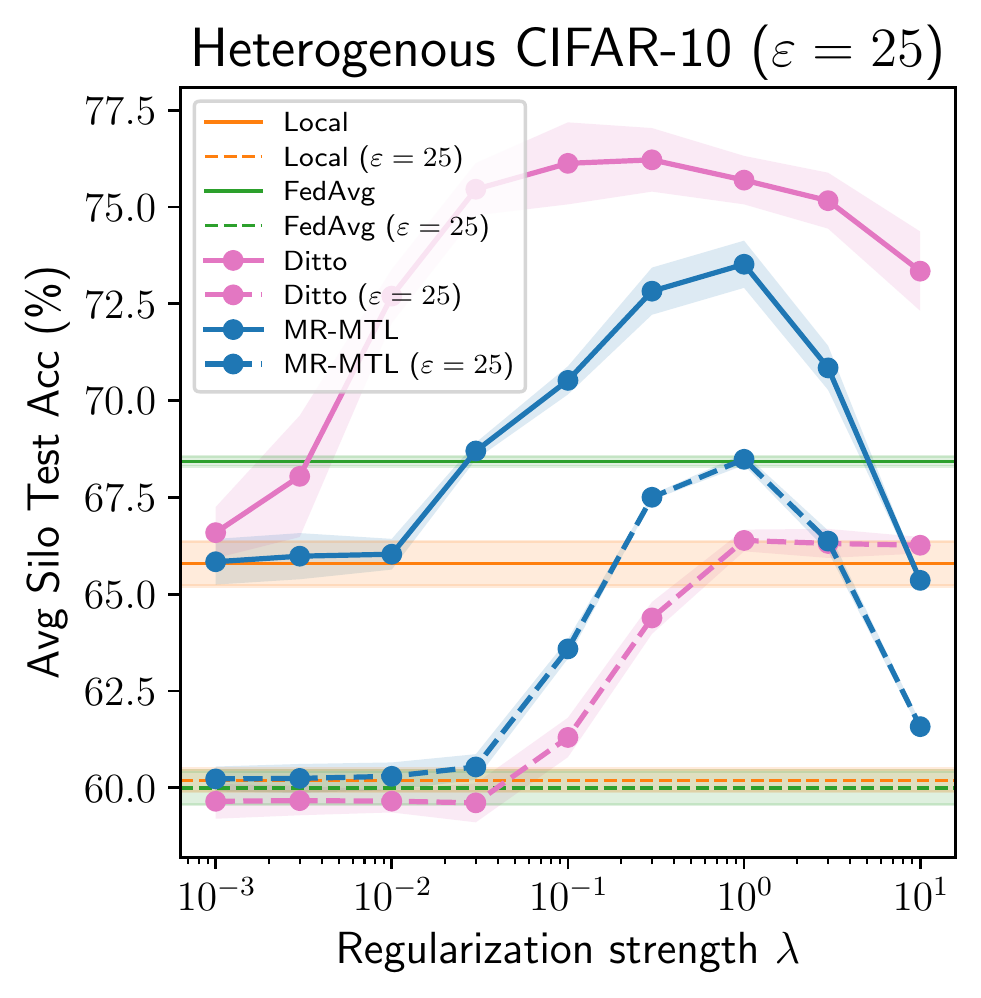}
    \caption{
      (Extension of \cref{fig:mrmtl-bump}) \textbf{Behavior of $\mrmtl$ as a function of $\lambda$ across varying privacy budgets} ($\delta=10^{-4}$) for $T=200$ rounds on \textbf{Heterogeneous CIFAR-10}.
      Solid lines refer to non-private runs (same across all plots) and dashed/dotted lines refer to private runs (with silo-specific sample-level DP).
    }
    \label{supp-fig:mrmtl-varying-bumps-cifar10}
  \end{figure}

\subsection{Subsampled ADNI Dataset}
\label{supp-sec:adni}

We further evaluate the suite of personalization methods on the subsampled ADNI dataset, and the results are shown in \cref{supp-fig:adni-pareto}. There are several notable observations:
\begin{enumerate}[leftmargin=*]
  \item \mrmtl is essentially recovering local training due to the high degree of heterogeneity present in this dataset, with $\lam^* \approx 0$ for most $\eps$ values except around $\eps \approx 3.5$ (annotated with arrows in \cref{supp-fig:adni-pareto}~(a, b)).
  \item Unlike results on other datasets, the privacy regime of interest where \mrmtl outperforms both local/FedAvg (around $\eps\approx 4$) is extremely narrow, compared to $\eps \approx 0.5$ for Vehicle (\cref{fig:pareto-merged}~(a)), $\eps \approx 6$ for School (\cref{fig:pareto-merged}~(b)), and $\eps \approx 1.5$ for GLEAM (\cref{fig:pareto-merged}~(c)). The utility advantage of \mrmtl is also insignificant.
  \item The underlying heterogeneity structure of the dataset is rather pathological in that it is not ameanable to \textit{both} mean-regularization and clustering.
  First, observe from \cref{supp-fig:adni-pareto}~(a) that \mrmtl still opts for a small $\lam^*$ in high-privacy regimes even when FedAvg performs better; one would expect \mrmtl to opt for a larger $\lam$ for lower DP noise (at a cost of higher heterogeneity).
  Second, from \cref{supp-fig:adni-pareto}~(b, c), we observe that both clustering and cluster-preconditoning did not lead to significant utility improvements, although the latter reduces the degree of heterogeneity and allows \mrmtl to opt for a larger $\lam^*$ in the privacy regime of interest.
\end{enumerate}

In general, determining how to better model such heterogeneity (especially in high privacy regimes)  is an interesting direction of future work.

\begin{figure}[t]
  \centering
  \includegraphics[width=\linewidth]{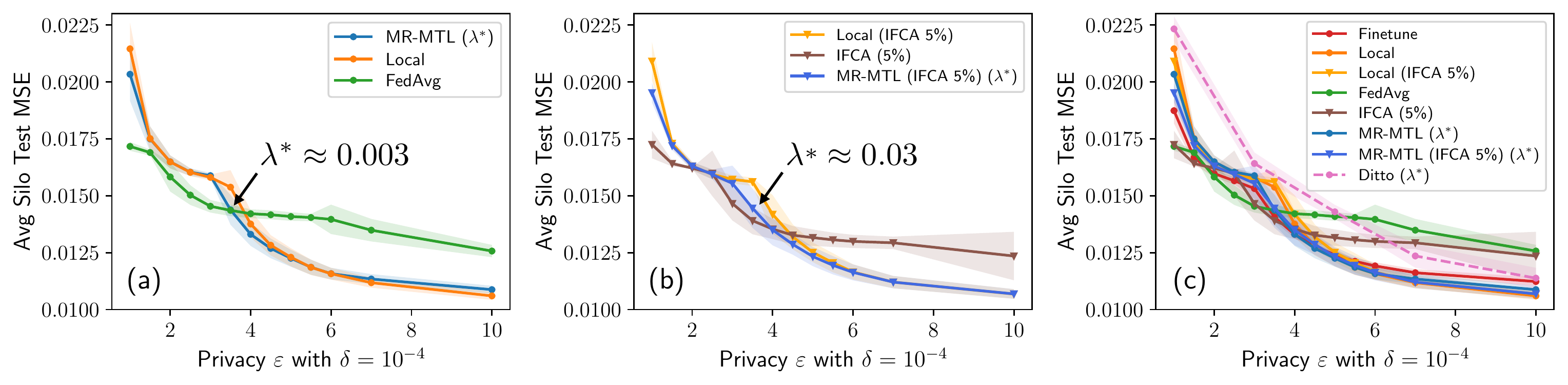}
  \caption{
    \textbf{Test accuracy} (mean~$\pm$~std with 3 seeds) vs \textbf{privacy budgets $\eps$} for various personalization methods on the {\textbf{subsampled ADNI}} dataset with $T=500$ rounds.
    The heterogeneity present in this dataset is unique in that it is not ameanable to both mean-regularization (a) and clustering (b, c), though clustering can mitigate the level of heterogeneity by allowing a larger $\lam^*$ for \mrmtl.
    In such cases of high heterogeneity, local performance tends to be superior for most privacy budgets, and we find that MR-MTL recovers this behavior. Determining how to better model such heterogeneity (especially in high privacy regimes) is an interesting direction of future work.
  }
  \label{supp-fig:adni-pareto}
\end{figure}

\subsection{Additional Discussions}

\begin{figure}
  \centering
  \includegraphics[width=0.4\linewidth]{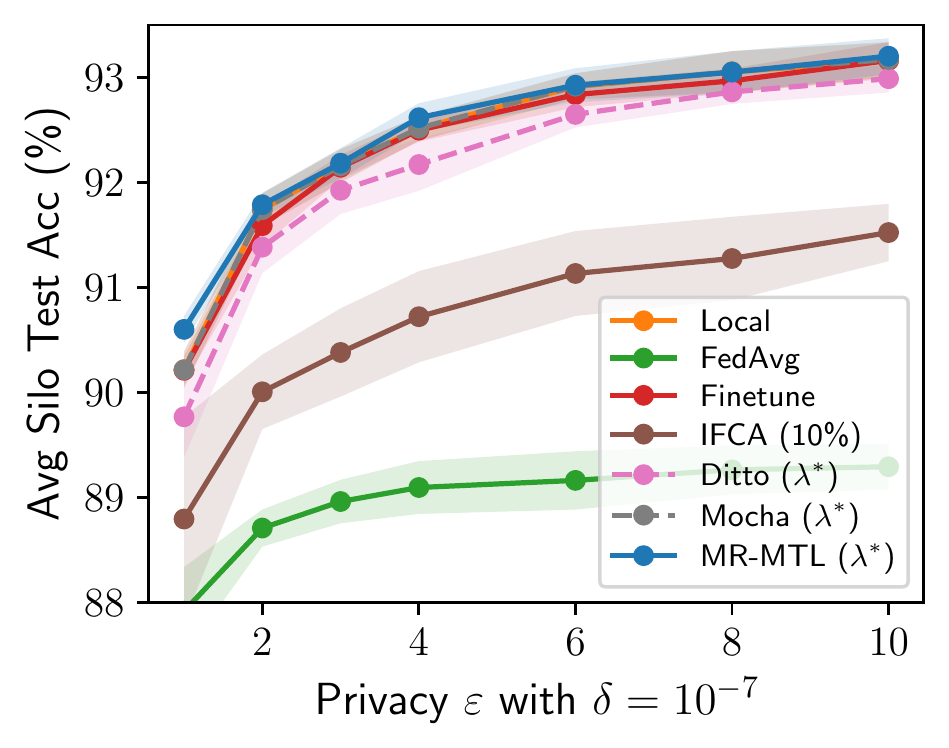}
  \captionof{figure}{(Extension of \cref{fig:pareto-merged}~(b)) Results on the \textbf{Vehicle} dataset for low privacy regimes ($1 \le \eps \le 10$). As privacy becomes weaker, the need for federation diminishes and personalization methods (including local training) perform better. \mrmtl can recover local training by using a small $\lam$.}
  \label{fig:pareto-sweep-eps10}
\end{figure}

\textbf{Behavior on larger privacy budgets (\cref{fig:pareto-sweep-eps10}).}\enspace
We extend the results on Vehicle (\cref{fig:pareto-merged}~(a, b)) by considering larger privacy budgets, and the results are shown in \cref{fig:pareto-sweep-eps10}. We note that as privacy requirement loosens, personalized learning (where each silo maintains its own model) tend to perform better than the case where models are shared across clients (FedAvg and IFCA). This is expected as clients in cross-silo datasets tend to have sufficient data for learning a good local model. Moreover, this also means $\mrmtl$ remains competitive as it can use a smaller $\lambda$.

\textbf{Effect of dataset subsampling (\cref{supp-fig:pareto-school-t200}).}\enspace
Subsampling local datasets would in principle turn the cross-silo learning setting closer to a cross-device learning setting, in which local training becomes less attractive as due to insufficient local data and may opt for federated training despite data heterogeneity. In \cref{supp-fig:pareto-school-t200}, we examine this behavior on the School dataset by using different train/test split ratios of 80\%/20\%, 50\%/50\%, and 20\%/80\%.
We observe that with less local training data: (1) FedAvg performs better since silos benefit from federation despite data heterogeneity; (2) the cross-over point between local training and FedAvg shifts to larger $\eps$ (or they may not be a cross-over point); and (3) \mrmtl may no longer provide an optimal point on the personalization spectrum, since small local datasets with silo-specific sample-level DP necessitate a larger $\lam$ (cf.~\cref{thm:optimal-lambda}) but \mrmtl may not recover FedAvg under (DP-)SGD.
In these cases, client-level DP protection (as is commonly used for cross-device FL) may be more appropriate.

\begin{figure}[ht]
  \centering
  \includegraphics[width=0.329\linewidth]{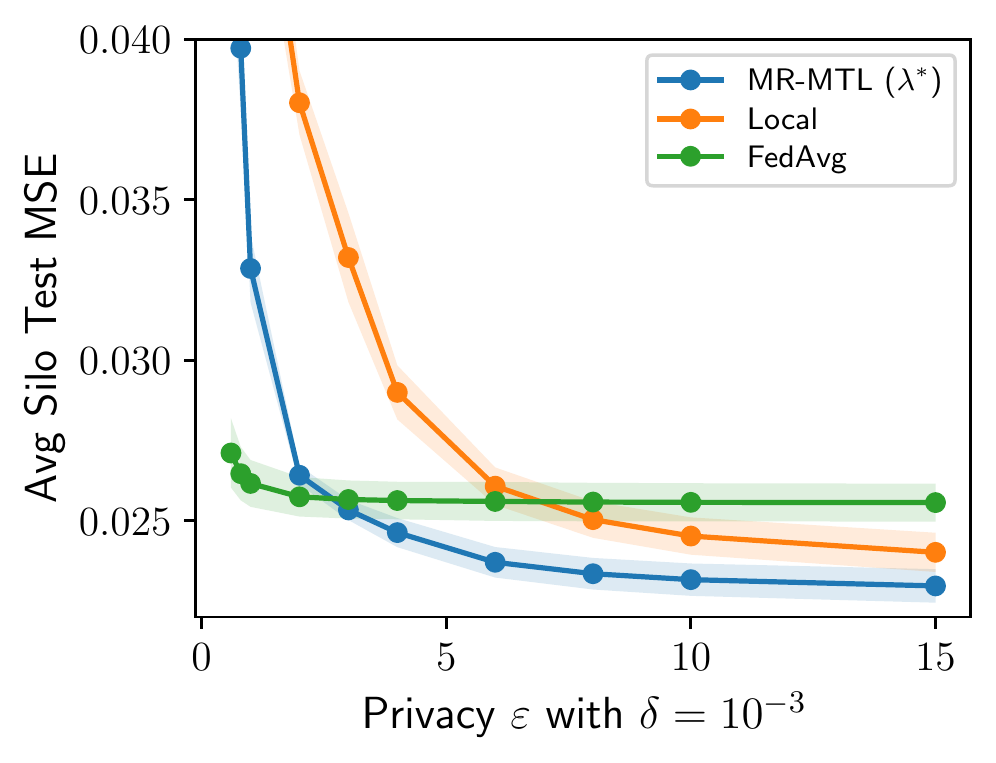}
  \includegraphics[width=0.329\linewidth]{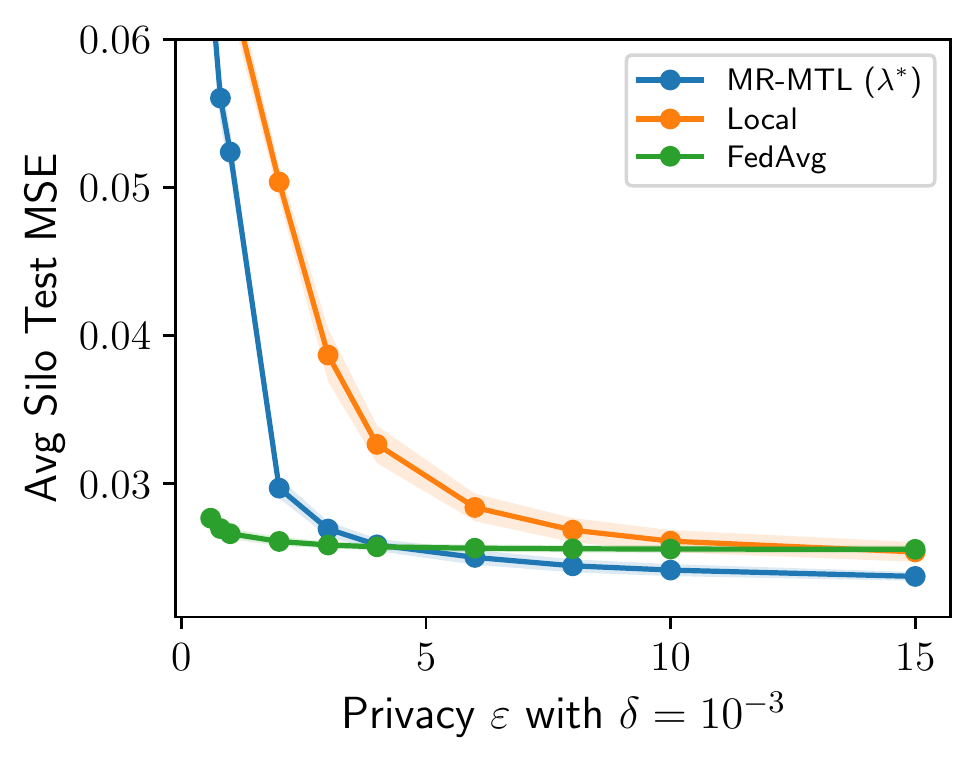}
  \includegraphics[width=0.329\linewidth]{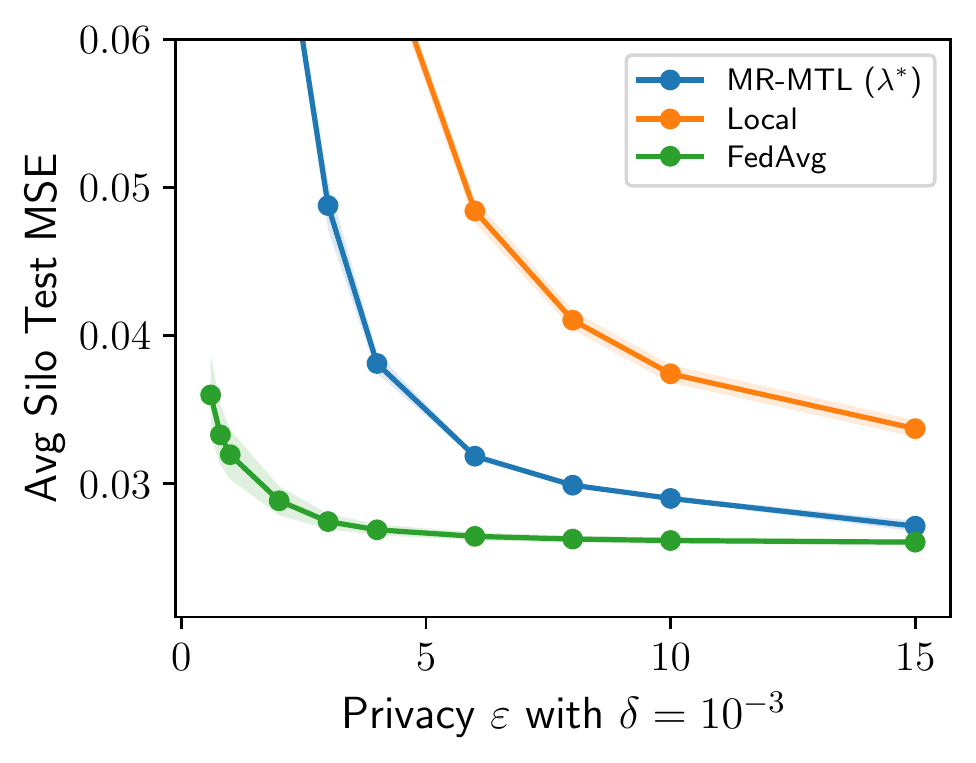}
  \includegraphics[width=0.329\linewidth]{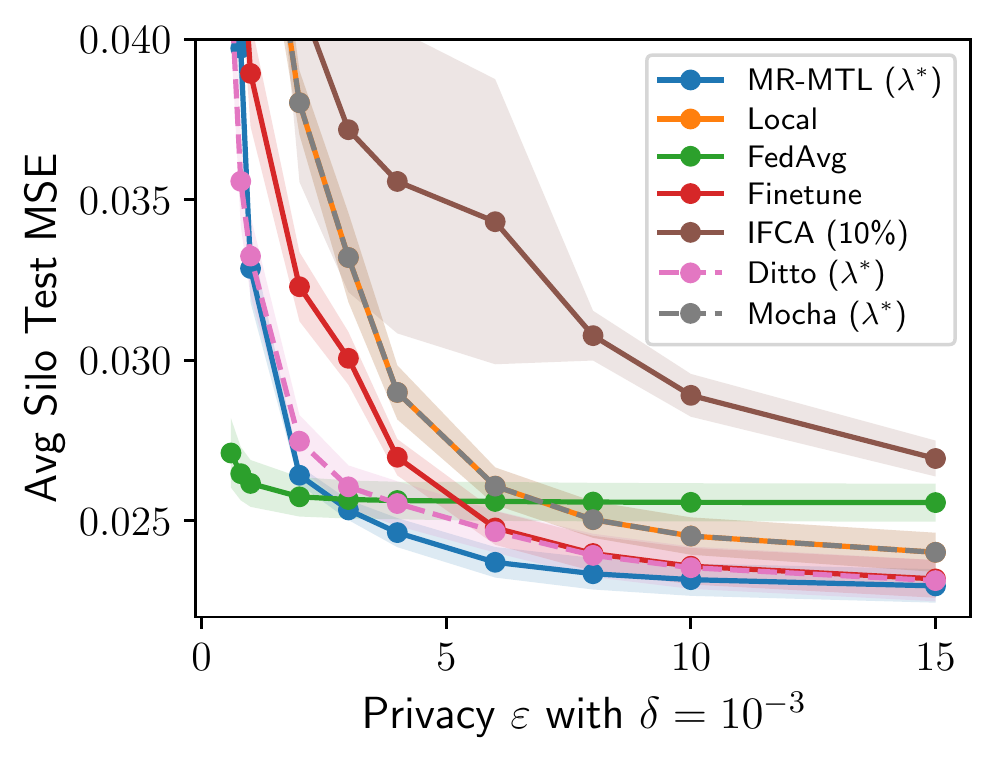}
  \includegraphics[width=0.329\linewidth]{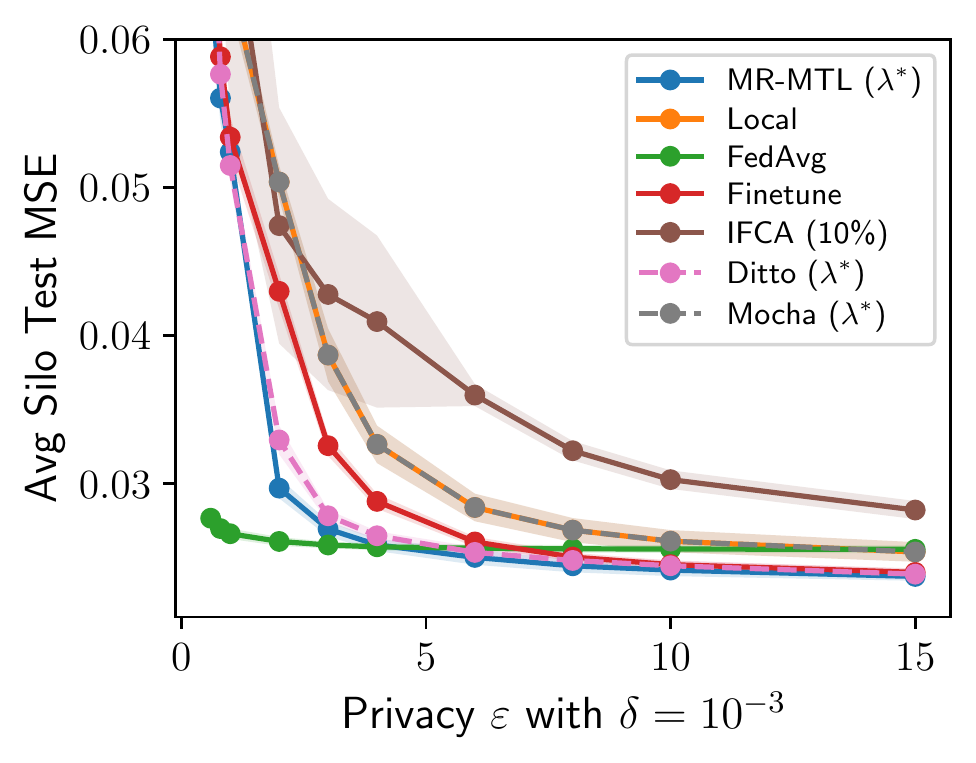}
  \includegraphics[width=0.329\linewidth]{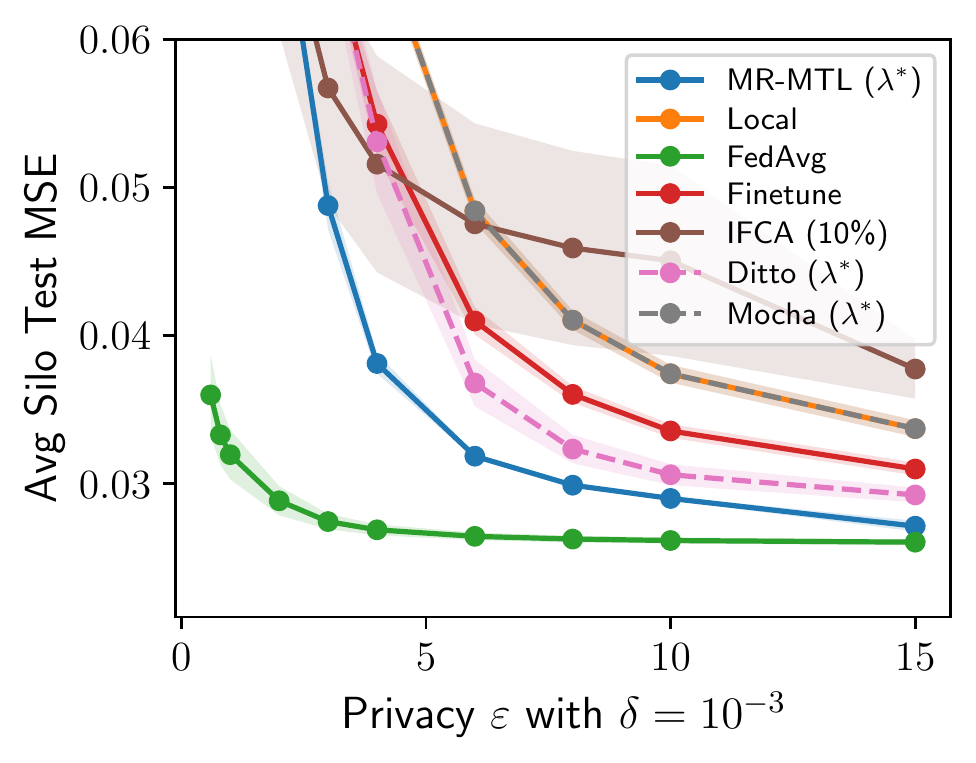}
  \caption{\textbf{Test MSE vs total privacy budgets on the School dataset} (80\%/20\%, 50\%/50\%, and 20\%/80\% train/test split of the local dataset by columns from left to right).
  The top row compares $\mrmtl$ to local training/FedAvg, which form the endpoints of the personalization spectrum with constant privacy costs, and the bottom row compares against other personalization methods.
  Under data subsampling (as little as 20\% training data in the 3rd column), we obtain a setting closer to cross-device FL where FedAvg outperforms personalization since clients benefit from others' training data despite their heterogeneity.
  }
  \label{supp-fig:pareto-school-t200}
\end{figure}

\end{document}